\documentclass[10pt,twocolumn,letterpaper]{article}

\usepackage{cvpr}              
\usepackage[accsupp]{axessibility}  

\usepackage[boxruled,vlined,linesnumbered]{algorithm2e}
\usepackage{multirow}
\usepackage{amsthm}
\usepackage{placeins}
\newtheorem{theorem}{Theorem}
\usepackage{algorithmic}
\usepackage[dvipsnames]{xcolor}

\definecolor{cvprblue}{rgb}{0.21,0.49,0.74}
\usepackage[pagebackref,breaklinks,colorlinks,citecolor=cvprblue]{hyperref}

\def\paperID{6696} 
\def\confName{CVPR}
\def\confYear{2024}

\title{Imputation-free and Alignment-free:
Incomplete Multi-view Clustering Driven by
Consensus Semantic Learning}

\author{
    Yuzhuo Dai\textsuperscript{1}, 
    Jiaqi Jin\textsuperscript{1}, 
    Zhibin Dong\textsuperscript{1}, 
    Siwei Wang\textsuperscript{2,}\thanks{Corresponding author}, 
    Xinwang Liu\textsuperscript{1,*}, 
    En Zhu\textsuperscript{1,*}, \\
    Xihong Yang\textsuperscript{1}, 
    Xinbiao Gan\textsuperscript{1}, 
    Yu Feng\textsuperscript{1} \\
    \textsuperscript{1}National University of Defense Technology, Changsha, China \\
    \textsuperscript{2}Intelligent Game and Decision Lab, Beijing, China \\
    {\tt\small \{yzdai24, wangsiwei13\}@nudt.edu.cn}
}

\begin{document}
\maketitle
\begin{abstract}
\noindent In incomplete multi-view clustering (IMVC), missing data induce prototype shifts within views and semantic inconsistencies across views. A feasible solution is to explore cross-view consistency in paired complete observations, further imputing and aligning the similarity relationships inherently shared across views. Nevertheless, existing methods are constrained by two-tiered limitations: (1) Neither instance- nor cluster-level consistency learning construct a semantic space shared across views to learn consensus semantics. The former enforces cross-view instances alignment, and wrongly regards unpaired observations with semantic consistency as negative pairs; the latter focuses on cross-view cluster counterparts while coarsely handling fine-grained intra-cluster relationships within views. (2) Excessive reliance on consistency results in unreliable imputation and alignment without incorporating view-specific cluster information. Thus, we propose an IMVC framework, imputation- and alignment-free for consensus semantics learning (FreeCSL). To bridge semantic gaps across all observations, we learn consensus prototypes from available data to discover a shared space, where semantically similar observations are pulled closer for consensus semantics learning. To capture semantic relationships within specific views, we design a heuristic graph clustering based on modularity to recover cluster structure with intra-cluster compactness and inter-cluster separation for cluster semantics enhancement. Extensive experiments demonstrate, compared to state-of-the-art competitors, FreeCSL achieves more confident and robust assignments on IMVC task.
\end{abstract}

\section{Introduction}
\label{sec:intorduction}
Thanks to representation learning enhanced by data observed from different perspectives, multi-view clustering (MVC) has achieved significant breakthroughs in the field of unsupervised learning \cite{zhou2024survey, yan2024differentiable, ke2024rethinking, cui2024novel,hu2023deep,fang2023comprehensive, li2018survey,10486880}. However, in practical applications, the assumption of data completeness is often difficult to satisfy that incomplete multi-view clustering (IMVC) is introduced \cite{tang2024incomplete, liu2018late, wang2021generative, feng2024partial, pu2024adaptive,10506102,zhang2021one}. In IMVC, missing data causes prototypes shifts within views and semantic misalignment across views, due to the discrepancy between the distributions of complete and incomplete instances \cite{huang2020partially, li2023incomplete}. More and more studies \cite{jin2023deep, dong2024subgraph} have noted that variations in complete instances across different views further exacerbate prototype misalignment. It is challenging to achieve semantic consistency on cluster assignments across all view data.

\begin{figure}[t!]
\includegraphics[width=\linewidth]{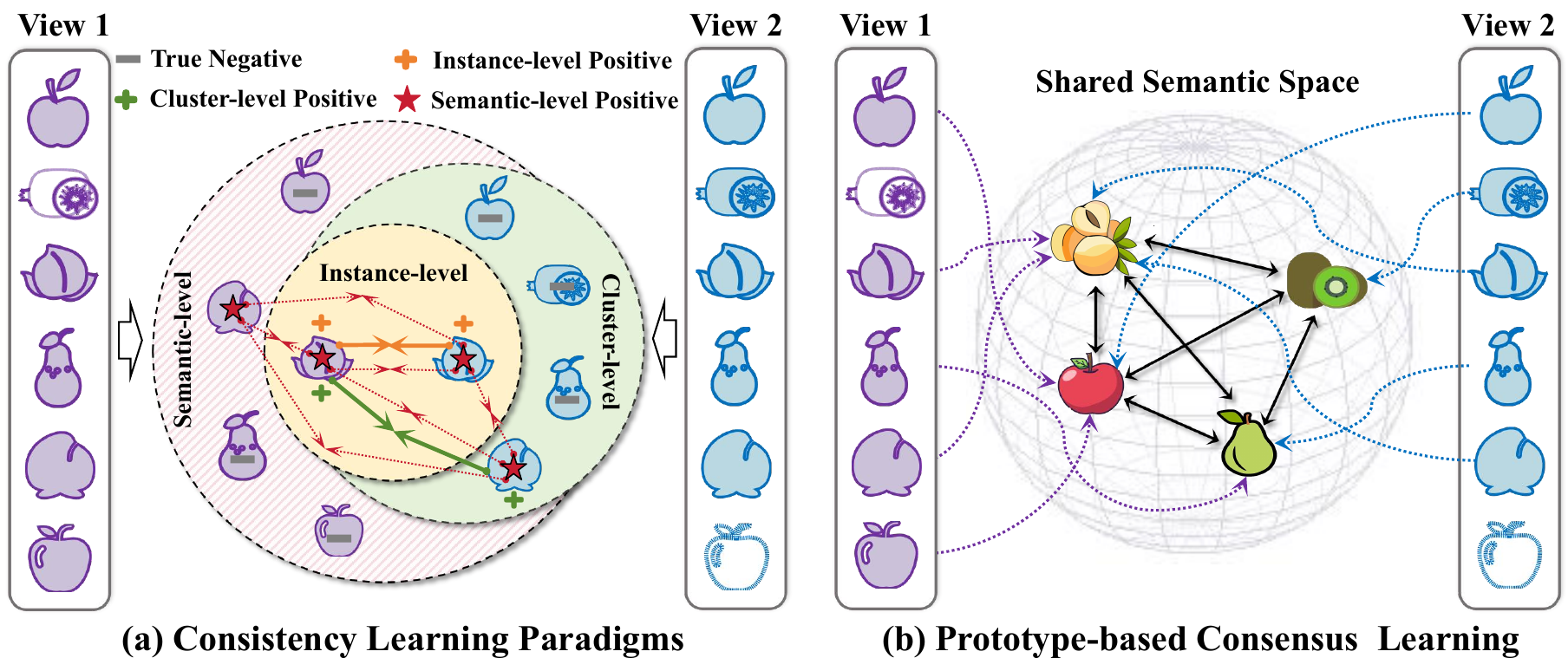}
\vspace{-16pt}
\caption{Research Motivation for Consensus Semantic Learning. (a) two existing paradigms, instance- and cluster-level, either treat cross-view unpaired observations with similar semantics as false negatives or neglect to treat within-view observations with similar semantics as true positives; (b) we propose a novel semantic-level paradigm based on contrastive clustering with a set of consensus prototypes to foster semantic consistency across all view data. }
\label{fig:intro}
\vspace{-9.5pt}
\end{figure}
To alleviate prototype shifts and misalignment while promoting semantic consensus in cluster assignments, existing IMVC methods explore consistency information from complete instances for imputation and alignment \cite{li2023cross,lin2022dual,tang2022deep}. Unfortunately, they still suffer from several significant drawbacks in practical applications. In terms of consistency learning, as shown in Fig. \ref{fig:intro} (a), one widely used paradigm is instance-level \cite{wang2022rethinking,huang2020deep}, which pulls paired observations (the same instance in different views) closer in the representation space by enforcing highly similar representations, but may inadvertently discard view-specific information \cite{sun2024robust,lu2024survey,yang2021partially}. More importantly, it tends to introduce false-negative noise, where unpaired observations with the same semantics are incorrectly treated as negative pairs \cite{lu2024decoupled,guo2024robust,caron2020unsupervised}. \cite{SURE},\cite{zeng2023semantic} and \cite{lu2024decoupled} attempted to optimize this issue by proposing a cluster-level paradigm \cite{MvCLN} that encourages observations to find their cluster counterparts across different views. It learns a cluster space shared across views but not applicable within specific views, as there are no clustering interactions among intra-view observations.

Since the above consistency learning paradigms fail to account for semantic consistency across all view data, they cannot mitigate semantic gaps among intra-cluster observations \cite{shen2021you,li2021contrastive,zhong2020deep}. Therefore, imputation of missing data is required to restore the original data distribution, including neighborhood-based recovery via cross-view graph structure transfer \cite{wang2022incomplete}, adversarial generation or contrastive prediction through cross-view mutual information interaction \cite{wang2021generative,zhang2020deep,jiang2019dm2c,xu2019adversarial}, and prototype-based imputation via cross-view sample-prototype relationship inheritance \cite{li2023incomplete}. Meanwhile, cross-view alignment of assignments \cite{tang2022deep,li2021contrastive}, prototypes \cite{jin2023deep, li2023incomplete,wangalign}, or distributions \cite{dong2024subgraph,xu2023adaptive} is also a crucial approach for further enhancing consistent learning. More related works are enumerated in Appendix A. Both imputation and alignment are limited by the consistent information from cross-view paired data and cannot fully exploit within-view unpaired data to mine view-specific cluster information, \emph{i.e.,} within-cluster compactness and between-cluster separation \cite{he2024robust,li2023incomplete}. Particularly, once the amount of missing data is too excessive to provide sufficient consistent information, model performance may even decline rapidly due to the noise introduced by improper imputation and alignment.

Realizing the above issues, we ponder: imputation and alignment aim to restore the similarity relationships inherited from other complete views for the clustering task. Can we directly bridge the semantic gaps while exploring the semantic relationships among all data in consistency learning, thereby avoiding imputation and alignment operations with uncertainty noise? Thus, we propose an IMVC framework, a novel semantic-level paradigm, as shown in Fig. \ref{fig:intro} (b), that is driven by imputation- and alignment-free consensus semantic learning (FreeCSL). Notably, our proposed consensus learning involves concurrent interaction among all data, rather than being limited to within or across views. To bridge semantic gaps among view observations and learn cluster semantics information, FreeCSL employs contrastive clustering based on consensus prototypes to discover a shared semantic space, where observations converge toward their semantic prototypes respectively. In practice, we set missing statistical weights for observations to facilitate view collaboration in consensus representations that integrate detailed information of all views while adapting the impact of different views with different missing instances. Based on consensus representation, our model can construct robust consensus prototypes for semantic-level clustering without imputation and alignment. To discover view-specific semantic relationships and further enrich consensus representations, FreeCSL exploit graph clustering to capture the cluster structure for each view, which maximizes graph modularity within views to enhance intra-cluster connections and reduce inter-cluster interactions. In short, our model encodes data correlation, discovered by semantic learning, into a shared space to obtain consensus semantic representations for instance clustering. Our prominent contributions can be summarized as follows:

\begin{itemize}
\item In terms of bridging semantic gaps, we design the consensus semantic learning module, a novel semantic-level paradigm based prototypical contrastive clustering, to discover a common semantic space where all observations are embedded as representations with consistent semantics, avoiding additional imputation and alignment.
\item In terms of exploring semantic relationships, we employ the cluster semantics enhancement module, a heuristic graph clustering method with modularity-based learning objective, to mine the inherent cluster structures that reveal the semantic correlations within views.
\item Extensive experiments show our model surpasses state-of-the-art (SOTA) competitors in complex tasks with high missing rates, multiple clusters and large-scale data.

\end{itemize}

\section{Method}
\label{sec:method}
In this section, FreeCSL, a deep IMVC method without imputation or alignment, is proposed to learn consensus semantic representations for clustering. The framework in Fig. \ref{fig:short} coordinates reconstruction (REC) module, cross-view consensus semantic learning (CSL) module, and within-view cluster semantic enhancement (CSE) module. 

\noindent 
\subsection{Problem Statement}

\textbf{Notations.} 
Given a multi-view dataset $\mathcal{X} = \{ \mathbf{X}^{{v}} \in \mathbb{R}^{N \times D_v} \}_{v=1}^{V}$ with $N$ instances across $V$ views, $\tilde{\mathbf{X}}^{{v}} = \{ \tilde{\mathbf {x}}_i^{{v}} \in \mathbb{R}^{D_v} \}_{i=1}^{N_v}$ is an incomplete subset of the $v$-th view with ${N_v}$ observations, and $\overline{\mathbf{X}}^{{m,n}} =\{ (\overline{\mathbf {x}}_i^{{m}},\overline{\mathbf {x}}_i^{{n}}) \}_{i=1}^{N_{mn}}$ is a pair-observed subset with ${N_{mn}}$ instances observed in both the $m$-th and $n$-th view. The task is to partition $N$ instances into $K$ clusters.

\noindent\textbf{Definition 1.} \textit{Instance-level Consistency (IC): $\forall m \neq n$, $\mathbf{x}_i^{{m}}$ and $\mathbf{x}_j^{{n}}$ are instance-level consistent across views if $i=j$ (they are cross-view observations of the same instance $\mathbf{x}$), expressed as $I(\mathbf{x}_i^{{m}},\mathbf{x}_j^{{n}})=1$ and 0 otherwise.}

\noindent\textbf{Definition 2.} 
\textit{Cluster-level Consistency (CC): $\forall m \neq n$, $\mathbf{x}_i^{{m}}$ and $\mathbf{x}_j^{{n}}$ are cluster-level consistent across views if they belong to the same cluster $k$, expressed as $C(\mathbf{x}_i^{{m}},\mathbf{x}_j^{{n}})= 1$ and 0 otherwise.}

\noindent\textbf{Definition 3.} 
\textit{Semantic-level Consensus (SC): $\forall m$ and $n$, $\mathbf{x}_i^{{m}}$ and $\mathbf{x}_j^{{n}}$ achieve semantic-level consensus in MVC task if all observations share a set of cluster prototypes $\mathbf{C}={\{\mathbf{c}_k\}_{k=1}^K}$ and $\arg\max\limits_{k} \mathcal{\rho}(\mathbf{x}_i^{m}, \mathbf{c}_k) = \arg\max\limits_{k} \mathcal{\rho}(\mathbf{x}_j^{n}, \mathbf{c}_k)$ , expressed as $S(\mathbf{x}_i^{{m}},\mathbf{x}_j^{{n}})=1$ and 0 otherwise.}

\begin{theorem}\label{thm:Theorem} 
Consensus semantic learning yields more confident and robust cluster assignments than instance- and cluster-level paradigms. (Proof is provided in Appendix B.)
\end{theorem}

\begin{theorem}\label{thm:Theorem2} 
Since paired observations $ (\overline{\mathbf {x}}_i^{{m}},\overline{\mathbf {x}}_i^{n})$ inherently satisfy instance- and cluster-level consistency, they can achieve semantic consensus via a shared set of prototypes $\mathbf{C}$. (Proof is provided in Appendix B.)
\end{theorem}

\begin{figure*}
  \centering  
 \includegraphics[width=16.2cm, height=9cm]
  {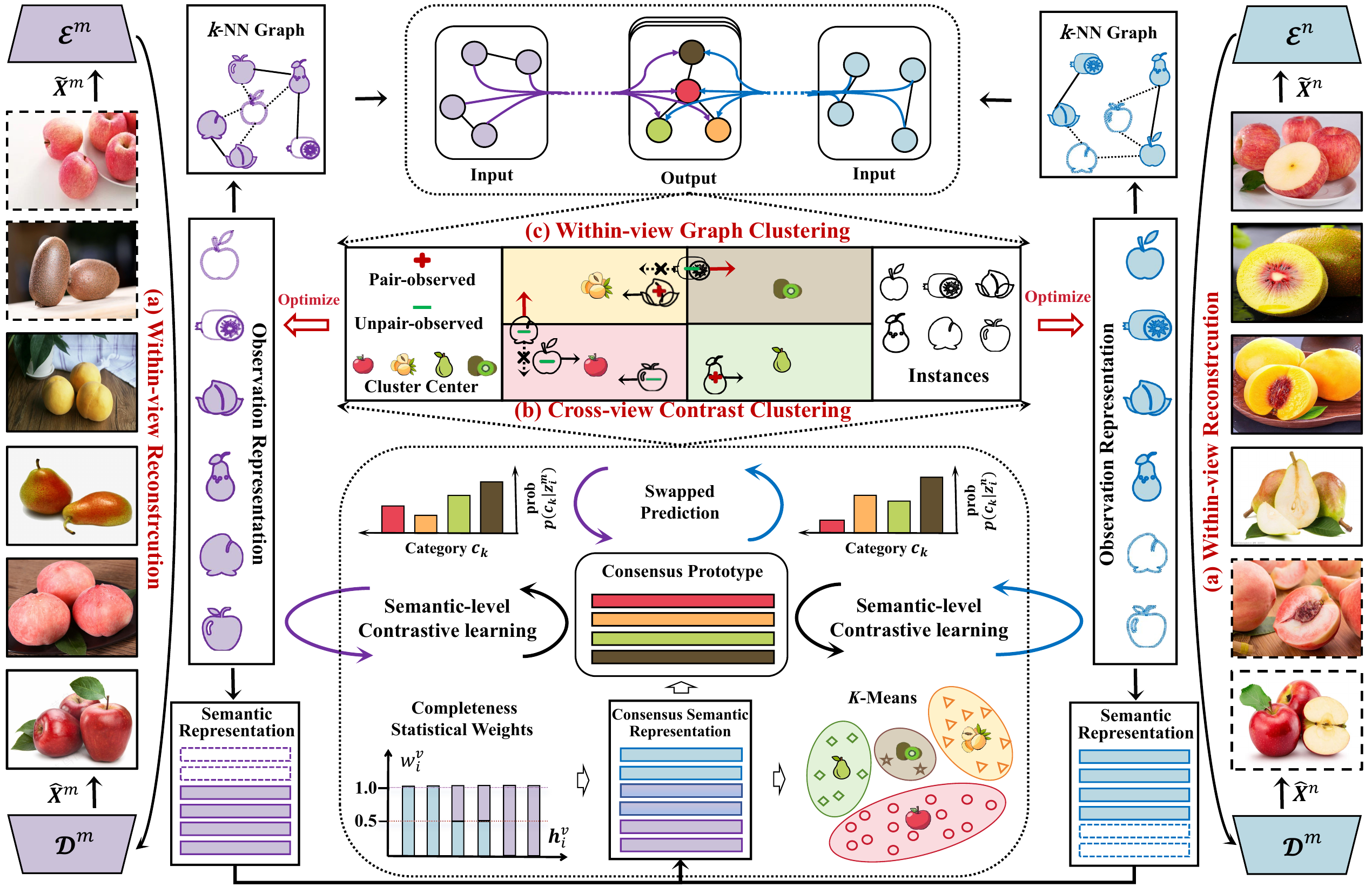}
    \caption{The framework of FreeCSL. (a) Reconstruction module, encodes observations into clustering-friendly representations for each view; (b) Consensus semantic learning module, learns semantic representations through cross-view contrastive clustering in a shared semantic space, where paired observations are assigned to their nearest semantic prototype for consistent assignments. (c) Cluster semantic enhancement module, enriches semantic representations with cluster structure information through within-view graph clustering, which applies GCN to aggregate view-specific semantic information and maximizes spectral modularity to recover cluster structure with greater separation and lower entropy. Ultimately, perform $k$-means on the consensus semantic representations to predict cluster labels.}
    \label{fig:short}
    \vspace{-5pt}
\end{figure*}

\subsection{Within-view Reconstruction}
To avoid clustering instability from similarity measures on manifold structures in high-dimensional spaces, we use autoencoders to encode view observations $\tilde{\mathbf{X}}^{{v}}$ into clustering-friendly low-dimensional representations $\tilde{\mathbf{Z}}^{{v}}$. Considering multiple views are mostly heterogeneous and differently distributed, we provide an independent encoder-decoder $\{\mathcal{E}_{v},\mathcal{D}_{v}\}_{v=1}^{V}$ for each view. The encoder embeds the latent representation and the decoder recovers the original data from it, which jointly minimize reconstruction loss:
\vspace{-3pt}
\begin{equation}
    \begin{aligned}
    \mathcal{L}_{rec} & = \sum_{v=1}^{V} \left\| \tilde{\mathbf{X}}^{{v}} - {\mathcal{D}}_{v}(\tilde{\mathbf{Z}}^{{v}}) \right\|_F^2 \\
    & = \sum_{v=1}^{V} \sum_{i=1}^{N_v} \left\| \tilde{\mathbf{x}}_i^{{v}} - \mathcal{D}_{v}(\mathcal{E}_{v}(\tilde{\mathbf{x}}_i^{{v}})) \right\|_2^2,
    \end{aligned}
\label{eq:rec Loss}
\end{equation}
where $\mathcal{E}_{v}(\tilde{\mathbf{X}}^{{v}}; \mathbf{\theta}_v) : \tilde{\mathbf{X}}^{{v}} \in \mathbb{R}^{N_v \times D_v} \rightarrow \tilde{\mathbf{Z}}^{{v}} \in \mathbb{R}^{N_v \times d}$ and $\mathcal{D}_{v}(\tilde{\mathbf{Z}}^{{v}}; \mathbf{\phi}_v) : \tilde{\mathbf{Z}}^{{v}} \in \mathbb{R}^{N_v \times d} \rightarrow \hat{\mathbf{X}}^{{v}} \in \mathbb{R}^{N_v \times D_v}$.

\subsection{Cross-view Consensus Semantic Learning}\label{sec:3.3}
Based on Theorem \ref{thm:Theorem}, we design prototypical contrastive clustering for consensus semantic learning. To reach semantic-level consensus on assignments, consensus prototypes with all view information, are introduced into contrastive clustering to explore a shared semantic space, where all observations with similar semantics are pulled closer.

\textbf{Consensus Representations and Consensus Prototypes.} \label{sec:Consensus Representation and Consensus Prototype}Due to the assumptions of consistency and complementarity in MVC, inseparable clusters in one view will become linearly separable by introducing complementary information from other views \cite{xu2022deep,hu2017sharable,zhang2024regularized}. We utilize a fusion manner, denoted as $\mathbb{T}(\cdot)$, to map the latent representations $\mathbf{Z}^{{v}}$ into a linearly weighted representation space for consensus representations $\mathbf{Z} \in \mathbb{R}^{N \times d}$ that integrates consistency and complementary information from multiple views:
\begin{equation}
\begin{split}
    \mathbf{Z} = \mathbb{T}(\{\mathbf{Z}^{{v}}\}_{v=1}^{V}) = \sum_{v=1}^{V} \mathbf{w}^{{v}} \mathbf{Z}^{{v}}
    = \Big \{\sum_{v=1}^{V} w_i^{{v}} \mathbf{z}_i^{{v}}\Big \} _{i=1}^{N},
\label{eq:consensus representations}
\end{split}
\end{equation}
where ${w}_i^{{v}}$ is the instance-level fine-grained fusion weight:
\begin{equation}
\begin{split}
{w}_i^{{v}} = \frac{\mathbb{I} \left( {\mathbf{z}_i^{{v}} \neq NaN } \right)
}{
\sum_{v'}^V\mathbb{I} \left( \mathbf{z}_i^{v'} \neq NaN \right)
},
\end{split}
\end{equation}
\noindent where $\mathbb{I}(\cdot)$ is the indicative function that takes $1$ when $\mathbf{z}_i^{{v}}$ is the representation of $i$-th instance $\mathbf{X}_i$ observed in view $v$, and 0 otherwise. We set a completeness statistical weight ${w}_i^{{v}}$ based on the number of observations in $V$ views for instance $\mathbf{X}_i$ as the fusion weight, which makes use of view-specific complementary information without discarding unpaired observations and mitigating the negative impact of missing noise by adapting to the differences in missing instances across different views.

We derive a set of consensus semantic prototypes $\mathbf{C} = \{\mathbf{c}_k \in \mathbb{R}^d \}_{k=1}^K$ via $k$-means on consensus representations $\mathbf{Z}$. As ``a representative embedding for semantically similar observations", $\mathbf{C}$ comprehensively captures the semantic information of all data \cite{yu2023sparse} and is continuously refined throughout consensus representation learning.

\textbf{Prototype-based Contrastive Clustering.} We perform semantic-level contrastive clustering in a shared vector space spanned by consensus prototypes, where the latent representation $\mathbf{z}_i^v$ encoded as the semantic representation $\mathbf{h}_i^v$ and $\mathbf{h}_i^v$ is assigned to the prototype $\mathbf{c}_k$ with the longest projected distance to obtain optimal assignments, in two steps:
\begin{itemize}
\item Semantic similarity measure: Encode the semantic representation ${\mathbf{h}^{{v}}_i}$ and project it onto each prototype, then calculate the probability $p_{i,k}^{{v}}$ of belonging to cluster $k$:
\begin{equation}
    p_{i,k}^{{v}} = \frac{\exp\left(\mathbf{h}_i^{{v}\top} \mathbf{c}_k / {\tau}  \right)}{\sum_{k^{'}}^K \exp\left(\mathbf{h}_i^{{v}\top} \mathbf{c}_{k^{'}} / {\tau} \right)},
\label{eq: the probability of representation imputated to prototype}
\end{equation}
\noindent where $\tau$ is a temperature parameter \cite{wu2018unsupervised}. $\mathbf{p}^{{v}}_i=\{p_{i,k}^{{v}}\}_{k=1}^K$ is the soft assignments of observation $\mathbf{x}^{{v}}_i$.

\item Swapped knowledge distillation: Based on Theorem \ref{thm:Theorem2}, $(\overline{\mathbf {x}}_i^{{m}},\overline{\mathbf {x}}_i^{{n}}) $ share the same cluster semantics on consensus prototypes $\mathbf{C}$, thus their true labels $\mathbf{y}_{i }^{{m}}, \mathbf{y}_{i}^{{n}}\in \mathbb{R}^{K}$ should be are distributionally consistent. To this end, ``swapped'' knowledge distillation (KD) utilizes pseudo-labels $\mathbf{q}_{i }^{{m}}$, $\mathbf{q}_{i}^{{n}}$ as mutual supervised signals to prompt
their cluster assignments $\mathbf{p}_{i }^{{m}}, \mathbf{p}_{i }^{{n}}$ as identical as possible:
\begin{equation}
    \mathbf{\ell}^{{m,n}}_{{cc}} = \ell_{{kd}}(\mathbf{H}^{{m}}, \mathbf{Q}^{{n}}) + \ell_{{kd}}(\mathbf{H}^{{n}}, \mathbf{Q}^{{m}}),
    \label{eq: cl loss}
\end{equation}
where $\displaystyle\ell_{{kd}}(\mathbf{H}^{{m}}, \mathbf{Q}^{{n}}) = -\frac{1}{N_{nm}} \sum_{i=1}^{N_{mn}}  \sum_{k=1}^K\mathbf{q}_{i}^{{n}} \log \mathbf{p}_{i}^{{m}}$.
\noindent We extend Eq. \eqref{eq: cl loss} to more than two views, fostering greater view collaboration to enhance semantic learning:
\begin{equation}
\mathcal{L}_{cc} =\sum_{m=1}^V \sum_{\substack{n=1 \\ n \neq m}}^V  \ell^{{m,n}}_{{cc}} .
\label{eq:total cl loss}
\end{equation}
\noindent \item Consensus label solution: To obtain pseudo-labels $\mathbf{Q}^{{m}} = \{ \mathbf{q}_{i }^{{m}} \}_{i=1}^{N_{mn}}$, project semantic representations $\mathbf{H}^{{m}}$ onto consensus prototypes $\mathbf{C}$ and maximize similarity between the representation $\mathbf{h}_i^{{m}}$ and its assigned prototype $\mathbf{c}_{k}$:
\begin{equation}
\mathbf{Q}^{*} = \arg\max_{\mathbf{Q}^{{m}} \in \mathcal{Q}} \, \text{Tr}\{({\mathbf{C}\mathbf{Q}^{{m}}})^\top \mathbf{H}^{{m}})\} + \alpha \mathcal{R}(\mathbf{Q}^{{m}}), \label{eq:optimal transport}
\end{equation} where $\mathcal{R}(\mathbf{Q}^{{m}})=-\sum_i^{N_{mn}}\mathbf{q}_{i}^{{m}} \log \mathbf{q}_{i}^{{m}}$, an entropy regularizer prevents trivial solutions via the smoothness $\alpha$. 

\end{itemize} 

\subsection{Within-view Cluster Semantic Enhancement}\label{sec:3.4}
Considering that graph pooling can discover data implicit correlation and reduce missing noise via aggregating the information of neighboring nodes, we design a graph clustering to enhance representations with cluster semantic information. It combines a heuristic learning objective based on modularity, an evaluation metric for cluster structure quality, to effectively recover the cluster structure\cite{yu2024non}. 

\textbf{Modularity.} Modularity \cite{newman2006modularity} quantifies the deviation between the actual cluster structure and the expected structure generated by random combinations in the null model. Modularity matrix $\mathbf{B}$ is defined as $\mathbf{B} = \mathbf{A} - \frac{\mathbf{d} \mathbf{d}^\top}{2m}$, 
where $\mathbf{A}$ is the adjacency matrix, $\mathbf{d}$ is the degree vector and $m$ is the total number of edges in the graph. The entry $b_{ij} \in \mathbf{B}$ measures connection strength between nodes $i$ and $j$.

\textbf{Modularity-based Graph clustering.} Graph clustering relies on graph pooling and node assignment.
For each view, we construct graph structures with $k$-nearest neighbors (KNN), then deploy graph convolutional network (GCN) $\mathcal{G}_v(\cdot)$ to aggregate graph embeddings:
\begin{equation}
    \begin{split}
    \mathcal{G}_v(\mathbf{Z}^{v};\mathbf{A}^{v}) = \sigma ( {\mathbf{D}^{v}}^{-\frac{1}{2}} {\mathbf{A}^{v}} {\mathbf{D}^{v}}^{-\frac{1}{2}} \mathbf{Z}^{v} \mathbf{W}^{v} 
    + \mathbf{Z}^{v} \mathbf{W}^{v}_s ),
    \end{split}
\label{eq:GCN}
\end{equation}
\noindent where $\sigma(\cdot)$ is an activation function used to reinforce the nonlinear aggregation capability of GCN. We replace a skip connection matrix $\mathbf{W}^{{v}}_s$ with the self-loop $\mathbf{A}^{{v}}+\mathbf{I}$ to alleviate the over-smoothing of node embeddings. Graph pooling operates on the latent representations $\mathbf{Z}^{{v}}$, while the adjacency matrix $\mathbf{A}^{{v}}$ is constructed from the original observations ${\mathbf{X}^{{v}}}$:\vspace{-8pt}
\begin{equation}
\begin{split}
    {a}^{v}_{i,j} = 
    \begin{cases}
    1, & \text{if } (\mathbf{{x}}^{v}_i \ and \ \mathbf{{x}}^{v}_j \neq NaN) \,
    \& \\
    \, &(\mathbf{{x}}^{v}_i \in \mathcal{N}_k(\mathbf{{x}}^{v}_j) \text{ or } \mathbf{{x}}^{v}_j \in \mathcal{N}_k(\mathbf{{x}}^{v}_i)), \\
    0, & \text{otherwise},
    \end{cases}
\end{split}
\label{eq:KNN}
\end{equation}
\noindent where $\mathcal{N}_k(\mathbf{{x}}_i^{{v}})$ is a set including the $k$-nearest neighbors of $\mathbf{{x}}_i^{{v}}$.
It brings neighbors closer in the embedding space without introducing missing noise into the graph structure.

The softmax classifier offers the soft assignment $\mathbf{P}^{{v}}$ for node embeddings encoded by $\mathcal{G}_v(\mathbf{Z}^{{v}};{\mathbf{A}}^{{v}})$:
\begin{equation}
\begin{split}
    \mathbf{P}^{{v}} = \text{Softmax}(\mathcal{G}_v(\mathbf{Z}^{{v}};{\mathbf{A}}^{{v}})).
\label{eq:P}
\end{split}
\end{equation}

To obtain a well-separated cluster assignment $\mathbf{P}^v$, we maximize the spectral modularity $\text{Tr}\left((\mathbf{P}^{{v}})^\top \mathbf{B}^{{v}} \mathbf{P}^{{v}}\right)$ \cite{brandes2006maximizing,kernighan1970efficient}, which captures intra-cluster connections and inter-cluster margins by mapping $\mathbf{P}^{{v}}$ to modularity $\mathbf{B}^{{v}}$; To further improve the confidence and robustness of $\mathbf{P}^{{v}}$, we apply a self-supervised signal $\mathbf{L}^{{v}}$, learned from contrastive clustering as detailed in Sec. \ref{sec:3.3}, to guide $\mathbf{P}^{{v}}$ in chasing $\mathbf{L}^{{v}}$ via self-knowledge distillation. Specifically, the following objective is designed to optimize $\mathbf{P}^{{v}}$:
\begin{equation}
\ell^{{v}}_{m}(\mathbf{P}^{{v}};\mathbf{L}^{{v}}) = -\frac{1}{2m} \text{Tr}\left((\mathbf{P}^{{v}})^\top \mathbf{B}^{{v}} \mathbf{P}^{{v}}\right) + \lambda\ \text{KL}(\mathbf{L}^{{v}} \parallel \mathbf{P}^{{v}}), \label{eq: spectral modularity loss}
\end{equation}

\noindent where Kullback-Leibler divergence $\text{KL}(\cdot)$ serves as a robust regularizer, leveraging $\lambda$ to regulate the information flow of $\mathbf{L}^{{v}} \in \mathbb{R}^{N \times K}$ to ensure stable cluster performance. We project semantic representations $\mathbf{H}^v$ learned through contrastive clustering onto their $k$-means prototypes and introduce Student’s t-distribution \cite{xie2016unsupervised} as kernel to predict high-confidence labels $\mathbf{L}^{{v}}$ with a nearly uniform distribution:
\begin{equation}
l^{{v}}_{i,k} = \frac{ \left( 1 + \frac{\|\mathbf{h}^{{v}}_i - \mathbf{c}^{{v}}_k\|^2}{\gamma} \right)^{-\frac{\gamma+1}{2}} }{ \sum_{k'} \left( 1 + \frac{\|\mathbf{h}^{{v}}_i - \mathbf{c}^{{v}}_{k'}\|^2}{\gamma} \right)^{-\frac{\gamma+1}{2}}} \in \mathbf{L}^{{v}}.
\label{eq:predict label}
\end{equation}

We conduct graph clustering on each view separately, then optimize assignments across all views concurrently:
\begin{equation}
\renewcommand{\arraystretch}{0.1}
\mathcal{L}_{gc} = \sum_{v=1}^{V}  \ell^{{v}}_{{m}}(\mathbf{P}^{{v}}, \mathbf{L}^{{v}}) .
\label{eq:gl loss}
\end{equation}

\subsection{Clustering Driven by Consensus Semantics}   
Our FreeCSL learns the semantic representations $\mathbf{H}^v$ end-to-end by jointly minimizing the reconstruction loss, contrastive clustering loss and graph clustering loss:
\begin{equation}
\renewcommand{\arraystretch}{0.1}
\mathcal{L} = \mathcal{L}_{{rec}} + \mathcal{L}_{cc} + \mathcal{L}_{gc}.
\label{eq:ovrall Loss}
\end{equation} 

Without searching for the optimal balancing weights of three loss terms, outstanding performance is easily achieved by applying $k$-means to the consensus semantic representations $\mathbf{H}$ learned via the fusion manner $\mathbb{T}(\{\mathbf{H}^{{v}}\}_{v=1}^{V})$.

The implementation for FreeCSL consists of two stages: warm-up training for the encoder-decoder, followed by fine-tuning for semantic representation learning and clustering. For details, refer to Algorithm \ref{Alg1}.

\begin{algorithm}[htp]
\SetAlgoNlRelativeSize{-1}
\SetAlgoLongEnd
\caption{FreeCSL for Learning and Clustering}
\label{Alg1}
\textbf{Input}: Complete, incomplete and pair-observed multi-view dataset
$\{\mathbf{X}^{{v}}\}_{v=1}^V$, $\{\tilde{\mathbf{X}}^{{v}}\}_{v=1}^V$, $ \{\overline{\mathbf{X}}^{{m,n}}\}_{m\neq n}^V$; networks $\{\mathcal{E}_{v},\mathcal{D}_{v},\mathcal{G}_{v}\}_{v=1}^{V}$; warm-up and fine-tuning epochs $e$, $E$. \\
\For{ $t = 1$ to $e$}{On $\{\tilde{\mathbf{X}}^{{v}}\}_{v=1}^{V}$, warming up $\{\mathcal{E}_{v},\mathcal{D}_{v}\}_{v=1}^{V}$ with Eq.\eqref{eq:rec Loss}.
}

\For{$t = 1$ to $E$}{
    Obtain latent representations $\{\mathbf{Z}^{{v}}|\mathcal{E}_{v}: \mathbf{{X}}^{{v}}  \rightarrow \mathbf{Z}^{{v}}\}_{v=1}^{V}$;\\
    Construct consensus prototypes $\mathbf{C}=\{\mathbf{c}_k\}_{k=1}^{K}$ as elaborated in Sec. \ref{sec:3.3}.\\
    Compute the reconstruction loss $\mathcal{L}_{rec}$ with Eq.\eqref{eq:rec Loss}. \\
    \tcp*[h]{Cross-view Contrastive Clustering}\\
    \For{$m,n = 1$ to $V$ $(m \neq n)$}
    {Compute the cluster probability $\mathbf{p}_i^{{m}}$,$\mathbf{p}_i^{{n}}$ with Eq.\eqref{eq: the probability of representation imputated to prototype};\\
    Solve consensus pseudo-label $\mathbf{q}_i^{{m}}$ with Eq.\eqref{eq:optimal transport};\\
    Get swapped loss $\mathbf{\ell}^{{m,n}}_{{cc}}$ with Eq.\eqref{eq: cl loss} on $\overline{\mathbf{X}}^{{m,n}}$.}
    Calculate contrastive clustering loss $\mathcal{L}_{cc}$ with Eq.\eqref{eq:total cl loss}\\
    \tcp*[h]{Within-view Graph Clustering}\\
    \For{$v = 1$ to $V$}{
    \ \ Construct adjacency matrix $\mathbf{A}^{{v}}$ with Eq.\eqref{eq:KNN}.\\
    Solve node assignment $\mathbf{P}^{{v}}$ with Eq.\eqref{eq:P}\\
    Solve pseudo-labels $\mathbf{L}^{{v}}$ with Eq.\eqref{eq:predict label};\\
    Compute KL-modularity loss $\mathcal{\ell}^{{v}}_{{m}}$ with Eq.\eqref{eq: spectral modularity loss}
    }
    Calculate graph clustering loss $\mathcal{L}_{{gc}}$ with Eq.\eqref{eq:gl loss}\\
    \tcp*[h]{Semantic Representation Learning}\\
    Calculate the overall loss $\mathcal{L}$ with Eq.\eqref{eq:ovrall Loss};\\
    Optimize $\{\mathcal{E}_{v},\mathcal{D}_{v},\mathcal{G}_{v}\}_{v=1}^{V}$ to minimize $\mathcal{L}$;\\
    Learn latent and consensus representations $ \{\mathbf{Z}^{{v}}\}_{v=1}^{V}$ and $ \mathbf{Z}$, semantic representations $\{\mathbf{H}^{{v}}\}_{v=1}^{V}$, consensus semantic representations $\mathbf{H}$, consensus prototypes $\mathbf{C}$.}
\textbf{Output}: $\hat{\mathbf{Y}} = \{{\hat{y}_i\}}^N_{i=1}$ predicted by $k$-means on $\mathbf{H}$.
\end{algorithm}

\section{Experiment}\label{sec:experiment}

\subsection{Experimental Settings}\label{experi:set}
We select four representative datasets namely Caltech-5V\cite{MFLVC}, ALOI-100\cite{du2021deep}, YoutubeFace10\cite{huang2023fast} and NoisyMNIST\cite{lin2021completer} to demonstrate our model’s performance in comparison with seven SOTA methods summarized in Table \ref{tab:methods}. To comprehensively evaluate experimental results, three metrics are adopted: accuracy (ACC), normalized mutual information
(NMI), and adjusted rand index (ARI). 
\begin{table}[!htbp]
\caption{SOTA methods categorized by the types of techniques for consistency, imputation, and alignment.}
\label{tab:methods}
\centering
\resizebox{0.49\textwidth}{!}{
\begin{tabular}{cccc}
\toprule
Competitors & Consistency & Imputation & Alignment \\
\midrule
{CPM-Nets (TPAMI'20)}       & instance-level       & mutual information interaction & \textbackslash    \\
{COMPLETER (CVPR'21)}       & instance-level       & mutual information interaction & \textbackslash     \\
{DIMVC (AAAI'22)}          & instance-level       & \textbackslash & assignment-based      \\
{SURE (TPAMI'23)}          & cluster-level        & graph structure transfer & \textbackslash  \\
{ProImp (IJCAI'23)}        & instance-level       & sample-prototype relationship inheritance & prototype-based\\
{ICMVC (AAAI'24)}          & instance-level       & graph structure transfer & assignment-based\\
{DIVIDE (AAAI'24)}         & cluster-level        & mutual information interaction & \textbackslash \\
\bottomrule
\end{tabular}}
\end{table}

\subsection{Implementation details}
Models are trained on an NVIDIA 3090 GPU using Adam optimizer in PyTorch 2.1.0. The warming up and fine-tuning learning rates are 0.0003 and 0.0005, with a batch size of 512. The hyperparameters are set as follows for different datasets: smoothness $\alpha=0.5$, temperature $\tau \in \{0.1, 0.2\}$, neighbors $\zeta=3$, and regularizer weight $\lambda \in \{0.05, 0.1, 0.2, 0.3\}$. For all datasets, 4-layer autoencoders and 2-layer GCNs with same MLP structures are used for each view. The layers of contrastive clustering and graph clustering are shared across all views via a single FC layer.

\begin{table*}[!htbp]
\centering
\caption{Performance comparison on four multi - view benchmarks. The best and second - best results are highlighted in \textcolor{red}{red} and \textcolor{blue}{blue}.}
\vspace{-5pt}
\renewcommand{\arraystretch}{0.9} 
\resizebox{1\textwidth}{!}{
\begin{tabular}{@{\hspace{10pt}}cc|ccc|ccc|ccc|ccc@{\hspace{10pt}}}
\toprule
\multirow{2}{*}{} & Missing rates & \multicolumn{3}{c|}{$r = 0.1$} & \multicolumn{3}{c|}{$r = 0.3$} & \multicolumn{3}{c|}{$r = 0.5$} & \multicolumn{3}{c}{$r = 0.7$} \\ 
\cmidrule(lr){2 - 14}
& Metrics & ACC (\%) & NMI (\%) & ARI (\%) & ACC (\%) & NMI (\%) & ARI (\%) & ACC (\%) & NMI (\%) & ARI (\%) & ACC (\%) & NMI (\%) & ARI (\%) \\
\midrule
\multirow{9}{*}{\rotatebox{90}{\textbf{Caltech - 5V}}}
  & CPM - Nets & {85.72} & 74.62 & 71.92 &\textcolor{blue}{85.43} & 72.90 & \textcolor{blue}{70.84} & 81.08 & 68.61 & 66.08 & \textcolor{blue}{79.50} & 65.52 & \textcolor{blue}{62.70} \\
 & COMPLETER & 58.68 & 67.69 & 42.98 & 65.07 & 68.75 & 47.10 & 64.46 & 66.09 & 42.80 & 69.00 & 67.46 & 51.49 \\
 & DIMVC & \textcolor{blue}{85.86} & \textcolor{blue}{80.01} & \textcolor{blue}{74.63} & 78.90 & \textcolor{blue}{78.76} & 69.55 & 80.01 & 70.08 & \textcolor{blue}{66.35} & {{69.86}} & {{60.16}} & {{48.70}} \\
 & SURE & 75.64 & 69.12 & 62.83 & 77.89 & 67.89 & 61.15 & 77.22 & 65.99 & 60.99 & 73.96 & 60.89 & 53.66 \\
 & ProImp & 66.62 & 57.66 & 47.54 & 72.17 & 61.70 & 52.49 & 65.28 & 55.77 & 46.22 & 60.36 & 49.23 & 39.63 \\
 & ICMVC & 76.86 & 73.92 & 67.37 & 76.93 & 73.39 & 66.30 & 79.29 & \textcolor{blue}{72.67} & {66.21} & 72.21 & 66.63 & 58.18 \\
 & DIVIDE & 57.29 & 47.00 & 33.23 & 68.14 & 58.93 & 44.16 & \textcolor{blue}{81.57} & 71.45 & 59.20 & 74.36 & \textcolor{blue}{68.62} & 60.14 \\
 & \textbf{Ours} & \textcolor{red}{91.57} & \textcolor{red}{85.32} & \textcolor{red}{83.14} & \textcolor{red}{88.86} & \textcolor{red}{81.26} & \textcolor{red}{78.68} & \textcolor{red}{88.36} & \textcolor{red}{80.01} & \textcolor{red}{78.02} & \textcolor{red}{83.64} & \textcolor{red}{71.79} & \textcolor{red}{68.06} \\
\midrule
\multirow{9}{*}{\rotatebox{90}{\textbf{ALOI - 100}}}
 & CPM - Nets & 63.95 & 79.14 & 51.26 & 52.30 & 70.74 & 39.79 & 42.74 & 62.46 & 28.40 & 30.52 & 54.03 & 17.18 \\
 & COMPLETER & 21.86 & 46.94 & 12.10 & 22.04 & 46.45 & 11.62 & 19.38 & 44.81 & 10.78 & 17.18 & 43.68 & 9.60 \\
 & DIMVC & 56.90 & 75.00 & 35.59 & 50.55 & 72.51 & 28.10 & 52.54 & 71.45 & 32.56 & 48.97 & 70.80 & 26.16 \\
 & SURE & 64.59 & 77.74 & 52.61 & 60.22 & 74.76 & 47.21 & \textcolor{blue}{60.65} & \textcolor{blue}{74.84} & \textcolor{blue}{47.39} & \textcolor{blue}{50.23} & \textcolor{blue}{68.91} & \textcolor{blue}{37.48} \\
 & ProImp & 26.15 & 52.94 & 31.40 & 24.52 & 53.00 & 28.18 & 27.97 & 58.10 & 33.42 & 20.33 & 50.10 & 24.28 \\
 & ICMVC & 55.48 & 74.04 & 42.91 & 48.83 & 67.46 & 36.50 & 35.42 & 56.98 & 24.19 & 26.84 & 47.43 & 16.63 \\
 & DIVIDE & \textcolor{blue}{76.01} & \textcolor{blue}{88.59} & \textcolor{blue}{70.59} & \textcolor{blue}{61.90} & \textcolor{blue}{80.53} & \textcolor{blue}{53.29} & 52.56 & 73.17 & 40.57 & 42.66 & 67.82 & 30.93 \\
 & \textbf{Ours} & {\textcolor{red}{91.13}} & {\textcolor{red}{95.22}} & {\textcolor{red}{87.59}} & {\textcolor{red}{88.28}} & {\textcolor{red}{93.34}} & {\textcolor{red}{83.95}} & {\textcolor{red}{84.56}} & {\textcolor{red}{90.53}} & {\textcolor{red}{78.90}} & {\textcolor{red}{76.44}} & {\textcolor{red}{85.00}} & {\textcolor{red}{67.68}} \\
\midrule
\multirow{9}{*}{\rotatebox{90}{\textbf{YouTubeFace10}}}
 & CPM - Nets & 75.48 & 78.56 & 66.50 & 67.50 & 75.82 & 61.15 & 66.60 & 75.19 & 62.01 & 67.32 & 74.10 & 62.54 \\
 & COMPLETER & 45.72 & 52.04 & 21.03 & 50.24 & 52.81 & 27.77 & 36.41 & 35.24 & 15.18 & 40.98 & 39.92 & 20.82 \\
 & DIMVC & \textcolor{blue}{81.77} & \textcolor{blue}{81.32} & \textcolor{blue}{74.59} & \textcolor{blue}{71.96} & \textcolor{blue}{76.42} & \textcolor{blue}{62.29} & \textcolor{blue}{73.87} & \textcolor{blue}{75.44} & \textcolor{blue}{63.97} & \textcolor{blue}{74.94} & \textcolor{blue}{79.34} & \textcolor{blue}{68.96} \\
 & SURE & 26.65 & 23.60 & 11.41 & 26.43 & 22.40 & 10.55 & 31.84 & 28.77 & 15.17 & 29.37 & 27.48 & 12.85 \\
 & ProImp & 48.97 & 53.20 & 33.83 & 43.30 & 49.04 & 26.53 & 39.67 & 50.18 & 26.68 & 49.52 & 51.47 & 29.12 \\
 & ICMVC & 68.90 & 74.17 & 60.06 & 62.76 & 65.44 & 53.03 & 44.56 & 52.51 & 32.51 & 27.44 & 27.09 & 15.33 \\
 & DIVC & 31.85 & 29.64 & 11.35 & 35.86 & 30.81 & 12.31 & 32.74 & 30.82 & 12.74 & 35.17 & 33.47 & 15.93 \\
 & \textbf{Ours} & {\textcolor{red}{82.93}} & {\textcolor{red}{83.55}} & {\textcolor{red}{74.76}} & {\textcolor{red}{80.77}} & {\textcolor{red}{81.46}} & {\textcolor{red}{71.62}} & {\textcolor{red}{80.19}} & {\textcolor{red}{81.07}} & {\textcolor{red}{71.37}} & {\textcolor{red}{76.62}} & {\textcolor{red}{81.31}} & {\textcolor{red}{73.22}} \\
\midrule
\multirow{9}{*}{\rotatebox{90}{\textbf{NoisyMNIST}}}
 & CPM - Nets & 48.64 & 44.08 & 32.00 & 42.88 & 37.01 & 26.92 & 50.52 & 43.54 & 33.56 & 46.46 & 37.18 & 27.04 \\
 & COMPLETER & 76.55 & 87.95 & 77.18 & 78.16 & 86.11 & 76.83 & 64.79 & 65.77 & 54.68 & 41.15 & 40.60 & 27.58 \\
 & DIMVC & 55.88 & 54.78 & 41.29 & 56.77 & 58.29 & 43.51 & 62.28 & 62.66 & 48.01 & 65.94 & 67.43 & 54.58 \\
 & SURE & 98.61 & 95.89 & 96.96 & 97.11 & \textcolor{blue}{93.95} & \textcolor{blue}{91.64} & \textcolor{blue}{94.85} & \textcolor{blue}{89.06} & \textcolor{blue}{87.89} & \textcolor{blue}{84.00} & \textcolor{blue}{78.25} & \textcolor{blue}{76.02} \\
 & ProImp & 98.30 & 95.31 & 96.35 & 93.29 & 86.05 & 85.38 & 78.16 & 69.45 & 64.37 & 50.04 & 41.33 & 33.34 \\
 & ICMVC & \textcolor{blue}{98.78} & \textcolor{blue}{96.36} & \textcolor{blue}{97.35} & \textcolor{red}{97.75} & 93.71 & 95.11 & 81.64 & 79.24 & 75.47 & 53.91 & 50.91 & 44.56 \\
 & DIVIDE & 94.72 & 91.44 & 89.43 & 95.85 & 89.71 & 91.06 & 57.09 & 57.29 & 46.28 & 28.57 & 25.61 & 11.14 \\
 & \textbf{Ours} & {\textcolor{red}{99.13}} & {\textcolor{red}{97.23}} & {\textcolor{red}{98.10}} & \textcolor{blue}{{97.68}} & {\textcolor{red}{93.94}} & {\textcolor{red}{94.94}} & {\textcolor{red}{96.04}} & {\textcolor{red}{89.81}} & {\textcolor{red}{91.48}} & {\textcolor{red}{92.19}} & {\textcolor{red}{82.50}} & {\textcolor{red}{83.56}} \\
\bottomrule
\end{tabular}}
\label{tab:performance}
\end{table*}

\subsection{Competitiveness of FreeCSL}\label{sec:experi:results}
FreeCSL is compared with seven SOTA methods across three metrics in Table \ref{tab:performance}. The results indicate that FreeCSL excellently handles challenges of high missing rates, multiple clusters, and large-scale issues in IMVC:
\begin{itemize}
\item From the perspective of effectiveness, FreeCSL surpasses most SOTA models, particularly on challenging datasets like ALOI-100 (100 clusters) and YouTubeFace10 (30,000 samples).It achieves significant improvements in ACC, with gains of 15.12\%, 26.38\%, 23.91\%, 26.21\% for ALOI-100, and 1.16\%, 8.81\%, 6.32\%, 1.68\% for YouTubeFace10, at missing rate $r=0.1, 0.3, 0.5, 0.7$.

\item From the perspective of robustness, Fig. \ref{fig:performance} (visualized from Table \ref{tab:performance}) shows our model's stability as $r$ increases. FreeCSL declines gradually, unlike other models with sharp ACC drops, highlighting its robustness in IMVC task. On the small dataset Caltech-5V with $r=0.7$, its Acc remains at 84.56\%, while on the larger dataset NoisyMNIST, also achieves 92.19\%.
\end{itemize} 

The above phenomenon can be explained as follows: As $r$ increases, paired observations become scarce. Methods like instance-level consistency learning or cross-view imputation and alignment (\emph{e.g.,} CPM-Net, COMPLETER, ICMVC), constrained by their dependence on consistency information, will introduce biases such as false negatives and semantic inconsistencies, resulting in degraded performance. For the complex task ALOI-100, cluster-level methods like SURE and DIVIDE surpass instance-level by exploiting cross-view cluster consistency. But, their reliance on cross-view graphs or mutual information for imputation, while neglecting within-view cluster affiliations, restricts performance. ProImp acknowledges the need to integrate within-view and cross-view information for semantically consistent imputation. However, its prototype alignment accumulates errors, particularly with a large number of clusters, achieving only 20\% ACC on ALOI-100.

The superiority of FreeCSL stems from the synergistic effect of the CSL and CSE modules, which eliminate semantic gaps and ensure consistent assignments for all observations in a shared semantic space.

\subsection{Understanding FreeCSL}\label{experi:visual}
\begin{table*}[!htbp]
\centering
\caption{Ablation study on Caltech-5V and ALOI-100. $\checkmark$ denotes FreeCSL with
the component and the best results are highlighted in \textcolor{red}{red}.}
\vspace{-6pt} 
\small 
\renewcommand{\arraystretch}{0.9} 
\resizebox{1\textwidth}{!}{
\begin{tabular}{@{\hspace{8pt}}cccc|ccc|ccc|ccc|ccc@{\hspace{8pt}}} 
\toprule
\noalign{\vspace{-2pt}} 
\multirow{2}{*}{} & \multicolumn{3}{c|}{Components} & \multicolumn{3}{c|}{$r=0.1$} & \multicolumn{3}{c|}{$r=0.3$} & \multicolumn{3}{c|}{$r=0.5$} & \multicolumn{3}{c}{$r=0.7$} \\
\noalign{\vspace{-2pt}} \cmidrule(lr){2-16}
\noalign{\vspace{-2pt}}
& $\mathcal{L}_{rec}$ & $\mathcal{L}_{cc}$ & $\mathcal{L}_{gc}$ & ACC (\%) & NMI (\%) & ARI (\%) & ACC (\%) & NMI (\%) & ARI (\%) & ACC (\%) & NMI (\%) & ARI (\%) & ACC (\%) & NMI (\%) & ARI (\%) \\
\noalign{\vspace{-2pt}}\\\midrule

\multirow{4}{*}{\rotatebox{90}{\textbf{Caltech-5V}}}
 & \checkmark &   &   & 75.00 & 65.64 & 58.89 & 69.00 & 57.84 & 51.49 & 52.21 & 44.18 & 36.71 & 51.93 & 43.42 & 35.75 \\
 & \checkmark & \checkmark &   & 85.50 & 78.07 & 72.28 & 82.00 & 74.53 & 67.67 & 85.36 & 74.32 & 71.10 & 79.14 & 67.81 & 62.78 \\
 & \checkmark &   & \checkmark & 81.07 & 70.45 & 66.13 & 80.86 & 69.61 & 67.49 & 74.14 & 59.39 & 55.90 & 70.71 & 53.25 & 49.95 \\
 & \checkmark & \checkmark & \checkmark & \textcolor{red}{91.57} & \textcolor{red}{85.32} & \textcolor{red}{83.14} & \textcolor{red}{88.86} & \textcolor{red}{81.26} & \textcolor{red}{78.68} & \textcolor{red}{88.36} & \textcolor{red}{80.01} & \textcolor{red}{78.02} & \textcolor{red}{83.64} & \textcolor{red}{71.79} & \textcolor{red}{68.06} \\
\midrule
\multirow{4}{*}{\rotatebox{90}{\textbf{ALOI-100}}}
 & \checkmark &   &   & 63.81 & 77.92 & 49.38 & 44.69 & 65.64 & 31.11 & 32.20 & 55.83 & 18.94 & 23.78 & 48.90 & 11.31 \\
 & \checkmark & \checkmark &   & 87.69 & 93.64 & 84.12 & 85.25 & 91.02 & 79.45 & 80.23 & 88.04 & 74.15 & 75.29 & 83.38 & 65.72 \\
 & \checkmark &   & \checkmark & 81.06 & 85.96 & 71.21 & 64.13 & 75.87 & 51.49 & 47.52 & 65.59 & 34.29 & 32.73 & 56.35 & 19.23 \\
 & \checkmark & \checkmark & \checkmark & \textcolor{red}{91.13} & \textcolor{red}{95.22} & \textcolor{red}{87.59} & \textcolor{red}{88.28} & \textcolor{red}{93.34} & \textcolor{red}{83.95} & \textcolor{red}{84.56} & \textcolor{red}{90.53} & \textcolor{red}{78.90} & \textcolor{red}{76.44} & \textcolor{red}{85.00} & \textcolor{red}{67.68} \\
\bottomrule
\end{tabular}}
\label{tab:ablation study}
\end{table*}

\textbf{Ablation Study.} 
As shown in Table \ref{tab:ablation study}, ablation studies on FreeCSL's three components reveal that the CSL module contributes the most. On ALOI-100, it demonstrates exceptional performance and strong stability, with ACC dropping only 12.4\% as $r$ rises from $0.1$ to $0.7$. This is attributed to CSL's establishment of a shared semantic space via prototype-based semantic contrast, reducing semantic gaps within clusters. The CSE module also plays a positive role in discovering tighter cluster structures by optimizing modularity. Thanks to its enhanced semantics, both the REC and CSL modules achieve notable improvements. 

\begin{table*}[!htbp]
\caption{Imputation- and alignment-free study on Caltech-5V and ALOI-100. ILR and ISR are filled with K-NN imputation via cross-view graph for latent representations and semantic representations $\mathbf{Z}^{(v)}$, $\mathbf{H}^{(v)}$. The best results are highlighted in \textcolor{red}{red}.}
\label{tab:Imputation study}
\vspace{-5pt}
\renewcommand{\arraystretch}{0.95} 
\centering
\resizebox{0.99\textwidth}{!}{
\begin{tabular}{@{\hspace{8pt}}cc|ccc|ccc|ccc|ccc@{\hspace{8pt}}} 
\toprule
\multirow{2}{*}{} & Missing rates & \multicolumn{3}{c|}{$r=0.1$} & \multicolumn{3}{c|}{$r=0.3$} & \multicolumn{3}{c|}{$r=0.5$} & \multicolumn{3}{c}{$r=0.7$} \\
\cmidrule(lr){2-14}
& Metrics & ACC (\%) & NMI (\%) & ARI (\%) & ACC (\%) & NMI (\%) & ARI (\%) & ACC (\%) & NMI (\%) & ARI (\%) & ACC (\%) & NMI (\%) & ARI (\%) \\
\midrule

\multirow{3}{*}{\rotatebox{90}{\scriptsize\textbf{Caltech-5V}}}
& ILR & \textcolor{red}{91.71} & \textcolor{red}{85.76} & \textcolor{red}{83.44} & 89.14 & 81.69 & 79.07 & {88.79} & \textcolor{red}{80.79} & 78.37 & 77.86 & 65.64 & 55.12 \\
& ISR & 91.64 & 85.55 & 83.30 & \textcolor{red}{89.43} & \textcolor{red}{82.26} & \textcolor{red}{79.62} & \textcolor{red}{88.86} & 80.78 & \textcolor{red}{78.74} & 81.29 & 66.68 & 63.59 \\
& FreeCSL & 91.57 & 85.32 & 83.14 & 88.86 & 81.26 & 78.68 & 88.36 & 80.01 & 78.02 & \textcolor{red}{83.64} & \textcolor{red}{71.79} & \textcolor{red}{68.06} \\

\midrule
\multirow{3}{*}{\rotatebox{90}{\scriptsize\textbf{ALOI-100}}}
& ILR & 58.92 & 79.67 & 41.54 & 53.46 & 75.88 & 36.89 & 45.83 & 69.47 & 27.41 & 38.08 & 58.96 & 21.35 \\
& ISR & 90.16 & \textcolor{red}{95.38} & \textcolor{red}{87.86} & \textcolor{red}{88.91} & \textcolor{red}{94.08} & \textcolor{red}{85.48} & 82.53 & 88.73 & 71.54 & 65.36 & 77.85 & 46.06 \\
& FreeCSL & \textcolor{red}{91.13} & 95.22 & 87.59 & 88.28 & 93.34 & 83.95 & \textcolor{red}{84.56} & \textcolor{red}{90.53} & \textcolor{red}{78.90} & \textcolor{red}{76.44} & \textcolor{red}{85.00} & \textcolor{red}{67.68} \\
\bottomrule
\end{tabular}}
\end{table*}

\begin{figure}[t]  
    \centering  
    \begin{subfigure}[t]{0.45\linewidth}  
        \includegraphics[width=\linewidth]{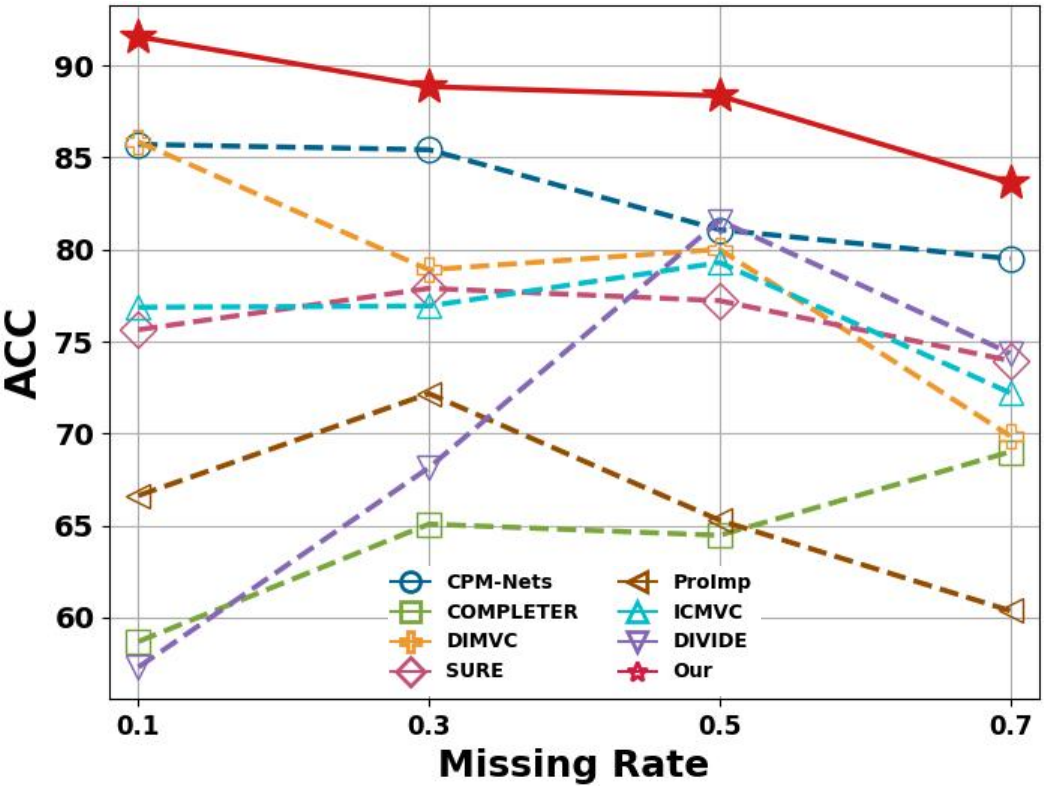}  
        \caption{Caltech-5V}  
    \end{subfigure}  
    \hspace{0.02\linewidth}
    \begin{subfigure}[t]{0.45\linewidth}  
        \includegraphics[width=\linewidth]{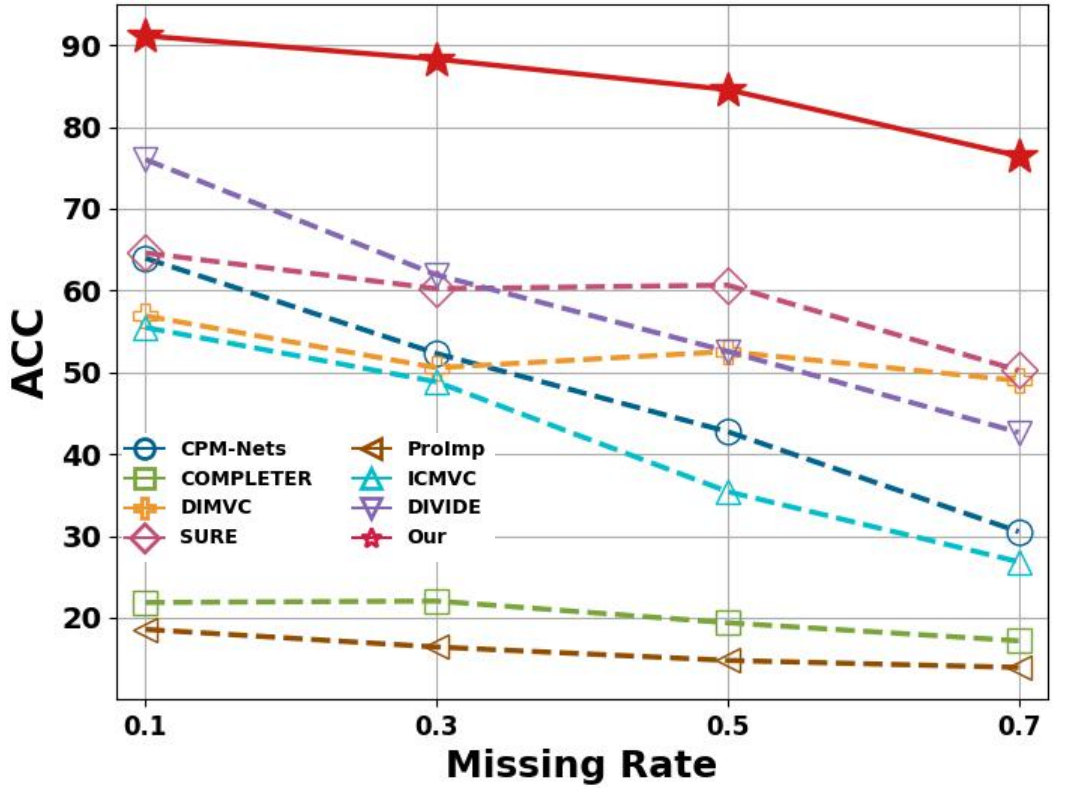}  
        \caption{ALOI-100}  
    \end{subfigure}  

    \vspace{0.1em} 

    \begin{subfigure}[t]{0.45\linewidth}  
        \includegraphics[width=\linewidth]{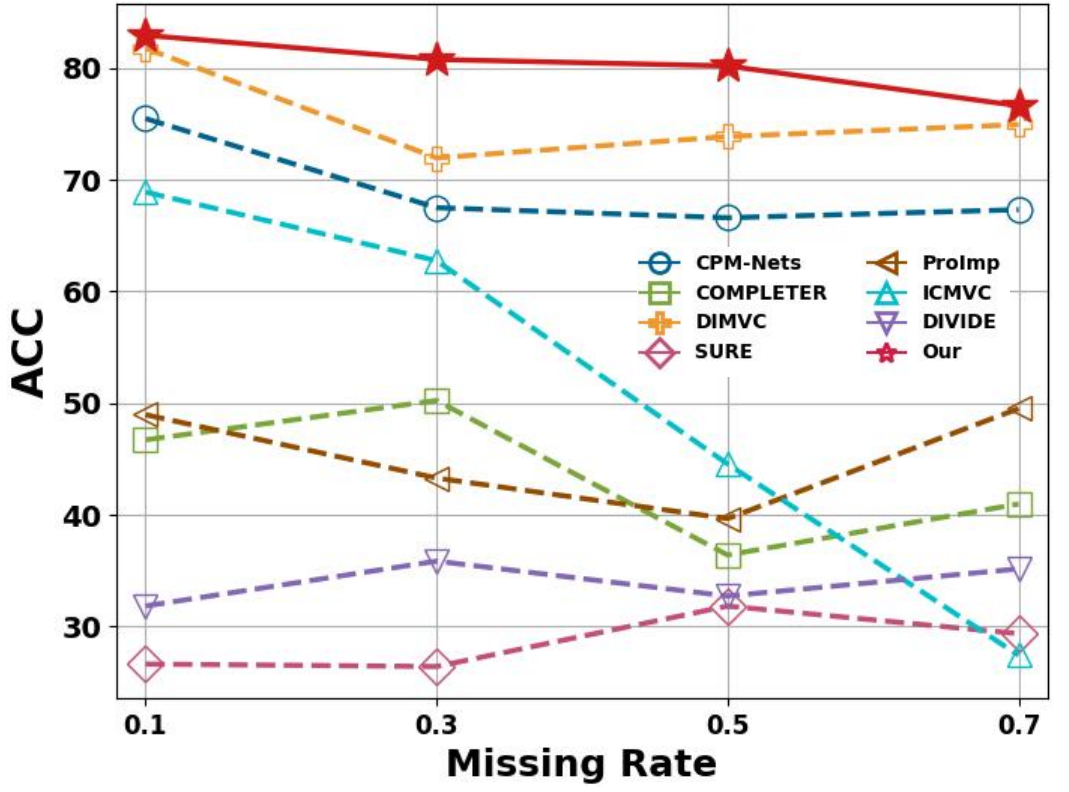} 
        \caption{YoutubeFace10}  
    \end{subfigure}  
    \hspace{0.02\linewidth}
    \begin{subfigure}[t]{0.45\linewidth}  
    \includegraphics[width=\linewidth]{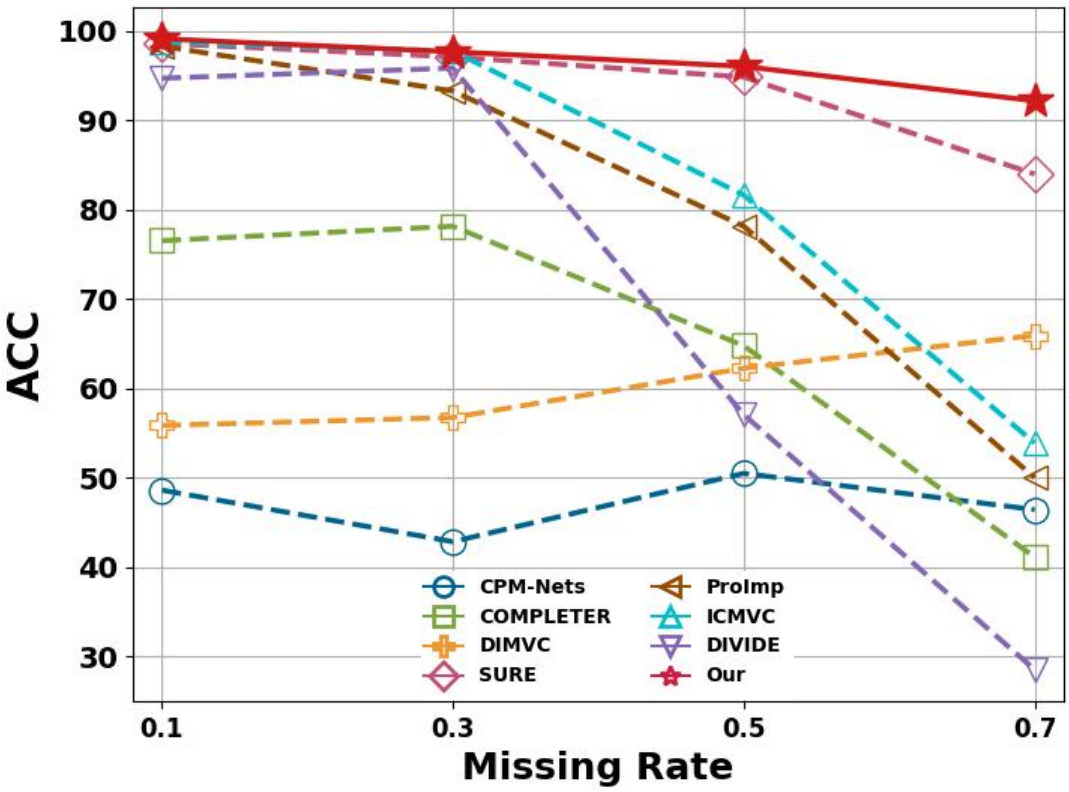}  
        \caption{NoisyMNIST}  
    \end{subfigure}  
    \vspace{-3pt}
    \caption{Visualization for Table \ref{tab:performance} based on metric ACC.}
    \vspace{-12pt}
\label{fig:performance}
\end{figure} 

\begin{figure}[htbp]
    \centering
    \begin{minipage}{0.5\textwidth}
        \begin{subfigure}{0.189\textwidth}
            \centering
            \includegraphics[width=\linewidth]{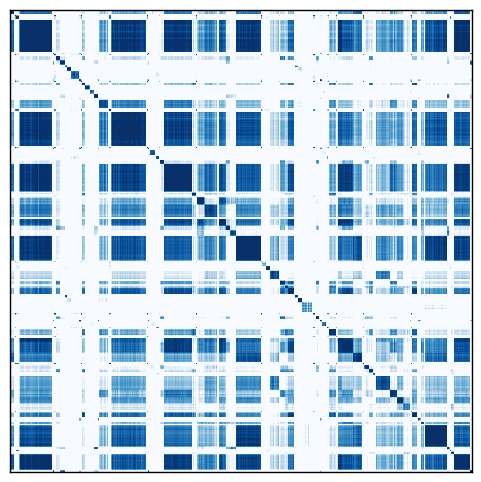}
            \caption{$\mathbf{Z}^{(1)}$}
        \end{subfigure}%
        \hfill
        \begin{subfigure}{0.189\textwidth}
            \centering
            \includegraphics[width=\linewidth]{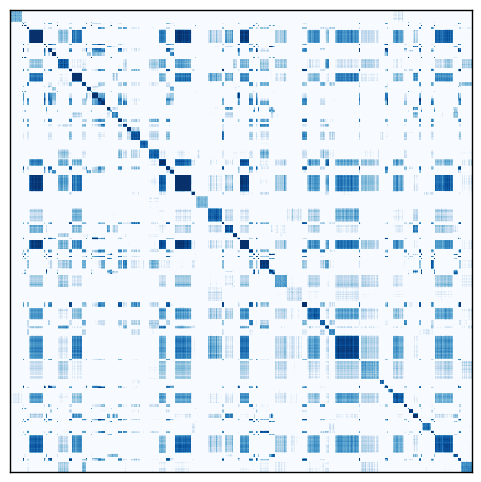}
            \caption{$\mathbf{Z}^{(2)}$}
        \end{subfigure}%
        \hfill
        \begin{subfigure}{0.189\textwidth}
            \centering
            \includegraphics[width=\linewidth]{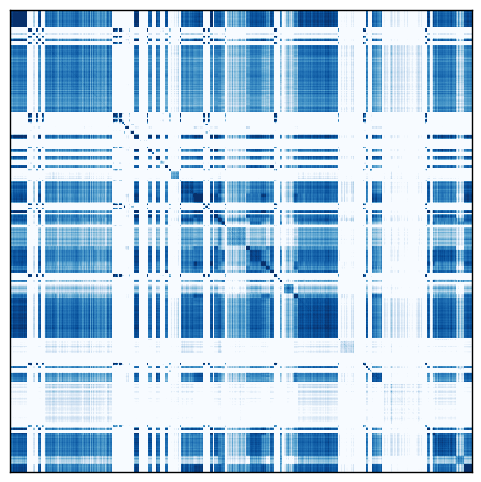}
            \caption{$\mathbf{Z}^{(3)}$}
        \end{subfigure}%
        \hfill
        \begin{subfigure}{0.189\textwidth}
            \centering
            \includegraphics[width=\linewidth]{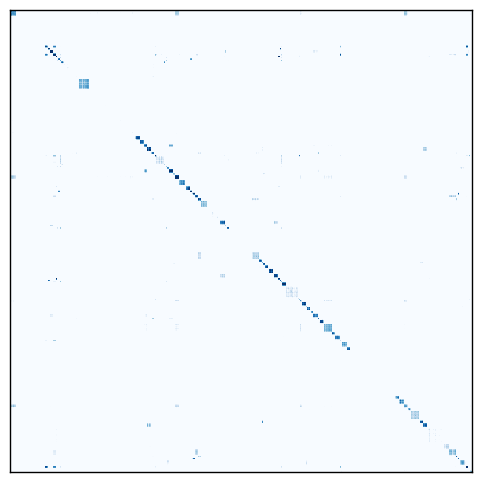}
            \caption{$\mathbf{Z}^{(4)}$}
        \end{subfigure}%
        \hfill
        \begin{subfigure}{0.19\textwidth}
            \centering
            \includegraphics[width=\linewidth]{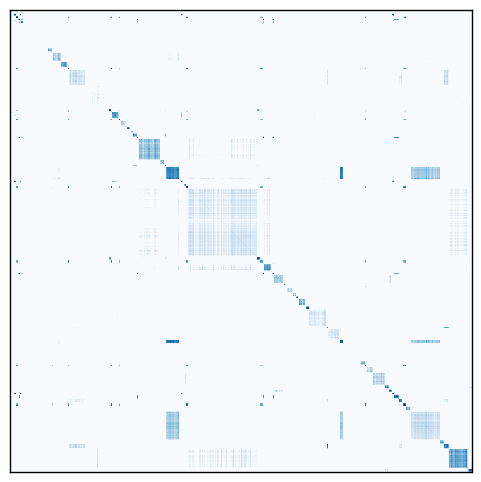}
            \caption{$\mathbf{Z}$}
        \end{subfigure}
    \end{minipage}
    \caption{Similarity matrices of $ \{\mathbf{Z}^{{v}}\}_{v=1}^{4}$, $\mathbf{Z}$ without consensus semantic learning on ALOI-100 with $r=0.5$.}
    \label{fig:one-row}
    \vspace{-10pt}
\end{figure}
\vspace{-2pt}
\begin{figure}[htbp]
    \centering
    \begin{minipage}{0.5\textwidth}
        \begin{subfigure}{0.19\textwidth}
            \centering
            \includegraphics[width=\linewidth]{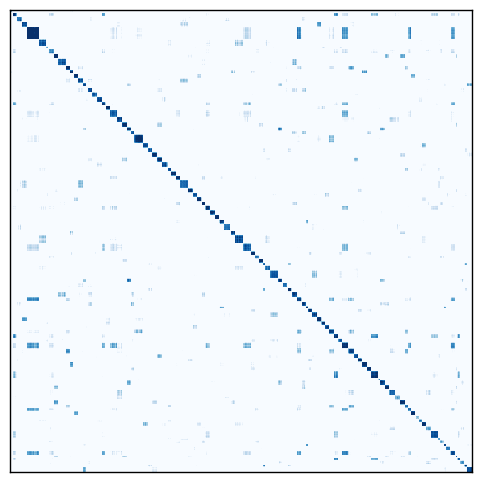}
            \caption{$\mathbf{H}^{(1)}$}
        \end{subfigure}%
        \hfill
        \begin{subfigure}{0.19\textwidth}
            \centering
            \includegraphics[width=\linewidth]{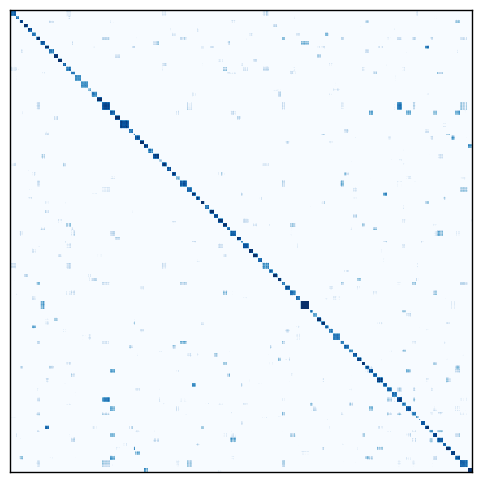}
            \caption{$\mathbf{H}^{(2)}$}
        \end{subfigure}%
        \hfill
        \begin{subfigure}{0.19\textwidth}
            \centering
            \includegraphics[width=\linewidth]{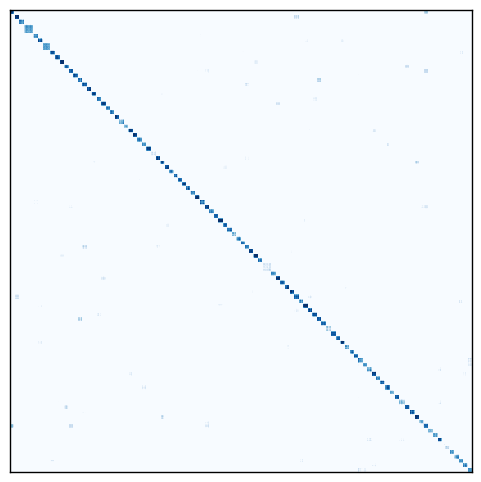}
            \caption{$\mathbf{H}^{(3)}$}
        \end{subfigure}%
        \hfill
        \begin{subfigure}{0.19\textwidth}
            \centering
            \includegraphics[width=\linewidth]{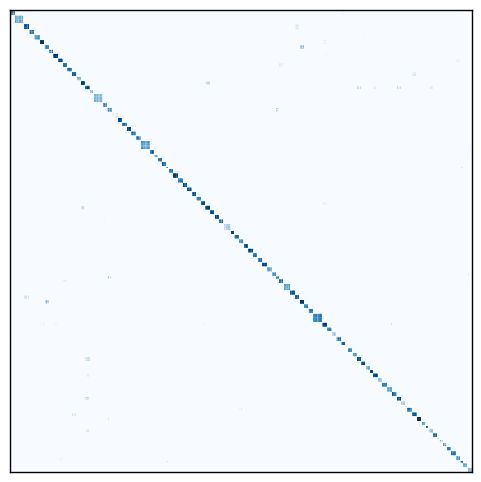}
            \caption{$\mathbf{H}^{(4)}$}
        \end{subfigure}%
        \hfill
        \begin{subfigure}{0.19\textwidth}
            \centering
            \includegraphics[width=\linewidth]{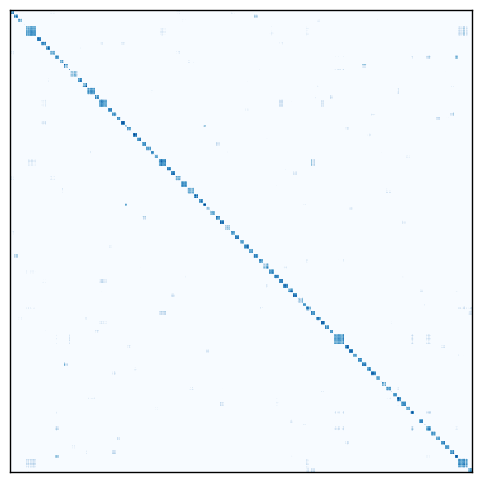}
            \caption{$\mathbf{H}$}
        \end{subfigure}
    \end{minipage}
    \vspace{-5pt}
    \caption{Similarity matrices of $ \{\mathbf{H}^{{v}}\}_{v=1}^{4}$, $\mathbf{H}$ with consensus semantic learning on ALOI-100 with $r=0.5$.}
    \label{fig:two-row}
\end{figure}

\textbf{Imputation- and Alignment-free CSL.} To verify FreeCSL can reduce semantic gaps and capture semantic relationships without requiring imputation or alignment, we set up two imputation experiments: one imputation for latent representations $\mathbf{Z}^{{v}}$ without consensus semantic learning named ILR, and the second for semantic representations $ \mathbf{H}^{{v}}$ learned from consensus semantic learning named ISR. They transfer complete $k$-nn graph structures from other views to incomplete views and utilize corresponding $k$ neighbors for imputation. The control groups perform $k$-means on the sum of all view matrices to predict cluster labels. 
Table \ref{tab:Imputation study} shows FreeCSL's robustness advantage at high $r$ on Caltech-5V, with notable performance disparity on ALOI-100. ILR, lacking consensus semantic learning, struggles with multi-cluster tasks due to cross-view semantic gaps.
ISR avoids semantic confusion via consensus semantic learning but underperforms FreeCSL due to biased imputation as consistency information decreases at higher $r$. FreeCSL, by contrast, integrates consistency and complementary information in consensus semantic representations for confident decisions.

To further illustrate how consensus semantic learning ``free up'' imputation and alignment, we construct cosine similarity matrices on ALOI-100 for view-specific latent representations $ \{\mathbf{Z}^{{v}}\}_{v=1}^{4}$ and their consensus $\mathbf{Z}$, as well as view-specific semantic representations $ \{\mathbf{H}^{{v}}\}_{v=1}^{4}$ and their consensus $\mathbf{H}$ in Fig. \ref{fig:two-row}, then analyze their information entropy. The similarity matrices of $ \{\mathbf{Z}^{{v}}\}_{v=1}^{4}$ are chaotic, with high uncertainty in intra- and inter- cluster relationships. The fusion manner $\mathbb{T}(\cdot)$ with Eq. (\ref{eq:consensus representations}) alleviates high entropy by incorporating view-specific complementary information. Compared to Fig. \ref{fig:one-row}, the similarity matrices of $ \{\mathbf{H}^{{v}}\}_{v=1}^{4}$ shows a clear block structure, with high intra-cluster similarity and distinct inter-cluster differences. This confirms $\mathbf{H}$, enhanced by view-specific information, reduces cross-view semantic gaps, while consensus semantic learning, capturing cluster semantic relationship, promotes high-confidence assignments as stated in Theorem \ref{thm:Theorem}.

\subsection{Analysis on FreeCSL}\label{experi:ablation}

\noindent\textbf{Convergence and Robustness Analysis.} 
In Fig. \ref{fig:Convergence}, we plot metrics and losses over training iterations, with error bands from 5 random Caltech-5V experiments to assess robustness. Both converge to stable values with minimal fluctuations and reach stability simultaneously. The gradual decrease in reconstruction loss indicates that the CSL and CSE modules are reasonable and beneficial, as they don't introduce significant discrepancies between semantic and latent representations, further underscoring the benefits of pre-training. Thanks to the harmonious collaboration of the three modules, our FreeCSL achieved satisfactory results within 30 iterations at minimal computational cost.

\noindent\textbf{Parameter Sensitivity Analysis.} 
Two hyperparameters in FreeCSL warrant investigation, namely, graph neighbors $\zeta$ in Eq. (\ref{eq:KNN}) and regularizer coefficient $\lambda$ in Eq. (\ref{eq: spectral modularity loss}). We expanded $\zeta$ and $\lambda$ to the range of 0.05 to 0.5 and 3 to 32 respectively. In Fig. \ref{fig:Sensitivity}, our model is insensitive to $\zeta$ but more affected by variations in $\lambda$. This is because the robust regularizer leverages semantic knowledge learned from the CSL module to guide the CSE model toward faster and more stable learning. A small $\lambda$ cannot provide effective consistency constraints while a large one makes CSL to dominate the learn direction, hindering CSE from exploring cluster structure information. Therefore, we set $\lambda \in [0.05, 0.2]$ to balance cooperative relationship between CSE and CSL. Without sacrificing ACC and NMI, we prefer to set the minimum $\zeta=3$ to reduce computational complexity. 
\begin{figure}[t!]
\includegraphics[width=8cm, height=4cm]{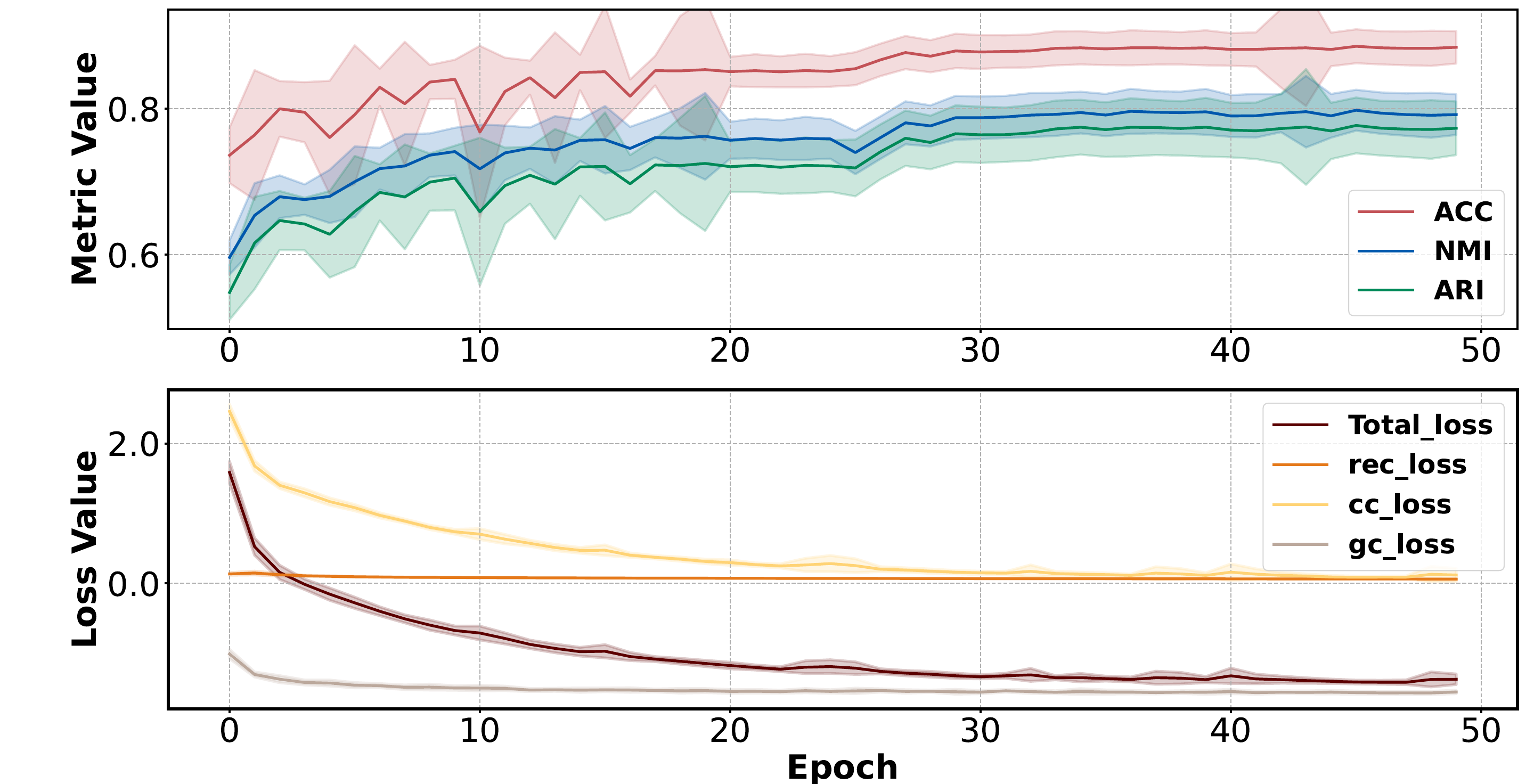}
\vspace{-5pt}
\caption{Convergence anlysis on Caltech-5V with $r=$ 0.5.}
\label{fig:Convergence}
\end{figure}

\begin{figure}[t]  
    \centering  
    \begin{subfigure}[t]{0.45\linewidth}  
        \includegraphics[width=\linewidth]{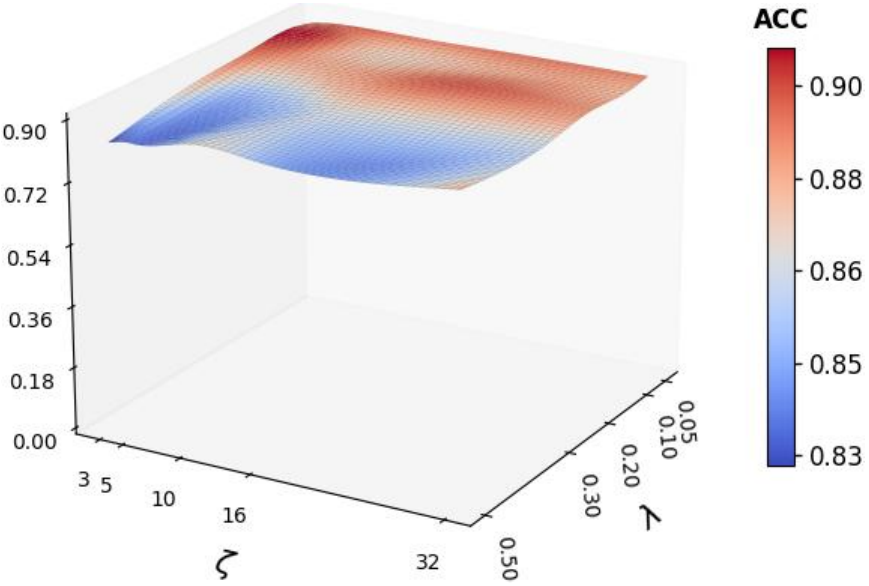}  
        \caption{ Caltech-5V}  
    \end{subfigure}  
   \hspace{0.01\linewidth}
    \begin{subfigure}[t]{0.45\linewidth}  
        \includegraphics[width=\linewidth]{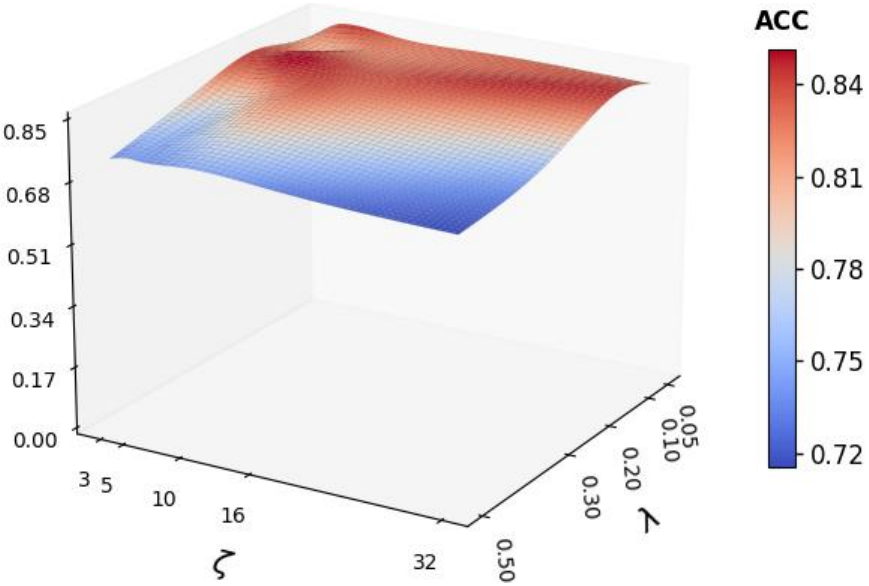}  
        \caption{ALOI-100}  
    \end{subfigure}  

    \vspace{-5pt}
    \caption{Parameter analyses for $\zeta$ and $\lambda$ with $r=0.5$.}
\label{fig:Sensitivity}
\end{figure}

\subsection{Visualization of Consensus Semantic Clusters}\label{experi:visual}
To further examine the quality of cluster structures formed by consensus semantic representations, t-SNE visualization based on true labels are plotted in Fig. \ref{fig:t-SNE} and clearly show minimal misclassification and clusters with strong intra-cluster cohesion while distinct inter-cluster separation. This indicates FreeCSL nearly recovers true cluster structures by learning consensus semantic information from all data.

\begin{figure}[t]  
    \centering  
    \begin{subfigure}[t]{0.45\linewidth}  
        \includegraphics[width=\linewidth]{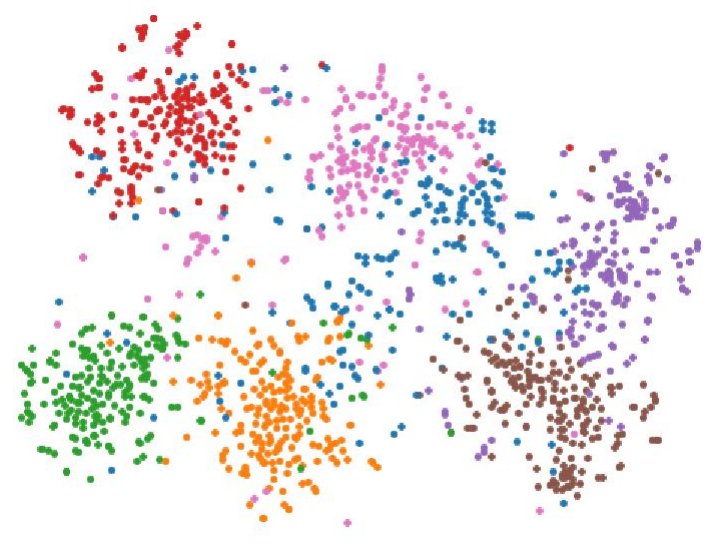}  
        \caption{Pre-training (NMI=44.18\%)}  
    \end{subfigure}  
    \hfill  
    \begin{subfigure}[t]{0.45\linewidth}  
        \includegraphics[width=\linewidth]{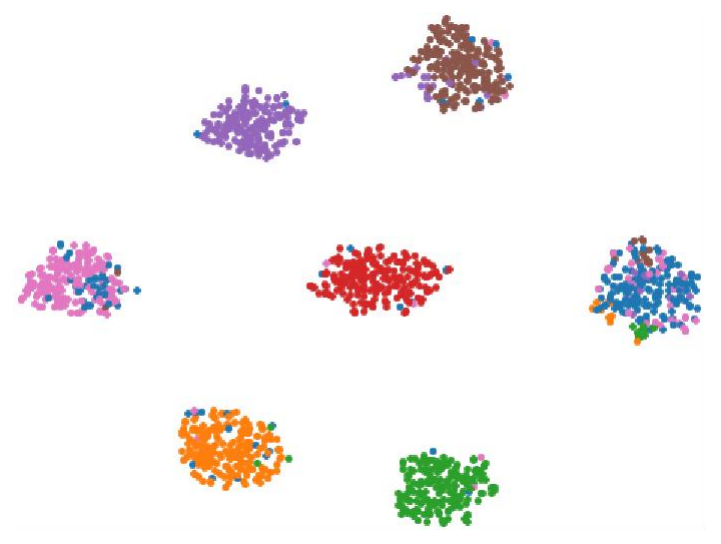}  
        \caption{Training (NMI=79.61\%)}  
    \end{subfigure}  

    \vspace{0.5em} 

    \begin{subfigure}[t]{0.45\linewidth}  
        \includegraphics[width=\linewidth]{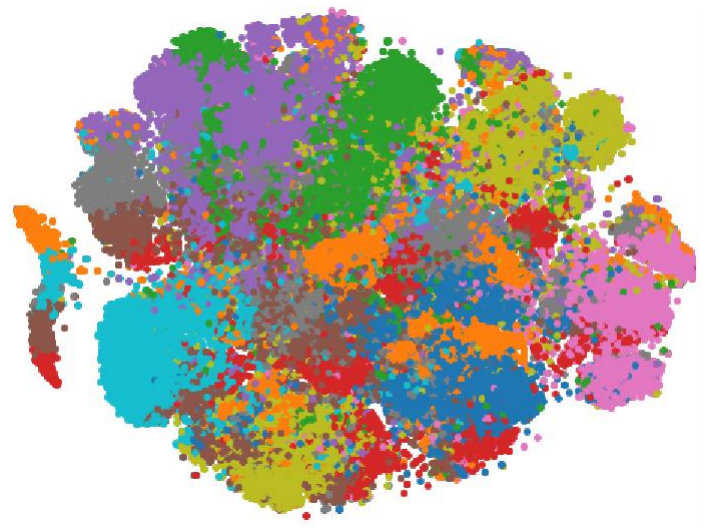}  
        \caption{Pre-training (NMI=33.74\%)}  
    \end{subfigure}  
    \hfill  
    \begin{subfigure}[t]{0.45\linewidth}  
        \includegraphics[width=\linewidth]{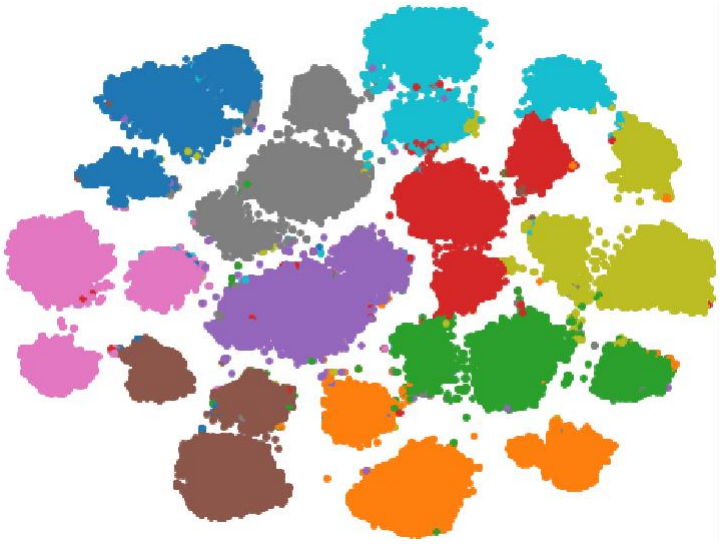}  
        \caption{Training (NMI=89.81\%)}  
    \end{subfigure}  
    \vspace{-3pt}
    \caption{Visualization for Caltech-5V and NoisyMNIST.}
    \vspace{-10pt}
\label{fig:t-SNE}
\end{figure} 
\section{Conclusion}\label{sec:conclusion}
In this paper, we propose FreeCSL, a novel semantic learning paradigm free from imputation and alignment compared to existing consistency learning. We design prototype-based contrastive clustering to discover a shared semantic space, where observations converge toward their respective semantic prototype and are encoded as consensus semantics representations for clustering. Furthermore, we employ modularity-inspired graph clustering to enrich semantic representation with view-specific cluster information. The effective synergy of consensus semantic learning and cluster semantic enhancement makes FreeCSL excel in most complex IMVC tasks.

\section*{Acknowledgments}
This work is supported by the National Key R\&D Program of China under Grant No.2022ZD0209103, the National Natural Science Foundation of China (project no. 62325604, 62276271, 62406329, 62476281, 62441618), National Natural Science Foundation of China Joint Found under Grant No. U24A20323.
\clearpage
{
    \small
    \newpage
    \bibliographystyle{ieeenat_fullname}
    \bibliography{main}

\begin{thebibliography}{71}
\providecommand{\natexlab}[1]{#1}
\providecommand{\url}[1]{\texttt{#1}}
\expandafter\ifx\csname urlstyle\endcsname\relax
  \providecommand{\doi}[1]{doi: #1}\else
  \providecommand{\doi}{doi: \begingroup \urlstyle{rm}\Url}\fi

\bibitem[Brandes et~al.(2006)Brandes, Delling, Gaertler, G{\"o}rke, Hoefer, Nikoloski, and Wagner]{brandes2006maximizing}
Ulrik Brandes, Daniel Delling, Marco Gaertler, Robert G{\"o}rke, Martin Hoefer, Zoran Nikoloski, and Dorothea Wagner.
\newblock Maximizing modularity is hard.
\newblock \emph{arXiv preprint physics/0608255}, 2006.

\bibitem[Caron et~al.(2020)Caron, Misra, Mairal, Goyal, Bojanowski, and Joulin]{caron2020unsupervised}
Mathilde Caron, Ishan Misra, Julien Mairal, Priya Goyal, Piotr Bojanowski, and Armand Joulin.
\newblock Unsupervised learning of visual features by contrasting cluster assignments.
\newblock \emph{Advances in neural information processing systems}, 33:\penalty0 9912--9924, 2020.

\bibitem[Chao et~al.(2024)Chao, Jiang, and Chu]{chao2024incomplete}
Guoqing Chao, Yi Jiang, and Dianhui Chu.
\newblock Incomplete contrastive multi-view clustering with high-confidence guiding.
\newblock In \emph{Proceedings of the AAAI Conference on Artificial Intelligence}, pages 11221--11229, 2024.

\bibitem[Chen et~al.(2020)Chen, Kornblith, Norouzi, and Hinton]{chen2020simple}
Ting Chen, Simon Kornblith, Mohammad Norouzi, and Geoffrey Hinton.
\newblock A simple framework for contrastive learning of visual representations.
\newblock In \emph{International conference on machine learning}, pages 1597--1607. PMLR, 2020.

\bibitem[Cui et~al.(2024)Cui, Ren, Pu, Li, Pu, Wu, Shi, and He]{cui2024novel}
Chenhang Cui, Yazhou Ren, Jingyu Pu, Jiawei Li, Xiaorong Pu, Tianyi Wu, Yutao Shi, and Lifang He.
\newblock A novel approach for effective multi-view clustering with information-theoretic perspective.
\newblock \emph{Advances in Neural Information Processing Systems}, 36, 2024.

\bibitem[Dong et~al.(2024)Dong, Jin, Xiao, Xiao, Wang, Liu, and Zhu]{dong2024subgraph}
Zhibin Dong, Jiaqi Jin, Yuyang Xiao, Bin Xiao, Siwei Wang, Xinwang Liu, and En Zhu.
\newblock Subgraph propagation and contrastive calibration for incomplete multiview data clustering.
\newblock \emph{IEEE Transactions on Neural Networks and Learning Systems}, 2024.

\bibitem[Du et~al.(2021)Du, Zhou, Yang, L{\"u}, and Wang]{du2021deep}
Guowang Du, Lihua Zhou, Yudi Yang, Kevin L{\"u}, and Lizhen Wang.
\newblock Deep multiple auto-encoder-based multi-view clustering.
\newblock \emph{Data Science and Engineering}, 6\penalty0 (3):\penalty0 323--338, 2021.

\bibitem[Fang et~al.(2023)Fang, Li, Li, Gao, Jia, and Zhang]{fang2023comprehensive}
Uno Fang, Man Li, Jianxin Li, Longxiang Gao, Tao Jia, and Yanchun Zhang.
\newblock A comprehensive survey on multi-view clustering.
\newblock \emph{IEEE Transactions on Knowledge and Data Engineering}, 35\penalty0 (12):\penalty0 12350--12368, 2023.

\bibitem[Feng et~al.(2024)Feng, Sheng, Wang, Gao, Tao, and Dong]{feng2024partial}
Wei Feng, Guoshuai Sheng, Qianqian Wang, Quanxue Gao, Zhiqiang Tao, and Bo Dong.
\newblock Partial multi-view clustering via self-supervised network.
\newblock In \emph{Proceedings of the AAAI Conference on Artificial Intelligence}, pages 11988--11995, 2024.

\bibitem[Guo et~al.(2024)Guo, Yang, Lin, Peng, and Hu]{guo2024robust}
Ruiming Guo, Mouxing Yang, Yijie Lin, Xi Peng, and Peng Hu.
\newblock Robust contrastive multi-view clustering against dual noisy correspondence.
\newblock \emph{Advances in Neural Information Processing Systems}, 37:\penalty0 121401--121421, 2024.

\bibitem[He et~al.(2024)He, Zhu, Hu, and Peng]{he2024robust}
Changhao He, Hongyuan Zhu, Peng Hu, and Xi Peng.
\newblock Robust variational contrastive learning for partially view-unaligned clustering.
\newblock In \emph{Proceedings of the 32nd ACM International Conference on Multimedia}, pages 4167--4176, 2024.

\bibitem[He et~al.(2020)He, Fan, Wu, Xie, and Girshick]{he2020momentum}
Kaiming He, Haoqi Fan, Yuxin Wu, Saining Xie, and Ross Girshick.
\newblock Momentum contrast for unsupervised visual representation learning.
\newblock In \emph{Proceedings of the IEEE/CVF conference on computer vision and pattern recognition}, pages 9729--9738, 2020.

\bibitem[Hu et~al.(2017)Hu, Lu, and Tan]{hu2017sharable}
Junlin Hu, Jiwen Lu, and Yap-Peng Tan.
\newblock Sharable and individual multi-view metric learning.
\newblock \emph{IEEE transactions on pattern analysis and machine intelligence}, 40\penalty0 (9):\penalty0 2281--2288, 2017.

\bibitem[Hu et~al.(2023)Hu, Zhen, Peng, Zhu, Lin, Wang, and Peng]{hu2023deep}
Peng Hu, Liangli Zhen, Xi Peng, Hongyuan Zhu, Jie Lin, Xu Wang, and Dezhong Peng.
\newblock Deep supervised multi-view learning with graph priors.
\newblock \emph{IEEE Transactions on Image Processing}, 33:\penalty0 123--133, 2023.

\bibitem[Huang et~al.(2023)Huang, Wang, and Lai]{huang2023fast}
Dong Huang, Chang-Dong Wang, and Jian-Huang Lai.
\newblock Fast multi-view clustering via ensembles: Towards scalability, superiority, and simplicity.
\newblock \emph{IEEE Transactions on Knowledge and Data Engineering}, 35\penalty0 (11):\penalty0 11388--11402, 2023.

\bibitem[Huang et~al.(2020{\natexlab{a}})Huang, Gong, and Zhu]{huang2020deep}
Jiabo Huang, Shaogang Gong, and Xiatian Zhu.
\newblock Deep semantic clustering by partition confidence maximisation.
\newblock In \emph{Proceedings of the IEEE/CVF conference on computer vision and pattern recognition}, pages 8849--8858, 2020{\natexlab{a}}.

\bibitem[Huang et~al.(2019)Huang, Zhou, Peng, Zhang, Zhu, and Lv]{huang2019multi}
Zhenyu Huang, Joey~Tianyi Zhou, Xi Peng, Changqing Zhang, Hongyuan Zhu, and Jiancheng Lv.
\newblock Multi-view spectral clustering network.
\newblock In \emph{IJCAI}, page~4, 2019.

\bibitem[Huang et~al.(2020{\natexlab{b}})Huang, Hu, Zhou, Lv, and Peng]{huang2020partially}
Zhenyu Huang, Peng Hu, Joey~Tianyi Zhou, Jiancheng Lv, and Xi Peng.
\newblock Partially view-aligned clustering.
\newblock \emph{Advances in Neural Information Processing Systems}, 33:\penalty0 2892--2902, 2020{\natexlab{b}}.

\bibitem[Ji et~al.(2021)Ji, Sun, Gao, Hu, and Yin]{ji2021decoder}
Qiang Ji, Yanfeng Sun, Junbin Gao, Yongli Hu, and Baocai Yin.
\newblock A decoder-free variational deep embedding for unsupervised clustering.
\newblock \emph{IEEE Transactions on Neural Networks and Learning Systems}, 33\penalty0 (10):\penalty0 5681--5693, 2021.

\bibitem[Jiang et~al.(2019)Jiang, Xu, Yang, Cao, and Huang]{jiang2019dm2c}
Yangbangyan Jiang, Qianqian Xu, Zhiyong Yang, Xiaochun Cao, and Qingming Huang.
\newblock Dm2c: Deep mixed-modal clustering.
\newblock \emph{Advances in Neural Information Processing Systems}, 32, 2019.

\bibitem[Jin et~al.(2023)Jin, Wang, Dong, Liu, and Zhu]{jin2023deep}
Jiaqi Jin, Siwei Wang, Zhibin Dong, Xinwang Liu, and En Zhu.
\newblock Deep incomplete multi-view clustering with cross-view partial sample and prototype alignment.
\newblock In \emph{Proceedings of the IEEE/CVF conference on computer vision and pattern recognition}, pages 11600--11609, 2023.

\bibitem[Ke et~al.(2024)Ke, Wang, Wang, and He]{ke2024rethinking}
Guanzhou Ke, Bo Wang, Xiaoli Wang, and Shengfeng He.
\newblock Rethinking multi-view representation learning via distilled disentangling.
\newblock In \emph{Proceedings of the IEEE/CVF Conference on Computer Vision and Pattern Recognition}, pages 26774--26783, 2024.

\bibitem[Kernighan and Lin(1970)]{kernighan1970efficient}
Brian~W Kernighan and Shen Lin.
\newblock An efficient heuristic procedure for partitioning graphs.
\newblock \emph{The Bell system technical journal}, 49\penalty0 (2):\penalty0 291--307, 1970.

\bibitem[Li et~al.(2023{\natexlab{a}})Li, Li, Yang, Hu, Peng, and Peng]{li2023incomplete}
Haobin Li, Yunfan Li, Mouxing Yang, Peng Hu, Dezhong Peng, and Xi Peng.
\newblock Incomplete multi-view clustering via prototype-based imputation.
\newblock \emph{arXiv preprint arXiv:2301.11045}, 2023{\natexlab{a}}.

\bibitem[Li et~al.(2023{\natexlab{b}})Li, Sun, Sun, Ren, and Sun]{li2023cross}
Xingfeng Li, Yinghui Sun, Quansen Sun, Zhenwen Ren, and Yuan Sun.
\newblock Cross-view graph matching guided anchor alignment for incomplete multi-view clustering.
\newblock \emph{Information Fusion}, 100:\penalty0 101941, 2023{\natexlab{b}}.

\bibitem[Li et~al.(2018)Li, Yang, and Zhang]{li2018survey}
Yingming Li, Ming Yang, and Zhongfei Zhang.
\newblock A survey of multi-view representation learning.
\newblock \emph{IEEE transactions on knowledge and data engineering}, 31\penalty0 (10):\penalty0 1863--1883, 2018.

\bibitem[Li et~al.(2021)Li, Hu, Liu, Peng, Zhou, and Peng]{li2021contrastive}
Yunfan Li, Peng Hu, Zitao Liu, Dezhong Peng, Joey~Tianyi Zhou, and Xi Peng.
\newblock Contrastive clustering.
\newblock In \emph{Proceedings of the AAAI Conference on Artificial Intelligence}, pages 8547--8555, 2021.

\bibitem[Li et~al.(2022)Li, Yang, Peng, Li, Huang, and Peng]{li2022twin}
Yunfan Li, Mouxing Yang, Dezhong Peng, Taihao Li, Jiantao Huang, and Xi Peng.
\newblock Twin contrastive learning for online clustering.
\newblock \emph{International Journal of Computer Vision}, 130\penalty0 (9):\penalty0 2205--2221, 2022.

\bibitem[Lin et~al.(2021)Lin, Gou, Liu, Li, Lv, and Peng]{lin2021completer}
Yijie Lin, Yuanbiao Gou, Zitao Liu, Boyun Li, Jiancheng Lv, and Xi Peng.
\newblock Completer: Incomplete multi-view clustering via contrastive prediction.
\newblock In \emph{Proceedings of the IEEE/CVF Conference on Computer Vision and Pattern Recognition}, pages 11174--11183, 2021.

\bibitem[Lin et~al.(2022)Lin, Gou, Liu, Bai, Lv, and Peng]{lin2022dual}
Yijie Lin, Yuanbiao Gou, Xiaotian Liu, Jinfeng Bai, Jiancheng Lv, and Xi Peng.
\newblock Dual contrastive prediction for incomplete multi-view representation learning.
\newblock \emph{IEEE Transactions on Pattern Analysis and Machine Intelligence}, 45\penalty0 (4):\penalty0 4447--4461, 2022.

\bibitem[Liu et~al.(2023)Liu, Zhang, Wang, Chen, Wang, Huang, Shen, and Wang]{liu2023efficient}
Chong Liu, Yuqi Zhang, Hongsong Wang, Weihua Chen, Fan Wang, Yan Huang, Yi-Dong Shen, and Liang Wang.
\newblock Efficient token-guided image-text retrieval with consistent multimodal contrastive training.
\newblock \emph{IEEE Transactions on Image Processing}, 32:\penalty0 3622--3633, 2023.

\bibitem[Liu et~al.(2018)Liu, Zhu, Li, Wang, Tang, Yin, Shen, Wang, and Gao]{liu2018late}
Xinwang Liu, Xinzhong Zhu, Miaomiao Li, Lei Wang, Chang Tang, Jianping Yin, Dinggang Shen, Huaimin Wang, and Wen Gao.
\newblock Late fusion incomplete multi-view clustering.
\newblock \emph{IEEE transactions on pattern analysis and machine intelligence}, 41\penalty0 (10):\penalty0 2410--2423, 2018.

\bibitem[Lu et~al.(2024{\natexlab{a}})Lu, Li, Li, Lin, and Peng]{lu2024survey}
Yiding Lu, Haobin Li, Yunfan Li, Yijie Lin, and Xi Peng.
\newblock A survey on deep clustering: from the prior perspective.
\newblock \emph{Vicinagearth}, 1\penalty0 (1):\penalty0 4, 2024{\natexlab{a}}.

\bibitem[Lu et~al.(2024{\natexlab{b}})Lu, Lin, Yang, Peng, Hu, and Peng]{lu2024decoupled}
Yiding Lu, Yijie Lin, Mouxing Yang, Dezhong Peng, Peng Hu, and Xi Peng.
\newblock Decoupled contrastive multi-view clustering with high-order random walks.
\newblock In \emph{Proceedings of the AAAI Conference on Artificial Intelligence}, pages 14193--14201, 2024{\natexlab{b}}.

\bibitem[Newman(2006)]{newman2006modularity}
Mark~EJ Newman.
\newblock Modularity and community structure in networks.
\newblock \emph{Proceedings of the national academy of sciences}, 103\penalty0 (23):\penalty0 8577--8582, 2006.

\bibitem[Pu et~al.(2024)Pu, Cui, Chen, Ren, Pu, Hao, Philip, and He]{pu2024adaptive}
Jingyu Pu, Chenhang Cui, Xinyue Chen, Yazhou Ren, Xiaorong Pu, Zhifeng Hao, S~Yu Philip, and Lifang He.
\newblock Adaptive feature imputation with latent graph for deep incomplete multi-view clustering.
\newblock In \emph{Proceedings of the AAAI Conference on Artificial Intelligence}, pages 14633--14641, 2024.

\bibitem[Shen et~al.(2021)Shen, Shen, Wang, Qin, Torr, and Shao]{shen2021you}
Yuming Shen, Ziyi Shen, Menghan Wang, Jie Qin, Philip Torr, and Ling Shao.
\newblock You never cluster alone.
\newblock \emph{Advances in Neural Information Processing Systems}, 34:\penalty0 27734--27746, 2021.

\bibitem[Sun et~al.(2024)Sun, Qin, Li, Peng, Peng, and Hu]{sun2024robust}
Yuan Sun, Yang Qin, Yongxiang Li, Dezhong Peng, Xi Peng, and Peng Hu.
\newblock Robust multi-view clustering with noisy correspondence.
\newblock \emph{IEEE Transactions on Knowledge and Data Engineering}, 2024.

\bibitem[Tang and Liu(2022)]{tang2022deep}
Huayi Tang and Yong Liu.
\newblock Deep safe multi-view clustering: Reducing the risk of clustering performance degradation caused by view increase.
\newblock In \emph{Proceedings of the IEEE/CVF Conference on Computer Vision and Pattern Recognition}, pages 202--211, 2022.

\bibitem[Tang et~al.(2024)Tang, Yi, Fu, and Tian]{tang2024incomplete}
Jingjing Tang, Qingqing Yi, Saiji Fu, and Yingjie Tian.
\newblock Incomplete multi-view learning: Review, analysis, and prospects.
\newblock \emph{Applied Soft Computing}, page 111278, 2024.

\bibitem[Wan et~al.(2024{\natexlab{a}})Wan, Liu, Gan, Liu, Wang, Wen, Wan, and Zhu]{10486880}
Xinhang Wan, Jiyuan Liu, Xinbiao Gan, Xinwang Liu, Siwei Wang, Yi Wen, Tianjiao Wan, and En Zhu.
\newblock One-step multi-view clustering with diverse representation.
\newblock \emph{IEEE Transactions on Neural Networks and Learning Systems}, pages 1--13, 2024{\natexlab{a}}.

\bibitem[Wan et~al.(2024{\natexlab{b}})Wan, Xiao, Liu, Liu, Liang, and Zhu]{10506102}
Xinhang Wan, Bin Xiao, Xinwang Liu, Jiyuan Liu, Weixuan Liang, and En Zhu.
\newblock Fast continual multi-view clustering with incomplete views.
\newblock \emph{IEEE Transactions on Image Processing}, 33:\penalty0 2995--3008, 2024{\natexlab{b}}.

\bibitem[Wang et~al.(2021{\natexlab{a}})Wang, Ding, Li, and Zheng]{wang2021cline}
Dong Wang, Ning Ding, Piji Li, and Haitao Zheng.
\newblock Cline: Contrastive learning with semantic negative examples for natural language understanding.
\newblock In \emph{Proceedings of the 59th Annual Meeting of the Association for Computational Linguistics and the 11th International Joint Conference on Natural Language Processing (Volume 1: Long Papers)}, pages 2332--2342, 2021{\natexlab{a}}.

\bibitem[Wang et~al.(2022{\natexlab{a}})Wang, Guo, Deng, and Lu]{wang2022rethinking}
Haoqing Wang, Xun Guo, Zhi-Hong Deng, and Yan Lu.
\newblock Rethinking minimal sufficient representation in contrastive learning.
\newblock In \emph{Proceedings of the IEEE/CVF Conference on Computer Vision and Pattern Recognition}, pages 16041--16050, 2022{\natexlab{a}}.

\bibitem[Wang et~al.(2021{\natexlab{b}})Wang, Ding, Tao, Gao, and Fu]{wang2021generative}
Qianqian Wang, Zhengming Ding, Zhiqiang Tao, Quanxue Gao, and Yun Fu.
\newblock Generative partial multi-view clustering with adaptive fusion and cycle consistency.
\newblock \emph{IEEE Transactions on Image Processing}, 30:\penalty0 1771--1783, 2021{\natexlab{b}}.

\bibitem[Wang et~al.(2022{\natexlab{b}})Wang, Liu, Liu, Jin, Tu, Zhu, and Zhu]{wangalign}
Siwei Wang, Xinwang Liu, Suyuan Liu, Jiaqi Jin, Wenxuan Tu, Xinzhong Zhu, and En Zhu.
\newblock Align then fusion: Generalized large-scale multi-view clustering with anchor matching correspondences.
\newblock \emph{Advances in Neural Information Processing Systems}, 35:\penalty0 5882--5895, 2022{\natexlab{b}}.

\bibitem[Wang et~al.(2022{\natexlab{c}})Wang, Chang, Fu, Wen, and Zhao]{wang2022incomplete}
Yiming Wang, Dongxia Chang, Zhiqiang Fu, Jie Wen, and Yao Zhao.
\newblock Incomplete multi-view clustering via cross-view relation transfer.
\newblock \emph{IEEE Transactions on Circuits and Systems for Video Technology}, 2022{\natexlab{c}}.

\bibitem[Wei et~al.(2020)Wei, Wang, Yu, Domeniconi, and Zhang]{wei2020deep}
Shaowei Wei, Jun Wang, Guoxian Yu, Carlotta Domeniconi, and Xiangliang Zhang.
\newblock Deep incomplete multi-view multiple clusterings.
\newblock In \emph{2020 IEEE International Conference on Data Mining (ICDM)}, pages 651--660. IEEE, 2020.

\bibitem[Wu et~al.(2018)Wu, Xiong, Yu, and Lin]{wu2018unsupervised}
Zhirong Wu, Yuanjun Xiong, Stella~X Yu, and Dahua Lin.
\newblock Unsupervised feature learning via non-parametric instance discrimination.
\newblock In \emph{Proceedings of the IEEE conference on computer vision and pattern recognition}, pages 3733--3742, 2018.

\bibitem[Xie et~al.(2021)Xie, Ding, Wang, Zhan, Xu, Sun, Li, and Luo]{xie2021detco}
Enze Xie, Jian Ding, Wenhai Wang, Xiaohang Zhan, Hang Xu, Peize Sun, Zhenguo Li, and Ping Luo.
\newblock Detco: Unsupervised contrastive learning for object detection.
\newblock In \emph{Proceedings of the IEEE/CVF international conference on computer vision}, pages 8392--8401, 2021.

\bibitem[Xie et~al.(2016)Xie, Girshick, and Farhadi]{xie2016unsupervised}
Junyuan Xie, Ross Girshick, and Ali Farhadi.
\newblock Unsupervised deep embedding for clustering analysis.
\newblock In \emph{International conference on machine learning}, pages 478--487. PMLR, 2016.

\bibitem[Xu et~al.(2019)Xu, Guan, Zhao, Wu, Niu, and Ling]{xu2019adversarial}
Cai Xu, Ziyu Guan, Wei Zhao, Hongchang Wu, Yunfei Niu, and Beilei Ling.
\newblock Adversarial incomplete multi-view clustering.
\newblock In \emph{IJCAI}, pages 3933--3939, 2019.

\bibitem[Xu et~al.(2022{\natexlab{a}})Xu, Li, Ren, Peng, Mo, Shi, and Zhu]{xu2022deep}
Jie Xu, Chao Li, Yazhou Ren, Liang Peng, Yujie Mo, Xiaoshuang Shi, and Xiaofeng Zhu.
\newblock Deep incomplete multi-view clustering via mining cluster complementarity.
\newblock In \emph{Proceedings of the AAAI conference on artificial intelligence}, pages 8761--8769, 2022{\natexlab{a}}.

\bibitem[Xu et~al.(2022{\natexlab{b}})Xu, Tang, Ren, Peng, Zhu, and He]{MFLVC}
Jie Xu, Huayi Tang, Yazhou Ren, Liang Peng, Xiaofeng Zhu, and Lifang He.
\newblock Multi-level feature learning for contrastive multi-view clustering.
\newblock In \emph{Proceedings of the IEEE/CVF Conference on Computer Vision and Pattern Recognition}, pages 16051--16060, 2022{\natexlab{b}}.

\bibitem[Xu et~al.(2023)Xu, Li, Peng, Ren, Shi, Shen, and Zhu]{xu2023adaptive}
Jie Xu, Chao Li, Liang Peng, Yazhou Ren, Xiaoshuang Shi, Heng~Tao Shen, and Xiaofeng Zhu.
\newblock Adaptive feature projection with distribution alignment for deep incomplete multi-view clustering.
\newblock \emph{IEEE Transactions on Image Processing}, 32:\penalty0 1354--1366, 2023.

\bibitem[Xu et~al.(2024)Xu, Dong, Qi, Zhang, Xiang, Xia, Xu, and Dou]{xu2024cmclrec}
Xiaolong Xu, Hongsheng Dong, Lianyong Qi, Xuyun Zhang, Haolong Xiang, Xiaoyu Xia, Yanwei Xu, and Wanchun Dou.
\newblock Cmclrec: Cross-modal contrastive learning for user cold-start sequential recommendation.
\newblock In \emph{Proceedings of the 47th International ACM SIGIR Conference on Research and Development in Information Retrieval}, pages 1589--1598, 2024.

\bibitem[Yan et~al.(2024)Yan, Jin, Han, and Ye]{yan2024differentiable}
Xiaoqiang Yan, Zhixiang Jin, Fengshou Han, and Yangdong Ye.
\newblock Differentiable information bottleneck for deterministic multi-view clustering.
\newblock In \emph{Proceedings of the IEEE/CVF Conference on Computer Vision and Pattern Recognition}, pages 27435--27444, 2024.

\bibitem[Yang et~al.(2021{\natexlab{a}})Yang, Fan, and Bouguila]{yang2021deep}
Lin Yang, Wentao Fan, and Nizar Bouguila.
\newblock Deep clustering analysis via dual variational autoencoder with spherical latent embeddings.
\newblock \emph{IEEE Transactions on Neural Networks and Learning Systems}, 34\penalty0 (9):\penalty0 6303--6312, 2021{\natexlab{a}}.

\bibitem[Yang et~al.(2021{\natexlab{b}})Yang, Li, Huang, Liu, Hu, and Peng]{MvCLN}
Mouxing Yang, Yunfan Li, Zhenyu Huang, Zitao Liu, Peng Hu, and Xi Peng.
\newblock Partially view-aligned representation learning with noise-robust contrastive loss.
\newblock In \emph{Proceedings of the IEEE/CVF conference on computer vision and pattern recognition}, pages 1134--1143, 2021{\natexlab{b}}.

\bibitem[Yang et~al.(2021{\natexlab{c}})Yang, Li, Huang, Liu, Hu, and Peng]{yang2021partially}
Mouxing Yang, Yunfan Li, Zhenyu Huang, Zitao Liu, Peng Hu, and Xi Peng.
\newblock Partially view-aligned representation learning with noise-robust contrastive loss.
\newblock In \emph{Proceedings of the IEEE/CVF conference on computer vision and pattern recognition}, pages 1134--1143, 2021{\natexlab{c}}.

\bibitem[Yang et~al.(2022{\natexlab{a}})Yang, Li, Hu, Bai, Lv, and Peng]{SURE}
Mouxing Yang, Yunfan Li, Peng Hu, Jinfeng Bai, Jiancheng Lv, and Xi Peng.
\newblock Robust multi-view clustering with incomplete information.
\newblock \emph{IEEE Transactions on Pattern Analysis and Machine Intelligence}, 45\penalty0 (1):\penalty0 1055--1069, 2022{\natexlab{a}}.

\bibitem[Yang et~al.(2022{\natexlab{b}})Yang, Li, Hu, Bai, Lv, and Peng]{yang2022robust}
Mouxing Yang, Yunfan Li, Peng Hu, Jinfeng Bai, Jian~Cheng Lv, and Xi Peng.
\newblock Robust multi-view clustering with incomplete information.
\newblock \emph{IEEE Transactions on Pattern Analysis and Machine Intelligence}, 2022{\natexlab{b}}.

\bibitem[Yu et~al.(2023)Yu, Liu, Wang, et~al.]{yu2023sparse}
Shengju Yu, Suyuan Liu, Siwei Wang, et~al.
\newblock Sparse low-rank multi-view subspace clustering with consensus anchors and unified bipartite graph.
\newblock \emph{IEEE Transactions on Neural Networks and Learning Systems}, 2023.

\bibitem[Yu et~al.(2024)Yu, Wang, Dong, et~al.]{yu2024non}
Shengju Yu, Siwei Wang, Zhibin Dong, et~al.
\newblock A non-parametric graph clustering framework for multi-view data.
\newblock In \emph{Proceedings of the AAAI conference on artificial intelligence}, pages 16558--16567, 2024.

\bibitem[Zeng et~al.(2023)Zeng, Yang, Lu, Zhang, Hu, and Peng]{zeng2023semantic}
Pengxin Zeng, Mouxing Yang, Yiding Lu, Changqing Zhang, Peng Hu, and Xi Peng.
\newblock Semantic invariant multi-view clustering with fully incomplete information.
\newblock \emph{IEEE Transactions on Pattern Analysis and Machine Intelligence}, 46\penalty0 (4):\penalty0 2139--2150, 2023.

\bibitem[Zhang et~al.(2020)Zhang, Cui, Han, Zhou, Fu, and Hu]{zhang2020deep}
Changqing Zhang, Yajie Cui, Zongbo Han, Joey~Tianyi Zhou, Huazhu Fu, and Qinghua Hu.
\newblock Deep partial multi-view learning.
\newblock \emph{IEEE transactions on pattern analysis and machine intelligence}, 44\penalty0 (5):\penalty0 2402--2415, 2020.

\bibitem[Zhang et~al.(2021)Zhang, Liu, Wang, Liu, Dai, and Zhu]{zhang2021one}
Yi Zhang, Xinwang Liu, Siwei Wang, Jiyuan Liu, Sisi Dai, and En Zhu.
\newblock One-stage incomplete multi-view clustering via late fusion.
\newblock In \emph{Proceedings of the 29th ACM international conference on multimedia}, pages 2717--2725, 2021.

\bibitem[Zhang et~al.(2024)Zhang, Tian, Ma, Li, Yang, Liu, Zhu, and Liu]{zhang2024regularized}
Yi Zhang, Fengyu Tian, Chuan Ma, Miaomiao Li, Hengfu Yang, Zhe Liu, En Zhu, and Xinwang Liu.
\newblock Regularized instance weighting multiview clustering via late fusion alignment.
\newblock \emph{IEEE Transactions on Neural Networks and Learning Systems}, 2024.

\bibitem[Zhong et~al.(2020)Zhong, Chen, Jin, and Hua]{zhong2020deep}
Huasong Zhong, Chong Chen, Zhongming Jin, and Xian-Sheng Hua.
\newblock Deep robust clustering by contrastive learning.
\newblock \emph{arXiv preprint arXiv:2008.03030}, 2020.

\bibitem[Zhou et~al.(2024)Zhou, Du, L{\"u}, Wang, and Du]{zhou2024survey}
Lihua Zhou, Guowang Du, Kevin L{\"u}, Lizheng Wang, and Jingwei Du.
\newblock A survey and an empirical evaluation of multi-view clustering approaches.
\newblock \emph{ACM Computing Surveys}, 56\penalty0 (7):\penalty0 1--38, 2024.

\bibitem[Zhu et~al.(2019)Zhu, Yao, Wang, Hui, Du, and Hu]{zhu2019multi}
Pengfei Zhu, Xinjie Yao, Yu Wang, Binyuan Hui, Dawei Du, and Qinghua Hu.
\newblock Multi-view deep subspace clustering networks.
\newblock \emph{arXiv preprint arXiv:1908.01978}, 2019.

\end{thebibliography}
}
\clearpage
\setcounter{page}{1}
\maketitlesupplementary

\section{Appendix A: Related Work}
\label{sec:rationale}

\subsection{Contrastive Learning for Consistency Learning}\label{sec21}
Exploring consistency information from complete instances across views is an effective way to alleviate instance observations missing and cluster distribution shifted in incomplete multi-view clustering (IMVC). Contrastive learning \cite{wu2018unsupervised, he2020momentum, chen2020simple, zbontar2021barlow}, as an unsupervised representation learning \cite{xu2024cmclrec,liu2023efficient,xie2021detco}, can learn the structural consistency information from multi-view data bring closer instances from positive samples and separate instances from negative samples \cite{wu2018unsupervised, he2020momentum, chen2020simple, wang2021cline,zbontar2021barlow}, and has been successfully extended to multi-view clustering (MVC) task. 

Specifically, the most widely applied contrastive learning paradigms construct positive and negative pairs at the instance-level. Despite instance-level paradigm have shown exceptional capability in consistency representation learning, two primary limits, false negative noise from intra-cluster observations for different instances and the local smoothness of instance representations, damage representation learning due to the loss of view-specific information.

After all, clustering is a one-to-many mapping. Recognizing false negative pairs (FPNs) causes detrimental impacts on clustering confidence and robustness, a cluster-level paradigm is proposed to discover cross-view cluster correspondences for intra-cluster but unpaired observations by reducing FPNs: TCL \cite{li2022twin} selects pseudo-labels with confidence-based criteria to mitigate false negative impacts, while the noise-robust contrastive loss proposed by SURE \cite{yang2022robust} further discriminate false negative pairs by using a adaptive threshold calculated from distances of all positive and negative pairs. DIVIDE \cite{lu2024decoupled} utilizes an anchor-based approach to identify out-of-domain samples through high-order random walks to mitigate the issue of false negatives. They resolve the confusion of cross-view cluster correspondences caused by instance-level paradigms, but at the cost of cluster information within specific views, which hinders semantic consistency in representation learning.

\subsection{Imputation and Alignment for IMVC}\label{sec22}
In IMVC, to preserve even recover relationships between data, imputation are supposed to handle missing data. 
Regarding the former, typical approaches include the cross-view transfer paradigm like neighborhood-based recovery, the cross-view interaction paradigm like adversarial generation or contrastive prediction. As members of transfer paradigm, the core idea of CRTC \cite{wang2022incomplete} and ICMVC \cite{chao2024incomplete} is to transfer the complete graph neighborhood relations from other views to missing views. However, neighborhood-based recovery, which uses cross-view neighbor information for imputation, overlooks complementary information specific to each view. To improve imputation performance, generative models such as autoencoders (AE) and generative adversarial networks (GAN), as well as discriminative models like contrastive learning, discover correlations across multi-view data to dynamically collaborate on both imputation and clustering. For examples, \cite{wei2020deep}, \cite{yang2021deep} and \cite{ji2021decoder} leverage the power of AEs in encoding latent representations to mine view-specific information for imputation; CPM-Nets \cite{zhang2020deep} and GP-MVC \cite{wang2021generative} encode a common representation with consistency and complementarity information across views and employ adversarial strategies to reconstruct the common representation to approximate generated observations within views; COMPLETER \cite{lin2021completer} and DCP \cite{lin2022dual} unify cross-view consistency learning and missing prediction into a deep framework to constrain both complete paired observations and incomplete recovered observations by maximizing mutual information and minimizing conditional entropy across views. Although they successfully apply view-specific information in imputation, they lose the cluster structure information within the missing views. Thus, ProImp \cite{li2023incomplete} proposed a novel paradigm based on within-view prototypes and cross-view observation-prototype relationships to further improve imputation performance.

However, the aforementioned imputation methods are limited by unsupervised learning and cannot restore the original distribution of view data. To achieve confident and robust clustering, a feasible solution is cross-view consistency alignment, generally categorized into cross-view cluster assignments-based, prototypes-based and distributions-based as the following works: To integrate soft labels from various views for decision fusion, DIMVC \cite{xu2022deep} aligns view-specific labels with a unified label using conditional entropy loss. DSIMVC \cite{tang2022deep} argues that multi-view data share common semantic information, so a contrastive loss is designed to align cluster assignments across views for consistency. CPSPAN \cite{jin2023deep} and ProImp \cite{li2023incomplete} employ Hungarian algorithm and bounded contrastive loss \cite{li2023incomplete} to calibrate prototype-shifted across views. To reduce cross-view distribution discrepancy arising from complete and incomplete data, APADC \cite{xu2023adaptive} minimize the mean discrepancy loss to align view distributions in a common representation space. SPCC \cite{dong2024subgraph} directly optimizes the distribution alignment loss of $K$ cluster across views.

Whether imputation or alignment, there is a deviation compared to the original data, and this deviation increases rapidly as the amount of available complete data decreases. To this end, different other IMVC methods, our FreeCSL, a novel consensus semantic-based paradigm,  discover the shared semantic space through consensus prototype-based contrastive clustering, where all available observations are encoded as representations with consensus semantics for clustering. More specifically,  during consensus learning, all observations can straightforwardly reach consensus on cluster semantic information without imputation and alignment.

\section{Appendix B: Theorem Proof}
\noindent\textbf{Definition 1.}\label{def11}
\textit{Instance-level Consistency (IC): $\forall m \neq n$, $\mathbf{x}_i^{{m}}$ and $\mathbf{x}_j^{{n}}$ are instance-level consistent across views if $i=j$ (they are cross-view observations of the same instance $\mathbf{x}$), expressed as $I(\mathbf{x}_i^{{m}},\mathbf{x}_j^{{n}})=1$ and 0 otherwise.}

\noindent\textbf{Definition 2.}
\textit{Cluster-level Consistency (CC): $\forall m \neq n$, $\mathbf{x}_i^{{m}}$ and $\mathbf{x}_j^{{n}}$ are cluster-level consistent across views if they belong to the same cluster $k$, expressed as $C(\mathbf{x}_i^{{m}},\mathbf{x}_j^{{n}})= 1$ and 0 otherwise.}

\noindent\textbf{Definition 3.}\label{def33}
\textit{Semantic-level Consensus (SC): $\forall m$ and $n$, $\mathbf{x}_i^{{m}}$ and $\mathbf{x}_j^{{n}}$ achieve semantic-level consensus in MVC task if all observations share a set of cluster prototypes $\mathbf{C}={\{\mathbf{c}_k\}_{k=1}^K}$ and $\arg\max\limits_{k} \mathcal{\rho}(\mathbf{x}_i^{m}, \mathbf{c}_k) = \arg\max\limits_{k} \mathcal{\rho}(\mathbf{x}_j^{n}, \mathbf{c}_k)$ , expressed as $S(\mathbf{x}_i^{{m}},\mathbf{x}_j^{{n}})=1$ and 0 otherwise.}

\subsection{Proof of Theorem 1}\label{sec22}
\begin{theorem}\label{thm:Theorem1} 
Consensus semantic learning yields more confident and robust cluster assignments than instance- and cluster-level paradigms.
\end{theorem}

\textit{Case 1: Instance-level paradigm 
pull paired observations $(\overline{\mathbf{x}}_i^{{m}},\overline{\mathbf{x}}_i^{{n}})$ closer and push unpaired observations $({\mathbf{x}}_i^{{m}},{\mathbf{x}}_j^{{n}})$ apart. However, if $C(\mathbf{x}_i^{{m}},\mathbf{x}_j^{{n}})= 1$, intra-cluster but unpaired observations are treated as negative pairs, introducing false negative noise into clustering.}

\textit{Case 2: Cluster-level paradigm encourages the observation $\mathbf{x}_i^{{m}}$ to find its cluster-level counterparts $\mathbf{x}_j^{{n}}$ from different view $n$ to mitigate false negative noisy. However, lacking within-view clustering mapping for view-specific cluster information, it explores cross-view cluster correspondences but fails to ensure cluster semantics consistency within views.}

\textit{Case 3: Semantic-level paradigm
construct a shared semantic space based on consensus prototypes $\mathbf{C}$ for all observations to eliminate semantic gaps and capture semantic relationships within clusters.}

\begin{proof}
Define a general consistency learning objective as
\begin{equation}
    \max \sum\limits_{m\neq n}^{V}\sum\limits_{i}^{N}\sum\limits_{j\neq j'}^{N}  \{Y\mathcal{\rho}^+ ( \mathbf {x}_i^{{m}}, \mathbf {x}_j^{{n}} ) + (Y-1)\mathcal{\rho}^-( \mathbf{x}_i^{{m}}, \mathbf{x}_{j'}^{{n}} ) \},
\end{equation}
where $Y=1/0$ mean positive/negative pairs, and $\mathcal{\rho}^{+/-}$ measure the similarity between positive/negative pairs.


\textbf{Instance-level paradigms:} 
When $Y=I(\mathbf{x}_i^{{m}},\mathbf{x}_j^{{n}})$, the objective of instance-level paradigms $f_{ic}$ is formulated as:
\begin{equation}
    f_{ic} = \sum\limits_{m\neq n}^{V}\sum\limits_{i}^{N} [\mathcal{\rho}^+(\mathbf{x}_i^{m},\mathbf {x}_i^{n}) -\sum\limits_{j \neq i}^{N} \mathcal{\rho}^-( \mathbf{x}_i^{{m}}, \mathbf{x}_j^{n})].
\end{equation}
When $C(\mathbf{x}_i^{{m}},\mathbf{x}_j^{{n}})=1$, the instance-level paradigm incorrectly treats them as negative pairs, introducing false negative noise $\epsilon = \mathbb{P}(C=1|I=0)$. $\mathbb{P}(C=1|I=0)$ is false negative probability that is determined by the cross-view same-cluster probability and the quality of the cluster structure. It is defined as $\mathbb{P}(C=1|I=0)= \frac{1}{K} + \beta r$, where $\beta$ quantifies the negative impact of missing rate $r$ on cluster structure quality. 

Define the number of instance-level positive pairs $N_{ip}$, the number of instance-level negative pairs $N_{in}$, the number of false negative pairs in unpaired observations $N_{fn}$ in views $m,n$ as:
\begin{equation}
\begin{split}
    N_{ip}&=\mathbb{E}[\sum\limits_{i=j}^{N}I({\mathbf {x}}_i^{{m}}, {\mathbf {x}}_j^{{n}})=1]= (1-r)^2N,\\
     N_{in}&=\mathbb{E}[\sum\limits_{i \neq j}^{N}I({\mathbf {x}}_i^{{m}}, {\mathbf {x}}_j^{{n}})=0]= 2r(1-r)N(N-1),\\
    N_{fn}&=\mathbb{E}[\sum\limits_{i\neq j}^{N}\{C({\mathbf {x}}_i^{{m}}, {\mathbf {x}}_j^{{n}})\cdot \mathbb{I}(I({\mathbf {x}}_i^{{m}}, {\mathbf {x}}_j^{{n}})=0)\}=1]\\
    &=2r(1-r)N(N-1)\cdot \mathbb{P}(C=1|I=0) \\
    &= 2r(1-r)N(N-1)\cdot \epsilon,\\
\end{split}
\end{equation}

The objective function $f_{ic}$ is further revised, and its expectation is as follows:
\begin{equation}
\begin{split}
    &f_{ic} = \sum\limits_{m\neq n}^{V}\sum\limits_{i}^{N_{ip}} [\mathcal{\rho}^+(\mathbf{x}_i^{m},\mathbf {x}_i^{n}) - (1+\epsilon)\sum\limits_{j \neq i}^{N_{in}} \mathcal{\rho}^-( \mathbf{x}_i^{{m}}, \mathbf{x}_j^{n})],\\
    &\mathbb{E}[f_{ic}] = V(V-1)\{N_{ip} \cdot \mathbb{E}[\rho^+] -(1 + \epsilon) \cdot  N_{in} \cdot \mathbb{E}[\rho^-]\}
\end{split}
\end{equation}

\begin{itemize}
\item  When maximizing $f_{ic}$, the noise term amplifies $\epsilon$ the penalty for negative pairs by $(1 +\epsilon)$, which suppresses intra-cluster similarity and undermines clustering performance. 
\item  Furthermore, since $N_{ip} \propto \frac{1}{r^2}$, $N_{in} \propto {r^2}$ and $\mathcal{\rho} \propto {r}$, as $r$ increases, the impact of false negative noise $\mathcal{\rho}$ on model performance will also increase.
\end{itemize}

\textbf{Cluster-level paradigms:} 
When $Y=C(\mathbf{x}_i^{{m}},\mathbf{x}_j^{{n}})$, the objective of cluster-level paradigms $f_{cc}$ is formulated as:
\begin{equation}
\begin{split}
    f_{cc} = \sum\limits_{m\neq n}^{V}\sum\limits_{i}^{N} [\mathcal{\rho}^+(\mathbf{x}_i^{m},\mathbf {x}_j^{n}) -\sum\limits_{j \neq i}^{N} \mathcal{\rho}^-( \mathbf{x}_i^{{m}}, \mathbf{x}_j^{n})].
\end{split}
\end{equation}

Due to the different data distribution across views caused by varying missing observations in each view, as well as the lack of clustering interaction among instances within views, there may be inconsistencies in cluster semantics and cluster distributions between view $m$ and $n$, introducing cluster consistency errors $\delta^{m,n}$.

Define $\mathbf{C}^{{v}} = \{\mathbf{c}_k^{{v}}\}_{k=1}^K$ as a set of cluster prototypes for $v$-th view data $\mathbf{X}^{{v}}$ and $p(\mathbf{X}^v|\mathbf{c}^v_k)$ as the probability distribution of $\mathbf{X}^{{v}}$ in the $k$-th cluster. $\delta^{m,n}$ include the following two errors: 
\begin{itemize}
\item Cluster semantic error ${\delta^{m,n}_{se}}$: two observations $\mathbf{x}_i^{{m}},\mathbf{x}_j^{{n}}$ from the same semantic cluster may be assigned to different clusters across views. Formally, when $S(\mathbf{x}_i^{{m}},\mathbf{x}_j^{{n}})=1$, Cluster-level paradigm mistakes $\arg\max\limits_{k} \mathcal{\rho}(\mathbf{x}_i^{m}, \mathbf{c}^v_k) \neq \arg\max\limits_{k'} \mathcal{\rho}(\mathbf{x}_j^{n}, \mathbf{c}^v_{k'})$) and can be quantified as:
\begin{equation}
\begin{split}
    \delta^{m,n}_{se} = \mathbb{A}(\mathbf{C}^{{m}},\mathbf{C}^{{n}}),
\end{split}
\end{equation}
where $\mathbb{A}(\cdot)$ is the cost function for optimally matching the prototypes between views, like cost matrix in Hungarian Algorithm, Optimal transport distance in Optimal Transport and contrastive loss in Contrastive Learning.
\item Cluster distribution error ${\delta^{m,n}_{st}}$: the data distribution of the same semantic cluster $k$ may vary across views. It means $p(\mathbf{X}^m|\mathbf{c}^m_k) \neq p(\mathbf{X}^n|\mathbf{c}^n_k)$ and and can be quantified as:
\begin{equation}
\begin{split}
    \delta^{m,n}_{st} = \sum\limits_k^K \mathbb{D}(p(\mathbf{X}^m|\mathbf{c}^m_k)||p(\mathbf{X}^n|\mathbf{c}^n_k)),
\end{split}
\end{equation}
\end{itemize}
where $\mathbb{D}(\cdot)$ quantifies the difference between the two distributions, like Kullback-Leibler Divergence, Total Variation Distance and Maximum Mean Discrepancy Distance.

Define the number of cluster-level positive pairs $N_{cp}$ and cluster-level negative pairs $N_{cp}$ as:
\begin{equation}
\begin{split}
    N_{cp}&=\mathbb{E}[\sum\limits_{i,j}^{N}C({\mathbf {x}}_i^{{m}}, {\mathbf {x}}_j^{{n}})=1]= (1-r)^2N^2\cdot \mathbb{P}(y_i^m=y_j^n),\\
     N_{cn}&=\mathbb{E}[\sum\limits_{i \neq j}^{N}C({\mathbf{x}}_i^{{m}}, {\mathbf {x}}_j^{{n}})=0]\\
     &= (1-r)^2N^2\cdot (1-\mathbb{P}(y_i^m=y_j^n)),
\end{split}
\end{equation}
where $\mathbb{P}(y_i^m=y_j^n)$ represents the probability that $\mathbf{x}_i^{m}$ and $\mathbf{x}_j^{n}$ belong to the same semantic cluster. If instances are uniformly distributed across $K$ clusters, $\mathbb{P}(y_i^m=y_j^n)=\frac{1}{K}$.

The objective function $f_{cc}$ is further revised, and its expectation is as follows:
\begin{equation}
\begin{split}
    &f_{cc} = \sum\limits_{m\neq n}^{V}\{\sum\limits_{i}^{N_{cp}} [\mathcal{\rho}^+(\mathbf{x}_i^{m},\mathbf {x}_j^{n}) -\sum\limits_{j \neq i}^{N_{cn}} \mathcal{\rho}^-( \mathbf{x}_i^{{m}}, \mathbf{x}_j^{n})]-\delta^{m,n}\},\\
    &\mathbb{E}[f_{cc}]= V(V-1)\{N_{cp} \cdot \mathbb{E}[\rho^+]
    -N_{cn} \cdot \mathbb{E}[\rho^-]-\mathbb{E}[\delta^{m,n}]\},
\end{split}
\end{equation}

\begin{itemize}
\item To ensure cluster semantic and distributions consistency, the cluster-level paradigm needs to optimize error term $\mathbb{E}[\delta^{m,n}]$. However, $\mathbb{E}[\delta^{m,n}]$ cannot be entirely eliminated and can only be minimized, which inevitably degrades the model's performance.
\item Furthermore, the missing rate $r$ disrupts the uniformity of the original cluster distribution ($\mathbb{P}(y_i^m=y_j^n)$ is no longer equal to $\frac{1}{K}$), thereby introducing both false negative and false positive noise in $N_{cp}$ and $N_{cn}$. This perturbation consequently exacerbates the degree of prototype and distribution shifts. As a result, $\mathbb{E}[\delta^{m,n}]$ will increase with $r$. \item Meanwhile, due to $\delta^{m,n} \propto K$ and $\mathbb{E}[f_{cc}] \propto V(V-1)$, an excessive number of clusters and views can cause $\mathbb{E}[\delta^{m,n}]$ to surge, significantly increasing the difficulty of optimization.
\end{itemize}

\textbf{Semantic-level paradigms:} Define the quantities of semantic-level positive and negative pairs:
\begin{equation}
    \begin{split}
    N_{sp} &= \mathbb{E} \left[ \sum_{i,j}^N S(\mathbf{x}_i^m, \mathbf{x}_j^n) = 1 \right] \\
    &= (1 - r)^2 N^2 \cdot \mathbb{P}(y_i^m = y_j^n),\\
    N_{sn} &= \mathbb{E} \left[ \sum_{i\neq j}^N S(\mathbf{x}_i^m, \mathbf{x}_j^n)
    = 0 \right] \\
    &= (1 - r)^2 N^2 \cdot \left( 1 - \mathbb{P}(y_i^m = y_j^n) \right),
    \end{split}
\end{equation}    
where \(\mathbb{P}(y_i^m = y_j^n)\) still represents the same-cluster probability of cross-view observations.

When $Y=S(\mathbf{x}_i^{{m}},\mathbf{x}_j^{{n}})$, the objective of semantic-level paradigms $f_{sc}$ is formulated as:
\begin{equation}
\begin{split}
    f_{sc} = \sum\limits_{m\neq n}^{V}\sum\limits_{i}^{N_{sp}} [\mathcal{\rho}^+(\mathbf{x}_i^{m},\mathbf {x}_j^{n}) -\sum\limits_{j \neq i}^{N_{sn}} \mathcal{\rho}^-( \mathbf{x}_i^{{m}}, \mathbf{x}_j^{n})].
\end{split}
\end{equation}

Compared with IC in False negative noise mitigation: semantic-level positive pairs $N_{sp}$ are defined as $S(\mathbf{x}_i^m, \mathbf{x}_j^n) = 1$, and its false negative noise $\epsilon_{\text{sc}}$ is quantified as:
\begin{equation}
\begin{split}
   \epsilon_{\text{sc}} = \mathbb{P}(\arg\max\limits_{k} \mathcal{\rho}(\mathbf{x}_i^{m}, \mathbf{c}_k) \neq \arg\max\limits_{k} \mathcal{\rho}(\mathbf{x}_j^{n}, \mathbf{c}_k) \mid C=1)
\end{split}
\end{equation}

\begin{itemize}
    \item $N_{sp}$ are constructed through consensus prototypes $\mathbf{C}$, avoiding cross-view matching: \[
   \epsilon_{\text{sc}} = \mathbb{P}(S(\mathbf{x}_i^m, \mathbf{x}_j^n) = 0 \mid C({\mathbf {x}}_i^{{m}}, {\mathbf {x}}_j^{{n}})=1) \approx 0
    \]
    Therefore, $N^{sc}_{fn}\propto \epsilon_{\text{sc}}\approx0$, 
    \item As the prototypes \( \mathbf{C} \) are optimized, the distance between different cluster prototypes \(\|\mathbf{c}_k - \mathbf{c}_{k'}\|\) increases, causing \(\mathbb{P}(\cdot) \propto \exp(-\|\mathbf{c}_k - \mathbf{c}_{k'}\|^2/\sigma^2)\) to decay exponentially. This drives \( N_{fn}^{sc} \to 0 \).

\end{itemize}

Compared with CC in Cluster 
Consistency Errors optimization: According to Definition 3, semantic-level paradigms enforces all views to share the same set of cluster prototypes \( \mathbf{C} \), fundamentally eliminating cross-view cluster semantic ambiguity. This is specifically manifested as:
\begin{itemize}
    \item Cross-view Semantic Consistency of shared Prototypes: \(\forall m, n, \mathbf{c}_k^m = \mathbf{c}_k^n = \mathbf{c}_k \), directly eliminating cluster semantic error \(\delta_{se}^{m,n}\) (i.e., \(\delta_{se}^{m,n} = 0\)).
    \item Implicit Constraint on Distribution Discrepancy: The shared prototypes project data from each view into a common space through the mapping function \(\psi(\cdot)\), causing the distribution discrepancy \(\delta_{st}^{m,n}\) to be constrained by the embedding distance \(\delta_{st}^{m,n} \propto \|\psi(\mathbf{X}^m) - \psi(\mathbf{X}^n)\|^2 \to 0 \), which is automatically minimized during optimization.
     \item False Positive/Negative Suppression: Due to sharing a set of semantic prototypes, the estimation of \(\mathbb{P}(y_i^m = y_j^n)\) remains $\frac{1}{K}$ unaffected by the view missing rate \(r\) (compared to \(\mathbb{P}(y_i^m = y_j^n) \neq 1/K\) in Cluster-level paradigms), thereby avoiding false negatives and false positives.
\end{itemize}

The objective function $f_{sc}$ can be formally expressed in expectation form as:
\begin{equation}
    \mathbb{E}[f_{sc}] = V(V-1)\{N_{sp} \cdot \mathbb{E}[\rho^+]
        -N_{sn} \cdot \mathbb{E}[\rho^-]\}.
\end{equation}
\begin{itemize}
    \item {Confidence and Robustness for Noise $\epsilon$ and Error $\delta$}: Compared to instance-level paradigms (containing explicit noise term $(1 + \epsilon)\mathbb{E}[\rho^-]$) and cluster-level paradigms (containing non-eliminable $\mathbb{E}[\delta^{m,n}]$), the semantic-level objective has no additional noise and error terms, and $N_{fn}^{sc}$, $N_{fp}^{sc}$ decays during optimization.
    \item {Confidence and Robustness for Missing Rate $r$}: Due to the shared prototype constraint, the ratio between $N_{sp}$ and $N_{sn}$ remains stable ($\mathbb{P}(y_i^m = y_j^n) = 1/K$). Even with high $r$, the objective function can still accurately model the cluster structure. 
\end{itemize}
\end{proof}

\subsection{Proof of Theorem 2}\label{sec22}
\begin{theorem}\label{thm:Theorem1} 
Since Paired observations $ (\overline{\mathbf {x}}_i^{{m}},\overline{\mathbf {x}}_i^{n})$ inherently satisfy instance- and cluster-level consistency, they can achieve semantic consensus via a shared set of prototypes $\mathbf{C}$.
\end{theorem}

\begin{proof}
\textbf{Instance-level Consistency:} 
According to Definition 1, paired observations $ (\overline{\mathbf {x}}_i^{{m}},\overline{\mathbf {x}}_i^{n})$ satisfy the condition that both are cross-view observations of the same instance $\mathbf{x}_i$, thus they are instance-level consistency  $I\bigl(\overline{\mathbf{x}}_i^m, \overline{\mathbf{x}}_i^n\bigr) = 1.$
Two observations essentially belong to the same underlying instance, with only view-specific noise or modality discrepancies causing observational differences.

\textbf{Cluster-level Consistency:}
$\overline{\mathbf{x}}_i^m$ and $\overline{\mathbf{x}}_i^n$ are cross-view observations of the same instance, they must belong to the same cluster. According to Definition 2, paired observations $ (\overline{\mathbf {x}}_i^{{m}},\overline{\mathbf {x}}_i^{n})$ are instance-level consistency  $C\bigl(\overline{\mathbf{x}}_i^m, \overline{\mathbf{x}}_i^n\bigr) = 1.$
This further ensures that, in addition to being similar in features, these two observations are also consistent in their cluster structure, indicating that both are grouped into the same semantic cluster across different views.

\textbf{Semantic-level Consistency:}
According to Definition 3, semantic-level consensus requires:
\begin{itemize}
    \item {Shared Cluster Prototypes}: All observations share the same set of prototypes \(\mathbf{C} = \{c_k\}_{k=1}^{K}\).
    \item {Consistent Prototype Assignment}: \(\arg\min\limits_{k} \rho(\overline{x}_i^m, c_k) = \arg\min\limits_{k} \rho(\overline{x}_i^n, c_k)\).
\end{itemize}

Paired observations $ (\overline{\mathbf {x}}_i^{{m}},\overline{\mathbf {x}}_i^{n})$ satisfy the following conditions:
\begin{itemize}
    \item {Condition 1}: Since \( \mathbf{C} \) is globally shared, observations from all views are assigned based on the same set of prototypes.
    \item{Condition 2}:
    Assume the nearest prototype for \(\overline{\mathbf{x}}_i^m\) is \( \mathbf{c}_k \):
    \[
    \arg \min_{k} d(\overline{\mathbf{x}}_i^m,\mathbf{c}_k) = k.
    \]
    Since \(\overline{\mathbf{x}}_i^m\) and \(\overline{\mathbf{x}}_i^n\) belong to the same cluster \( \mathbf{c}_k \) (CC), and prototype \(\mathbf{c}_k \) is the central representation of this cluster, the nearest prototype for \(\overline{\mathbf {x}}_i^n\) should also be \( \mathbf{c}_k \).
    Otherwise, if the nearest prototype for \(\overline{\mathbf{x}}_i^n\) is \( \mathbf{c}_{k'} \) (\( k' \neq k \)), it would contradict the cluster consistency (CC). Therefore, it must satisfy:
    \[
    \arg \min\limits_{k} \rho(\overline{\mathbf {x}}_i^m, \mathbf{c}_k) = \arg \min\limits_{k} \rho(\overline{\mathbf {x}}_i^n, \mathbf{c}_k) = k.\]
\end{itemize}

The conditions of SC all hold. According to Definition 3,  the paired observations $ (\overline{\mathbf {x}}_i^{{m}},\overline{\mathbf {x}}_i^{n})$  have reached semantic-level consensus $S(\mathbf{x}_i^m, \mathbf{x}_i^n) = 1.$

\end{proof}

\section{Appendix C: Experiments}
\subsection{Experimental Settings}\label{sec21}

\textbf{Datasets.} From the perspective of clustering task complexity in the number of clusters, views, feature dimensions, and samples,  six widely applied public datasets are selected for experiments:
\begin{table}[!htbp]
\caption{Multi-view benchmark datasets in experiments.}
\label{tab:dataset}
\vspace{-8pt}
\centering
\resizebox{.49\textwidth}{!}{
\begin{tabular}{ccccc}
\toprule
Dataset         & Samples   & Clusters & Views & Dimensionality \\
\midrule
Yale \cite{zhu2019multi}       & 165       & 15       & 3     & 3304/6750/4096  \\
Caltech-5V\cite{MFLVC}        & 1400       & 7       & 5     & 1984/512/928/254/40  \\
NUSWIDEOBJ10\cite{huang2019multi}   & 6251       & 10      & 5     & 129/74/145/226/65   \\
ALOI-100\cite{du2021deep}       & 10800      & 100     & 4     & 77/13/64/125    \\
YouTubeFace10\cite{huang2023fast}   & 38654      & 10      & 4     & 944/576/512/640  \\
NoisyMNIST\cite{lin2021completer}      & 70000      & 10      & 2     & 784/784          \\
\bottomrule
\end{tabular}}
\label{tab1:dataset}
\end{table}

\textbf{Competitors.} To validate the effectiveness of our model from the perspective of consistency learning, imputation and alignment, we select seven state-of-the-art methods as competitors and summarize them in Table \ref{tab:sota methods} according to the consistency, imputation and alignment techniques they employ. 
\begin{itemize}
\item CPM-Net \cite{zhang2020deep}, encodes view-specific information into a common representation based on instance-level consistency and employs GANs to impute missing data across views.
\item COMPLETER \cite{lin2021completer}, maximize mutual information and minimize conditional entropy across views based on instance-level consistency to achieve contrastive representation learning and duel missing prediction.
\item DIMVC \cite{xu2022deep}, performs instance-level contrastive learning to construct a common representation, while aligns view-specific cluster assignments with the common assignment for decision fusion.
\item SURE \cite{yang2022robust}, introduces an adaptive distance threshold for positive-negative pairs to identify and penalize false negative pairs, enabling cluster-level contrastive learning. Additionally, it transfers the cluster relationships from other complete views to the missing views for imputation.  
\item ProImp \cite{li2023incomplete}, conducts instance-level contrastive learning and prototypes alignment to ensure consistency across views, then fills in missing observations by referring to prototypes in the missing views and the observation-prototype relationships in other complete views. 
\item ICMVC \cite{chao2024incomplete}, transfers graph relationships from complete views to missing views for imputation based on instance-level consistency. To further enhance consistency in cluster assignments, it constrains view-specific assignments to align with the high-confidence common representation.  
\item DIVIDE \cite{lu2024decoupled}, leverages random walks to progressively discover positive and negative pairs for cross-view cluster alignment. Through cluster-level contrastive learning, it explores cross-view consistency information to recover missing views.
\end{itemize}

\begin{table}[!htbp]
\caption{SOTA methods categorized by the types of techniques for consistency, imputation, and alignment.}
\label{tab:methods}
\centering
\resizebox{0.49\textwidth}{!}{
\begin{tabular}{cccc}
\toprule
Competitors & Consistency & Imputation & Alignment \\
\midrule
{CPM-Nets (TPAMI'20)}       & instance-level       & mutual information interaction & \textbackslash    \\
{COMPLETER (CVPR'21)}       & instance-level       & mutual information interaction & \textbackslash     \\
{DIMVC (AAAI'22)}          & instance-level       & \textbackslash & assignment-based      \\
{SURE (TPAMI'23)}          & cluster-level        & graph structure transfer & \textbackslash  \\
{ProImp (IJCAI'23)}        & instance-level       & sample-prototype relationship inheritance & prototype-based\\
{ICMVC (AAAI'24)}          & instance-level       & graph structure transfer & assignment-based\\
{DIVIDE (AAAI'24)}         & cluster-level        & mutual information interaction & \textbackslash \\
\bottomrule
\end{tabular}}
\end{table}

\subsection{Implementation details}\label{experi:ablation}
Our model consists of three modules: reconstruction (REC) module, consistency semantic learning (CSL) module and cluster semantic enhancement (CSE) module, as well as four components: encoder, decoder, contrastive clustering and graph clustering. The implementation details are as follows:
\begin{table}[!htbp]
\caption{FreeCSL architecture details.}
\label{tab:methods}
\vspace{-5pt}
\centering
\resizebox{0.49\textwidth}{!}{
\begin{tabular}{ccc}
\toprule
Component    & Layer    & Dimension  \\
\midrule
Encoder    & 4-layer MLPs   &  $\textit{view\_dim} \rightarrow 500 \rightarrow 500 \rightarrow 2000 \rightarrow 64 $ \\
Decoder    & 4-layer MLPs   &  $  64 \rightarrow 2000 \rightarrow 500 \rightarrow 500 \rightarrow \textit{view\_dim}$  \\
contrastive clustering    & 1-layer FC   &  $64 \rightarrow 64 $\\
graph clustering   & 2-layer GCNs and 1-layer FC   &  $64 \rightarrow 128 \rightarrow 64 \rightarrow \textit{cluster\_num} $\\
\bottomrule
\end{tabular}}
\label{tab1:methods}
\end{table}

\subsection{Competitiveness of FreeCSL}\label{sec21}
\begin{table*}[!htbp]
\centering
\caption{Clustering performance comparisons on Yale and NUSWIDEOBJ10. The best and second - best results are highlighted in \textcolor{red}{red} and \textcolor{blue}{blue}.}
\vspace{-5pt}
\renewcommand{\arraystretch}{0.9} 
\resizebox{1\textwidth}{!}{
\begin{tabular}{@{\hspace{10pt}}cc|ccc|ccc|ccc|ccc@{\hspace{10pt}}}
\toprule
\multirow{2}{*}{} & Missing rates & \multicolumn{3}{c|}{$r = 0.1$} & \multicolumn{3}{c|}{$r = 0.3$} & \multicolumn{3}{c|}{$r = 0.5$} & \multicolumn{3}{c}{$r = 0.7$} \\
\cmidrule(lr){2 - 14}
& Metrics & ACC (\%) & NMI (\%) & ARI (\%) & ACC (\%) & NMI (\%) & ARI (\%) & ACC (\%) & NMI (\%) & ARI (\%) & ACC (\%) & NMI (\%) & ARI (\%) \\
\midrule
\multirow{8}{*}{\rotatebox{90}{\textbf{Yale}}}
 & {CPM-Nets} & 54.24 & 60.82 & 37.55 & 56.66 & \textcolor{blue}{63.25} & \textcolor{blue}{40.22} & 53.34 & 59.58 & 34.22 & \textcolor{blue}{55.76} & \textcolor{blue}{58.20} & \textcolor{blue}{33.10} \\
 & {COMPLETER} & 29.09 & 37.10 & 2.36  & 20.30 & 29.61 & 1.20  & 16.97 & 26.08 & 0.97  & 10.91 & 16.88 & 0.32  \\
 & {DIMVC} & 27.91 & 32.27 & 7.94  & 23.12 & 26.79 & 2.85  & 21.76 & 26.92 & 3.46  & 34.32 & 39.47 & 11.21 \\
 & {SURE} & 42.30 & 49.57 & 22.12 & 38.91 & 43.90 & 13.61 & 34.30 & 39.07 & 9.08  & 25.33 & 33.79 & 3.92  \\
 &{ProImp} & \textcolor{blue}{57.98} & \textcolor{blue}{63.37} & \textcolor{blue}{38.95} & \textcolor{blue}{56.77} & 60.43 & 35.54 & \textcolor{blue}{55.96} & \textcolor{blue}{58.47} & \textcolor{blue}{32.91} & 52.12 & 56.11 & 30.19 \\
 & {ICMVC} & 49.70 & 61.52 & 30.64 & 50.30 & 61.62 & 31.49 & 46.67 & 58.91 & 27.84 & 42.42 & 54.43 & 23.21 \\
 & {DIVIDE} & 55.15 & 56.37 & 28.97 & 42.42 & 45.38 & 18.25 & 32.12 & 35.80 & 8.40  & 30.91 & 32.03 & 6.48  \\
 & \textbf{{Our}} &\textcolor{red}{ 62.42} & \textcolor{red}{65.87} & \textcolor{red}{45.71} & \textcolor{red}{60.00} & \textcolor{red}{64.73} & \textcolor{red}{41.33} & \textcolor{red}{60.00} & \textcolor{red}{63.14} & \textcolor{red}{40.85} & \textcolor{red}{60.61} & \textcolor{red}{60.30} & \textcolor{red}{37.69} \\ 
\midrule
\multirow{8}{*}{\rotatebox{90}{\textbf{NUSWIDEOBJ10}}}
 & CPM-Nets & 21.07 & 7.76 & 3.93 & 22.39 & 6.88 & 3.97 & 21.18 & 5.97 & 3.06 & 20.24 & 4.60 & 1.86 \\
 & COMPLETER & 23.38 & 8.16 & 2.58 & 21.36 & 9.90 & 4.61 & \textcolor{blue}{23.34} & 9.94 & 4.60 & \textcolor{blue}{23.48} & {10.96} & 5.37 \\
 & DIMVC & 22.51 & 11.46 & 6.61 & 21.33 & 11.89 & 5.43 & 21.26 & 10.64 & 5.03 & 23.04 & 10.40 & 5.68 \\
 & SURE & 20.87 & 10.90 & 5.39 & 21.83 & 11.24 & 6.07 & 21.93 & 11.14 & 5.92 & 22.78 & 10.54 & \textcolor{blue}{6.16} \\
 & ProImp & 22.81 & 11.31 & 5.85 & 22.88 & 11.40 & 6.11 & 23.26 & 11.20 & 6.20 & 22.55 & \textcolor{blue}{11.24} & 5.94 \\
 & ICMVC & 20.92 & 10.15 & 5.06 & 21.10 & 10.59 & 5.19 & 20.89 & 10.20 & 5.04 & 20.09 & 9.58 & 5.06 \\
 & DIVIDE & \textcolor{blue}{23.95} & \textcolor{blue}{12.97} & \textcolor{blue}{7.75} & \textcolor{blue}{24.24} & \textcolor{blue}{13.22} & \textcolor{blue}{7.67} & 22.81 & \textcolor{blue}{12.90} & \textcolor{blue}{7.43} & 23.45 & 10.78 & 6.05 \\
 & \textbf{Ours} & \textcolor{red}{25.61} & \textcolor{red}{16.31} & \textcolor{red}{8.75} & \textcolor{red}{24.68} & \textcolor{red}{15.14} & \textcolor{red}{7.95} & {\textcolor{red}{24.03}} & {\textcolor{red}{14.10}} &\textcolor{red} {{7.69}} & \textcolor{red}{{23.88}} & {\textcolor{red}{12.67}} & {\textcolor{red}{6.82}} \\
\bottomrule
\end{tabular}}
\label{tab:nud}
\end{table*}

To further enhance the credibility of our model, we supply a comparative experiment on Yale and NUSWIDEOBJ10 dataset, and present the comparison results, along with the visualizations based on ACC and NMI metric, in Table \ref{tab:nud} and Fig. \ref{fig:two performance}. As mentioned in our main text, FreeCSL surpasses all competitors and demonstrates more stable performance in various missing rates even on the small-sample dataset Yale, as FreeCSL avoids the errors associated with imputation and alignment.


\begin{figure}[t]  
    \centering  
    \begin{subfigure}[t]{0.49\linewidth}  
        \includegraphics[width=\linewidth]{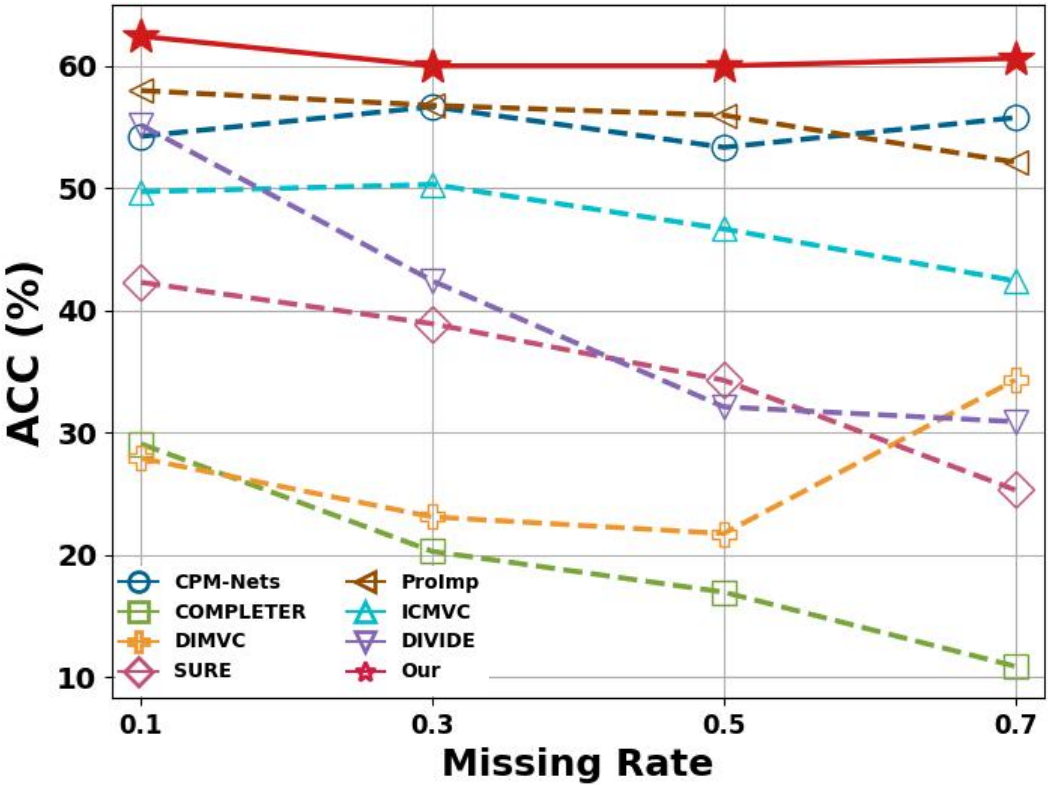}  
        \caption{Yale}  
    \end{subfigure}  
    \hfill
    \begin{subfigure}[t]{0.49\linewidth}  
        \includegraphics[width=\linewidth]{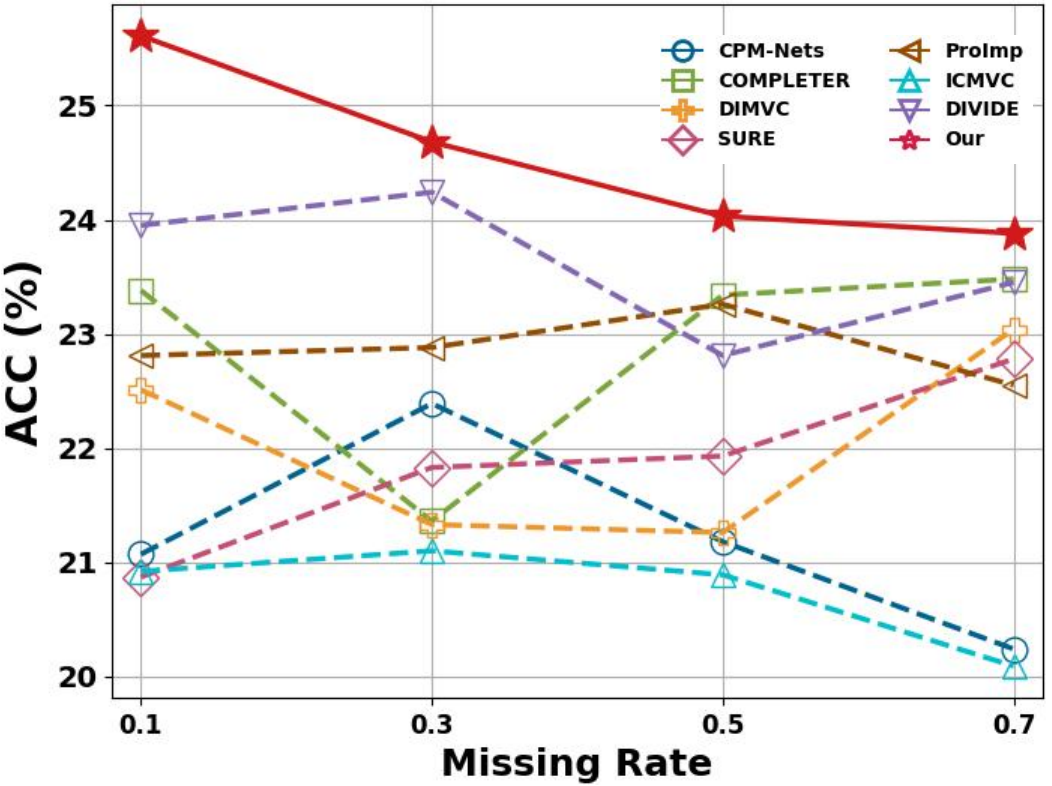}  
        \caption{NUSWIDEOBJ10}  
    \end{subfigure}
    \vspace{-5pt}
    \caption{Visualization for Table \ref{tab:nud} based on metric ACC.}
\label{fig:two performance}
\end{figure}

\subsection{Understanding FreeCSL}\label{sec21}
\begin{table*}[!htbp]
\centering
\caption{Ablation study on YoutubeFace10, NoisyMNIST, Yale and NUSWIDEOBJ10. $\checkmark$ denotes FreeCSL with
the component and the best results are highlighted in \textcolor{red}{red}.}
\vspace{-5pt}
\renewcommand{\arraystretch}{1} 
\resizebox{1\textwidth}{!}{
\begin{tabular}{@{\hspace{10pt}}cccc|ccc|ccc|ccc|ccc@{\hspace{10pt}}}
\toprule
\noalign{\vspace{-1pt}}
\multirow{2}{*}{} &\multicolumn{3}{c|}{Components} &\multicolumn{3}{c|}{$r=0.1$} & \multicolumn{3}{c|}{$r=0.3$} & \multicolumn{3}{c|}{$r=0.5$} & \multicolumn{3}{c}{$r=0.7$} \\
\noalign{\vspace{-2pt}} \cmidrule(lr){2-16}
\noalign{\vspace{-2pt}}
& $\mathcal{L}_{rec}$ & $\mathcal{L}_{cc}$ & $\mathcal{L}_{gc}$& ACC (\%) & NMI (\%) & ARI (\%) & ACC (\%) & NMI (\%) & ARI (\%) & ACC (\%) & NMI (\%) & ARI (\%) & ACC (\%) & NMI (\%) & ARI (\%) \\
\noalign{\vspace{-2pt}}\\\midrule

\multirow{4}{*}{\rotatebox{90}{\scriptsize\textbf{YouTube10}}}
 & \checkmark &   & &71.57 & 75.65 & 64.43 & 68.37 & 75.63 & 63.87 & 65.28 & 69.41 & 55.66 & 59.63 & 62.36 & 46.51 \\
& \checkmark &  \checkmark & &79.20 & 81.58 & 71.60 & 77.13 & 80.07 & 68.86 & 75.13 & 79.17 & 65.38 & 71.80 & 75.95 & 65.24 \\
& \checkmark &  &\checkmark&76.55 & 80.13 & 69.21 & 72.89 & 70.94 & 63.68 & 72.08 & 68.06 & 61.62 & 68.88 & 64.71 & 56.58 \\
& \checkmark &  \checkmark &\checkmark &{\textcolor{red}{82.93}} & {\textcolor{red}{83.55}} & {\textcolor{red}{74.76}} & {\textcolor{red}{80.77}} & {\textcolor{red}{81.46}} & {\textcolor{red}{71.62}} & {\textcolor{red}{80.19}} & {\textcolor{red}{81.07}} & {\textcolor{red}{71.37}} & {\textcolor{red}{76.62}} & {\textcolor{red}{81.31}} & {\textcolor{red}{73.22}} \\
\midrule
\multirow{4}{*}{\rotatebox{90}{\scriptsize\textbf{NoisyMNIST}}}
 & \checkmark &   &&33.45 & 26.44 & 16.81 & 25.25 & 14.80 & 8.01 & 23.52 & 15.46 & 7.46 & 24.17 & 15.46 & 6.68 \\
& \checkmark &  \checkmark & &{98.17} & {97.04} & {96.97} &{96.27} & {93.69} & {93.91} & {95.25} & {89.08} & {90.38} & 90.96 & 81.15 & 82.37 \\
& \checkmark &  &\checkmark&53.38 & 50.60 & 37.16 & 39.24 & 37.59 & 20.84 & 33.89 & 29.02 & 14.19 & 33.08 & 26.70 & 14.23 \\
& \checkmark &  \checkmark &\checkmark  & \textcolor{red}{{99.13}} & \textcolor{red}{{97.23}} & \textcolor{red}{{98.10}} & \textcolor{red}{{97.68}} & \textcolor{red}{{93.94}} & \textcolor{red}{{94.94}} & \textcolor{red}{{96.04}} & \textcolor{red}{{89.81}} & \textcolor{red}{{91.48}} & {\textcolor{red}{92.19}} & {\textcolor{red}{82.50}} & {\textcolor{red}{83.56}} \\
\midrule
\multirow{4}{*}{\rotatebox{90}{\scriptsize\textbf{Yale}}}
 & \checkmark &   &&50.91 & 60.79 & 36.40 & 44.85 & 50.61 & 25.49 & 34.55 & 44.81 & 17.03 & 33.33 & 40.25 & 12.86 \\
& \checkmark &  \checkmark & &55.15 & 58.99 & 34.75 & 56.97 & 59.11 & 36.01 & 56.97 & 61.71 & 61.71 & 56.36 & 59.57 & 35.19 \\
& \checkmark &  &\checkmark&54.55 & 58.72 & 33.98 & 46.06 & 46.06 & 23.87 & 36.36 & 47.59 & 19.45 & 35.15 & 42.16 & 12.20 \\
& \checkmark &  \checkmark &\checkmark  &\textcolor{red} {62.42} & \textcolor{red} {65.87} & \textcolor{red} {45.71} & \textcolor{red} {60.00} & \textcolor{red} {64.73} & \textcolor{red} {42.14} & \textcolor{red} {60.00} & \textcolor{red} {63.14} & \textcolor{red} {40.85} & \textcolor{red} {60.61} & \textcolor{red} {60.30} & \textcolor{red} {37.69} \\
\midrule
\multirow{4}{*}{\rotatebox{90}{\scriptsize\textbf{NUSWIDEOBJ}}}
 & \checkmark &   &&19.32 & 5.69 & 2.70 & 18.00 & 3.80 & 1.57 & 17.65 & 2.96 & 0.57 & 19.20 & 3.57 & 0.45 \\
& \checkmark &  \checkmark & &23.68 & 16.13 & 8.50 & 23.56 & 14.84 & 7.62 & 23.36 & 13.63 & 7.29 & 22.86 & 12.60 & 6.50 \\
& \checkmark &  &\checkmark&23.23 & 9.67 & 4.98 & 23.63 & 8.21 & 4.95 & 20.83 & 5.67 & 2.96 & 20.22 & 6.10 & 2.42 \\
& \checkmark &  \checkmark &\checkmark  & \textcolor{red}{25.61} & \textcolor{red}{16.31} & \textcolor{red}{8.75} & \textcolor{red}{24.68} & \textcolor{red}{15.14} & \textcolor{red}{7.95} & {\textcolor{red}{24.03}} & {\textcolor{red}{14.10}} &\textcolor{red} {{7.69}} & \textcolor{red}{{23.88}} & {\textcolor{red}{12.67}} & {\textcolor{red}{6.82}} \\
\bottomrule
\end{tabular}}
\label{tab:three ablation study}
\end{table*}

\textbf{Ablation Study.} 
The proposed FreeCSL contains three modules: reconstruction (REC) module, cross-view consistency semantic learning (CSL) module, and within-view cluster semantic enhancement (CSE) module. To further verify the importance of each module, we conducted extra ablation experiments on YoutubeFace10, NoisyMNIST, Yale and NUSWIDEOBJ10 datasets as shown in Table \ref{tab:three ablation study}. With the REC module as the baseline, both CSE module and CSE module contribute significantly to the improved performance of all datasets. Furthermore, due to the synergistic effect of the three modules, our model exhibits more confident and stable performance across different missing rates compared to the ablation group.

\textbf{Imputation- and Alignment-free CSL.}
To demonstrate our model can learn semantic knowledge from view data and achieve consistent and reliable clustering assignments without imputation or alignment, we make efforts in two aspects: conducting imputation experiments and visualizing similarity matrices, both based on latent and semantic representations learned from YoutubeFace10, NoisyMNIST, and NUSWIDEOBJ10 datasets. 

\begin{table*}[!htbp]
\caption{Imputation- and alignment-free study on YoutubeFace10, NoisyMNIST, Yale and NUSWIDEOBJ10. ILR and ISR are filled by K-NN imputation via cross - view
graph for and semantic representations $\mathbf{Z}^{(v)}$, $\mathbf{H}^{(v)}$. The best results are highlighted in \textcolor{red}{red}.}
\label{tab:Three Imputation study}
\vspace{-5pt}
\renewcommand{\arraystretch}{1} 
\centering
\resizebox{0.99\textwidth}{!}{
\begin{tabular}{@{\hspace{10pt}}cc|ccc|ccc|ccc|ccc@{\hspace{10pt}}}
\toprule
\multirow{2}{*}{} & Missing rates & \multicolumn{3}{c|}{$r = 0.1$} & \multicolumn{3}{c|}{$r = 0.3$} & \multicolumn{3}{c|}{$r = 0.5$} & \multicolumn{3}{c}{$r = 0.7$} \\
\cmidrule(lr){2 - 14}
& Metrics & ACC (\%) & NMI (\%) & ARI (\%) & ACC (\%) & NMI (\%) & ARI (\%) & ACC (\%) & NMI (\%) & ARI (\%) & ACC (\%) & NMI (\%) & ARI (\%) \\
\midrule

\multirow{3}{*}{\rotatebox{90}{\scriptsize\textbf{YouTube10}}}
&ILR&82.66 & 82.79 & 74.20 & 80.46 & 81.18 & 71.54 & 80.37 & 81.28 & 71.58 & 73.68 & 75.70 & 63.39 \\
&ISR&82.72 & 82.86 & 72.69 & \textcolor{red}{81.07} & \textcolor{red}{82.63} & \textcolor{red}{72.82} & \textcolor{red}{80.63} & \textcolor{red}{81.67} & \textcolor{red}{72.00} & 73.91 & 75.84 & 63.81 \\
&FreeCSL&{\textcolor{red}{82.93}} & {\textcolor{red}{83.55}} & {\textcolor{red}{74.76}} & {{80.77}} & {{81.46}} & {{71.62}} & {{80.19}} & {{81.07}} & {{71.37}} & {\textcolor{red}{76.62}} & {\textcolor{red}{81.31}} & {\textcolor{red}{73.22}} \\
\midrule
\multirow{3}{*}{\rotatebox{90}{\tiny\textbf{NoisyMNIST}}}
&ILR&99.12 & 97.21 & 98.08 & \textcolor{red}{98.06} & \textcolor{red}{94.50} & \textcolor{red}{95.83} & 95.98 & 89.74 & 91.10 & 90.99 & 80.19 & 81.16 \\
&ISR&\textcolor{red}{99.15} & \textcolor{red}{97.31} & \textcolor{red}{98.15}& 97.83 & 93.86 & 95.28 & 95.80 & 89.23 & 90.98 & 90.69 & 79.76 & 80.57 \\
&FreeCSL&{{99.13}} & {{97.23}} & \textcolor{red}{{98.10}} & {{97.68}} & {{93.94}} & {{94.94}} & \textcolor{red}{{96.04}} & \textcolor{red}{{89.81}} & \textcolor{red}{{91.48}} & {\textcolor{red}{92.19}} & {\textcolor{red}{82.50}} & {\textcolor{red}{83.56}} \\
\midrule
\multirow{3}{*}{\rotatebox{90}{\textbf{Yale}}}
&ILR&55.15 & 61.63 & 37.54 & 56.36 & 62.72 & 40.53 & 53.33 & 59.89 & 35.05 & 50.30 & 55.21 & 29.07 \\
&ISR&58.18 & 60.84 & 37.37 & \textcolor{red}{60.00} & 64.26 & \textcolor{red}{42.48} & 56.97 & 60.33 & 36.39 & 52.73 & 55.56 & 29.91 \\
&FreeCSL&\textcolor{red} {62.42} & \textcolor{red} {65.87} & \textcolor{red} {45.71} & \textcolor{red} {60.00} & \textcolor{red} {64.73} & {42.14} & \textcolor{red} {60.00} & \textcolor{red} {63.14} & \textcolor{red} {40.85} & \textcolor{red} {60.61} & \textcolor{red} {60.30} & \textcolor{red} {37.69} \\
\midrule
\multirow{3}{*}{\rotatebox{90}{\tiny\textbf{NUSWIDEOBJ}}}
&ILR&24.09 & 15.29 & 7.48 & 24.50 & 14.06 & 7.39 & 22.32 & 12.67 & 5.79 & 22.78 & 11.39 & 5.32 \\
&ISR&24.22 & \textcolor{red}{16.44} & 8.53 & \textcolor{red}{25.07} & \textcolor{red}{15.26} & \textcolor{red}{8.33} & 23.93 & 13.79 & 7.13 & 22.45 & 11.83 & 6.16 \\
&FreeCSL&\textcolor{red}{25.61} & {16.31} & \textcolor{red}{8.75} & {24.68} & {15.14} & {7.95} & {\textcolor{red}{24.03}} & {\textcolor{red}{14.10}} &\textcolor{red} {{7.69}} & \textcolor{red}{{23.88}} & {\textcolor{red}{12.67}} & {\textcolor{red}{6.82}} \\
\bottomrule
\end{tabular}}
\end{table*}

Notably, both the latent and semantic representations $\{\mathbf{Z}^v\}_{v=1}^V, \{\mathbf{H}^v\}_{v=1}^V$ are outputs of our model after training. The latent representation $\mathbf{Z}^v$ refers to the output after the decoder but before the CSL module, while the semantic representation $\mathbf{H}^v$ has undergone nonlinear mapping through the CSL module. We impute the missing views for two sets $\{\mathbf{Z}^v\}_{v=1}^V$ and $\{\mathbf{H}^v\}_{v=1}^V$, with mean values based on the neighborhood relationships observed in complete view data. Finally, we perform K-means on consensus representations $\mathbf{Z}$ and $\mathbf{H}$ fused by the representation fusion manner $\mathbb{T}(\{\mathbf{Z}^v\}_{v=1}^V$, $\mathbb{T}(\{\mathbf{H}^v\}_{v=1}^V$ described in Section 2.3 of our main text.

In Table \ref{tab:Three Imputation study}, at small missing rates, our model performs comparably regardless of whether the missing data are imputed or not. As the missing rate increases and the available information for imputation decreases, our model without imputation exhibit superior robustness. Improper imputation introduces noise, while our model, combining the CSL and CSE modules, successfully captures semantic knowledge from view data (embedded in both latent and semantic representations) and leveraging the fusion method $\mathcal{T}(\cdot)$, effectively integrate the consistency and complementary information across views. Thus, our FreeCSL achieves excellent performance without incurring extra computational cost or suffering clustering accuracy loss arising from imputation.

We visualize the cosine similarity matrices of the latent representations $\{\mathbf{Z}^v\}_{v=1}^V$, semantic representations $\{\mathbf{H}^v\}_{v=1}^V$, and their consensus representations $\mathbf{Z}$, $\mathbf{H}$ learned from YoutubeFace10, NoisyMNIST, Yale and NUSWIDEOBJ10 datasets in Fig. \ref{fig:you}-\ref{fig:6}, further confirming the advantages of our model in consensus semantic learning. The experimental results on Four datasets commonly reflect two findings:
\begin{itemize}
    \item The similarity matrices of semantic representations, compared to latent ones, show a clearer and more uniform block structure along the diagonal. This indicates that the semantic representations, jointly optimized by  the CSL and CSE modules, are well-suited for clustering task.
    \item Our consensus prototype-based semantic learning and  consensus representation-based semantic clustering, effectively reduces entropy within clusters and enhances more confident assignments by integrating view-specific information.
\end{itemize}

\begin{figure}[htbp]
    \centering
    \begin{minipage}{0.5\textwidth}
        \begin{subfigure}{0.189\textwidth}
            \centering
            \includegraphics[width=\linewidth]{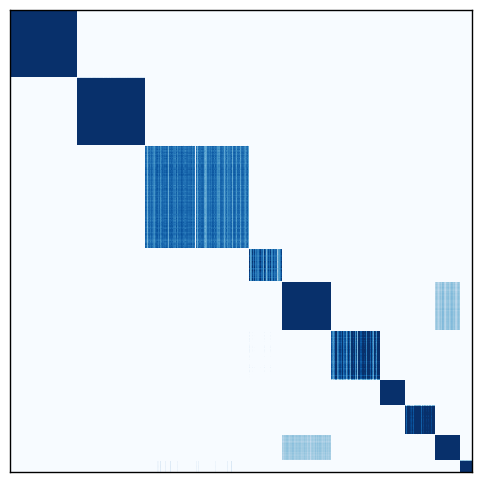}
            \caption{$\mathbf{Z}^{(1)}$}
        \end{subfigure}%
        \hfill
        \begin{subfigure}{0.189\textwidth}
            \centering
            \includegraphics[width=\linewidth]{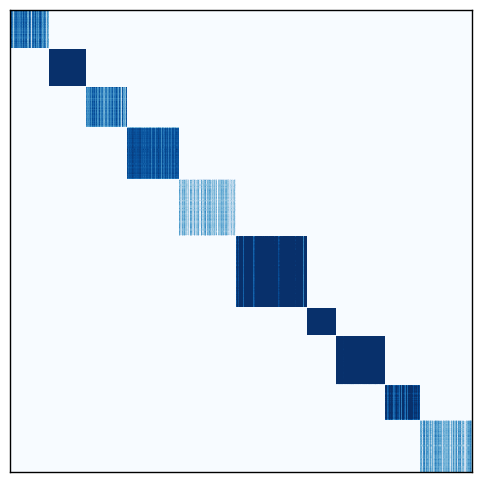}
            \caption{$\mathbf{Z}^{(2)}$}
        \end{subfigure}%
        \hfill
        \begin{subfigure}{0.189\textwidth}
            \centering
            \includegraphics[width=\linewidth]{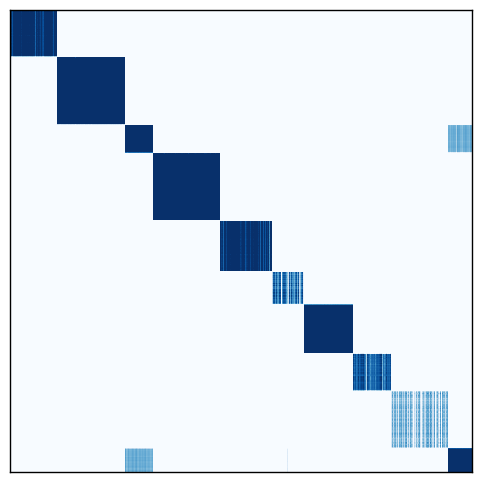}
            \caption{$\mathbf{Z}^{(3)}$}
        \end{subfigure}%
        \hfill
        \begin{subfigure}{0.189\textwidth}
            \centering
            \includegraphics[width=\linewidth]{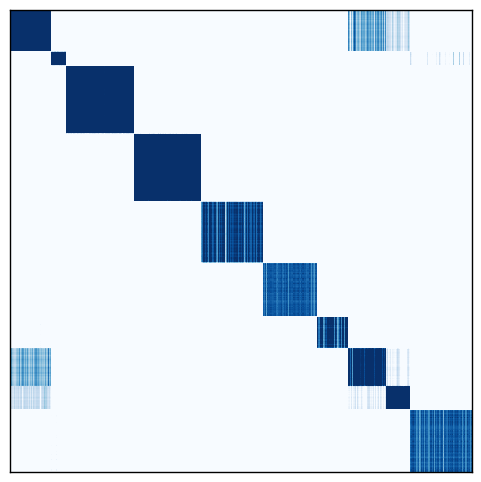}
            \caption{$\mathbf{Z}^{(4)}$}
        \end{subfigure}%
        \hfill
        \begin{subfigure}{0.19\textwidth}
            \centering
            \includegraphics[width=\linewidth]{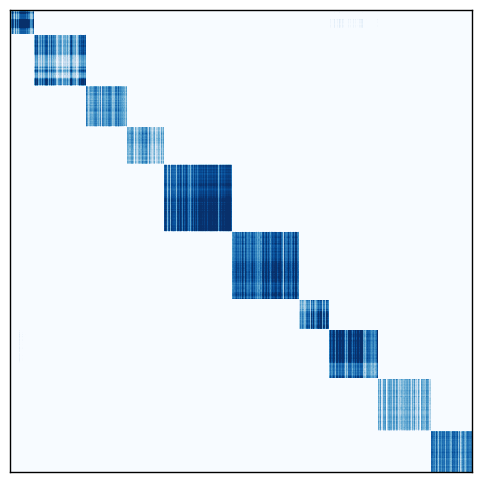}
            \caption{$\mathbf{Z}$}
        \end{subfigure}
    \end{minipage}
    \vspace{-10pt}
    \caption{Similarity matrices of $ \{\mathbf{Z}^{{v}}\}_{v=1}^{4}$, $\mathbf{Z}$ on YouTubeFace10 with $r=0.5$.}
    \label{fig:you}
\end{figure}
\begin{figure}[htbp]
    \centering
    \begin{minipage}{0.5\textwidth}
        \begin{subfigure}{0.19\textwidth}
            \centering
            \includegraphics[width=\linewidth]{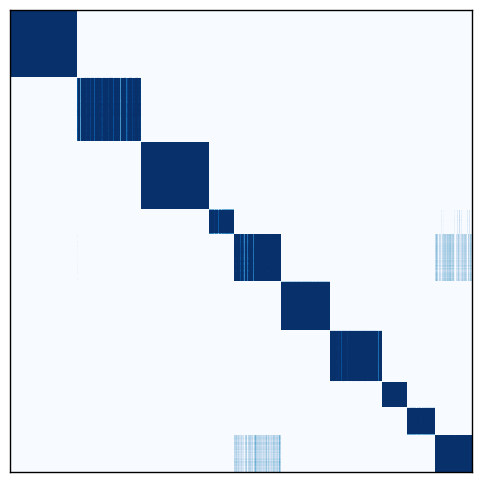}
            \caption{$\mathbf{H}^{(1)}$}
        \end{subfigure}%
        \hfill
        \begin{subfigure}{0.19\textwidth}
            \centering
            \includegraphics[width=\linewidth]{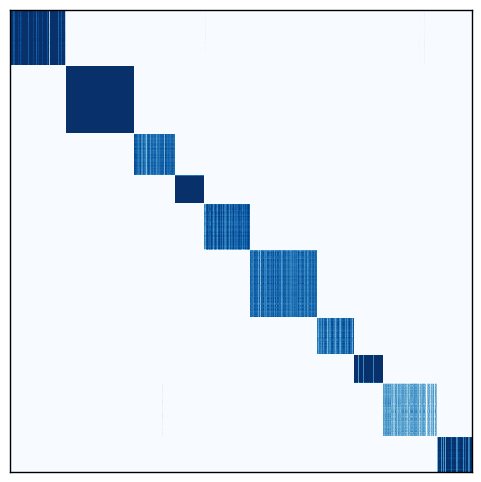}
            \caption{$\mathbf{H}^{(2)}$}
        \end{subfigure}%
        \hfill
        \begin{subfigure}{0.19\textwidth}
            \centering
            \includegraphics[width=\linewidth]{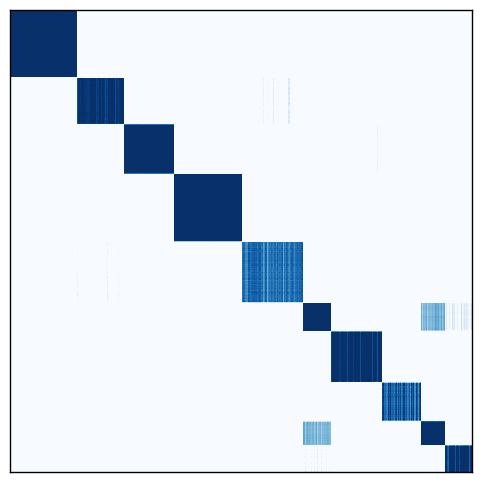}
            \caption{$\mathbf{H}^{(3)}$}
        \end{subfigure}%
        \hfill
        \begin{subfigure}{0.19\textwidth}
            \centering
            \includegraphics[width=\linewidth]{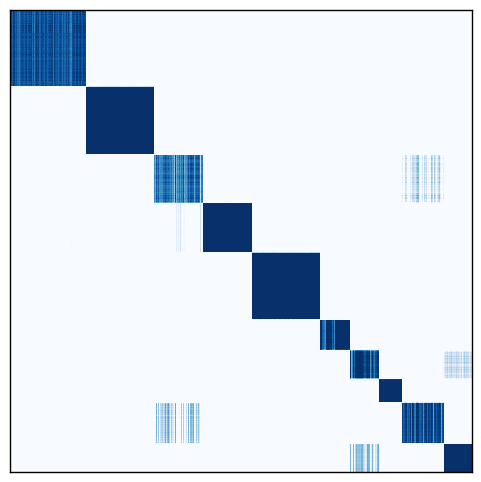}
            \caption{$\mathbf{H}^{(4)}$}
        \end{subfigure}%
        \hfill
        \begin{subfigure}{0.19\textwidth}
            \centering
            \includegraphics[width=\linewidth]{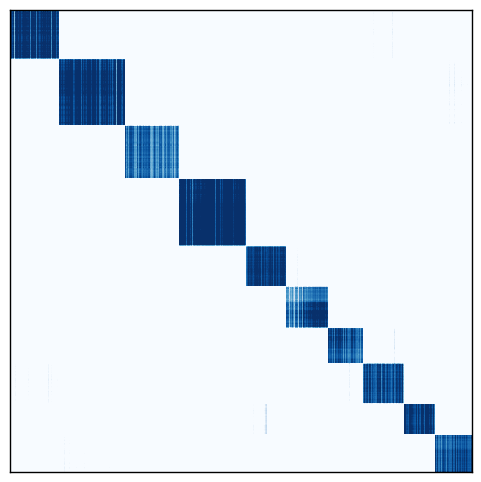}
            \caption{$\mathbf{H}$}
        \end{subfigure}
    \end{minipage}
    \vspace{-10pt}
    \caption{Similarity matrices of $ \{\mathbf{H}^{{v}}\}_{v=1}^{4}$, $\mathbf{H}$ on YouTubeFace10 with $r=0.5$.}
    \label{fig:two-row}
\end{figure}
\vspace{-2pt}
\begin{figure}[htbp]
    \centering
    \begin{minipage}{0.5\textwidth}
        \begin{subfigure}{0.15\textwidth}
            \centering
            \includegraphics[width=\linewidth]{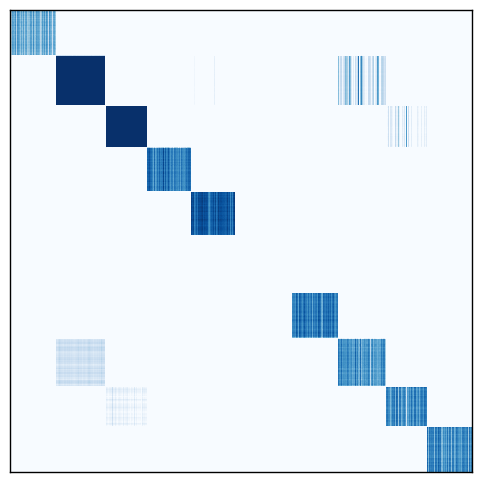}
            \caption{$\mathbf{Z}^{(1)}$}
        \end{subfigure}%
        \hfill
        \begin{subfigure}{0.15\textwidth}
            \centering
            \includegraphics[width=\linewidth]{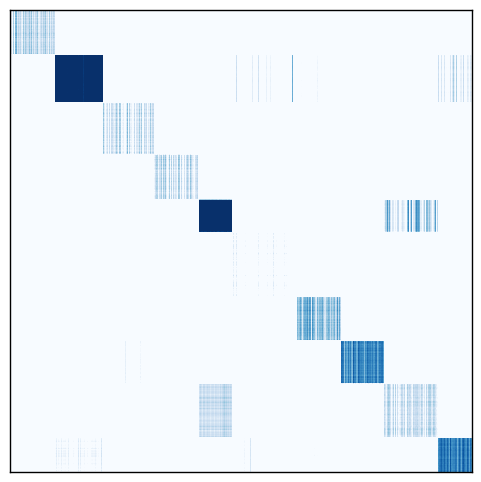}
            \caption{$\mathbf{Z}^{(2)}$}
        \end{subfigure}%
        \hfill
        \begin{subfigure}{0.15\textwidth}
            \centering
            \includegraphics[width=\linewidth]{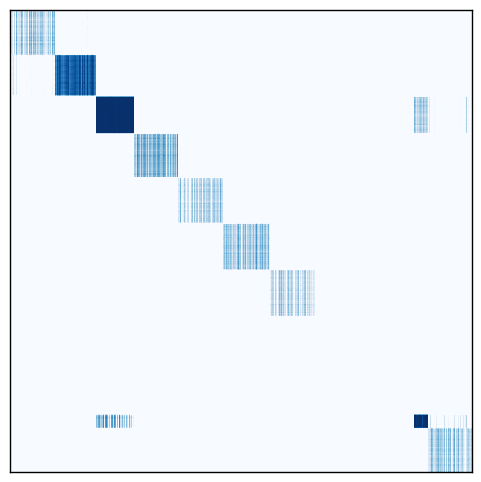}
            \caption{$\mathbf{Z}$}
        \end{subfigure}%
        \hfill
        \begin{subfigure}{0.15\textwidth}
            \centering
            \includegraphics[width=\linewidth]{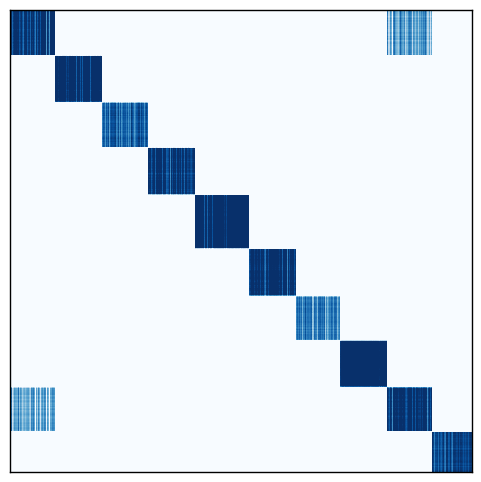}
            \caption{$\mathbf{H}^{(1)}$}
        \end{subfigure}%
        \hfill
        \begin{subfigure}{0.15\textwidth}
            \centering
            \includegraphics[width=\linewidth]{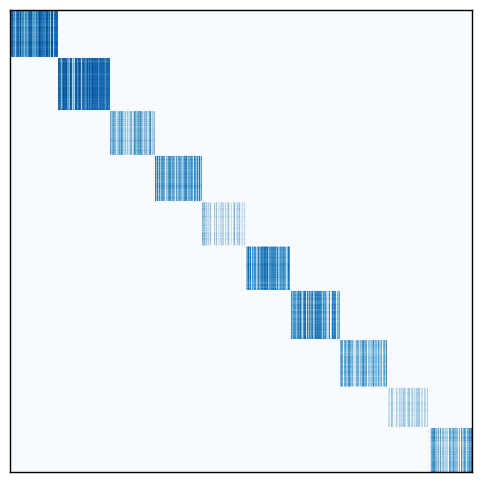}
            \caption{$\mathbf{H}^{(2)}$}
        \end{subfigure}%
        \hfill
        \begin{subfigure}{0.15\textwidth}
            \centering
            \includegraphics[width=\linewidth]{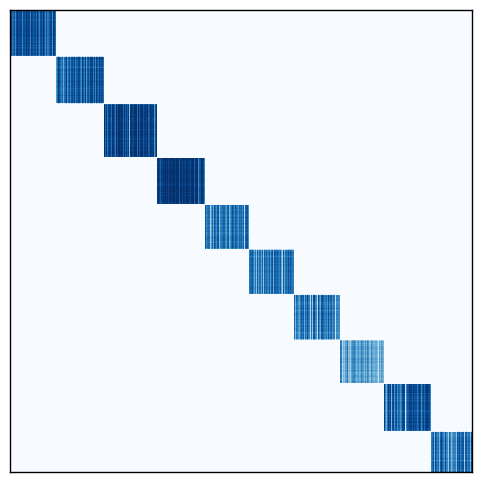}
            \caption{$\mathbf{H}$}
        \end{subfigure}
    \end{minipage}
    \vspace{-10pt}
    \caption{Similarity matrices of $ \{\mathbf{Z}^{{v}}\}_{v=1}^{2}$ and $\mathbf{Z}$,  $ \{\mathbf{H}^{{v}}\}_{v=1}^{2}$ and $\mathbf{H}$ on NoisyMNIST with $r=0.5$.}
    \label{fig:2}
\end{figure}
\vspace{-2pt}
\begin{figure}[htbp]
    \centering
    \begin{minipage}{0.5\textwidth}
        \begin{subfigure}{0.24\textwidth}
            \centering
            \includegraphics[width=\linewidth]{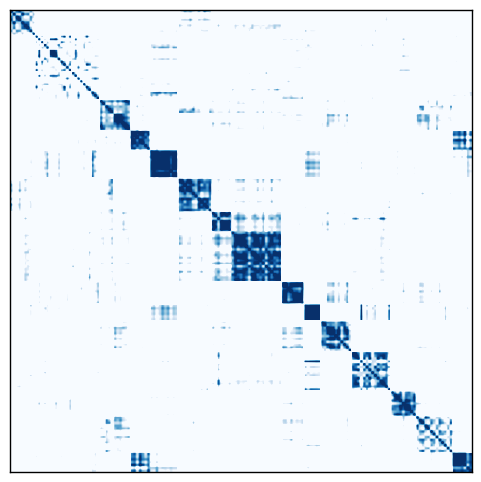}
            \caption{$\mathbf{Z}^{(1)}$}
        \end{subfigure}%
        \hfill
        \begin{subfigure}{0.24\textwidth}
            \centering
            \includegraphics[width=\linewidth]{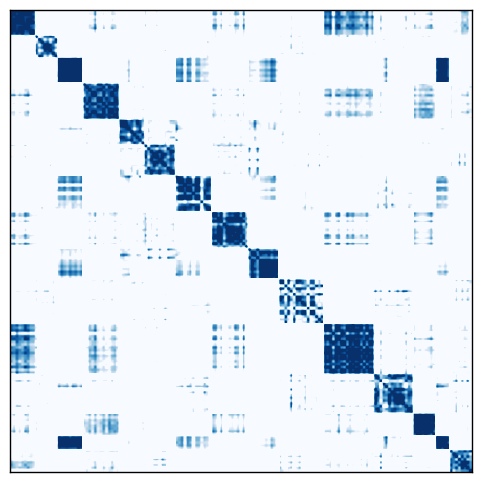}
            \caption{$\mathbf{Z}^{(2)}$}
        \end{subfigure}%
        \hfill
        \begin{subfigure}{0.24\textwidth}
            \centering
            \includegraphics[width=\linewidth]{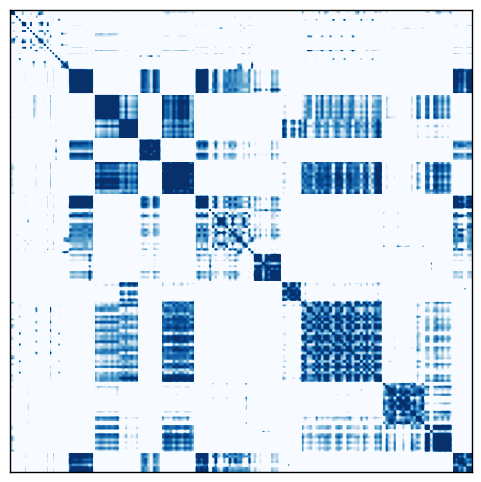}
            \caption{$\mathbf{Z}^{(3)}$}
        \end{subfigure}%
        \hfill
        \begin{subfigure}{0.24\textwidth}
            \centering
            \includegraphics[width=\linewidth]{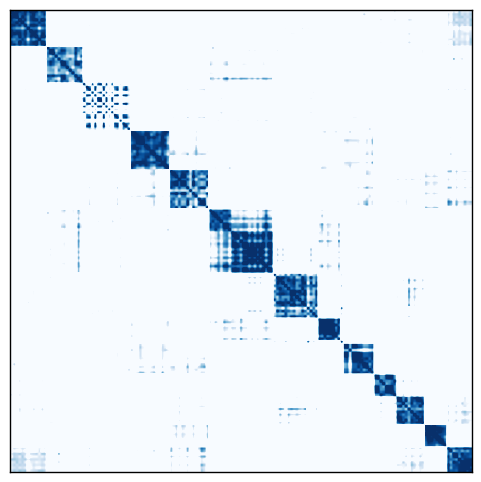}
            \caption{$\mathbf{Z}$}
        \end{subfigure}%
    \end{minipage}
    \vspace{-10pt}
    \caption{Similarity matrices of $ \{\mathbf{Z}^{{v}}\}_{v=1}^{3}$, $ \mathbf{Z}$ on Yale with $r=0.5$.}
    \label{fig:one-row-four}
\end{figure}
\vspace{-2pt}
\begin{figure}[htbp]
    \centering
    \begin{minipage}{0.5\textwidth}
        \begin{subfigure}{0.24\textwidth}
            \centering
            \includegraphics[width=\linewidth]{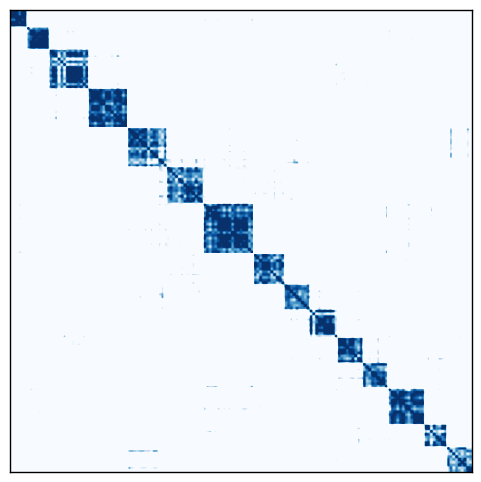}
            \caption{$\mathbf{H}^{(1)}$}
        \end{subfigure}%
        \hfill
        \begin{subfigure}{0.24\textwidth}
            \centering
            \includegraphics[width=\linewidth]{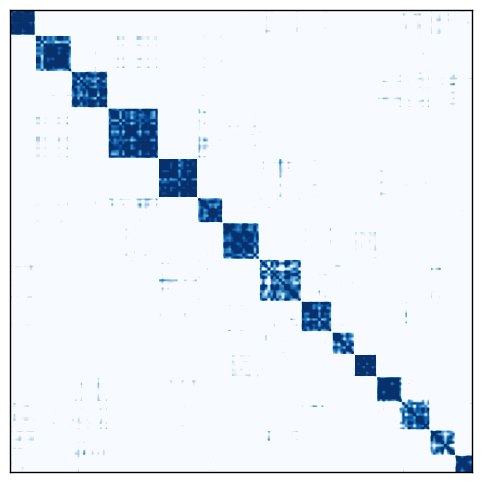}
            \caption{$\mathbf{H}^{(2)}$}
        \end{subfigure}%
        \hfill
        \begin{subfigure}{0.24\textwidth}
            \centering
            \includegraphics[width=\linewidth]{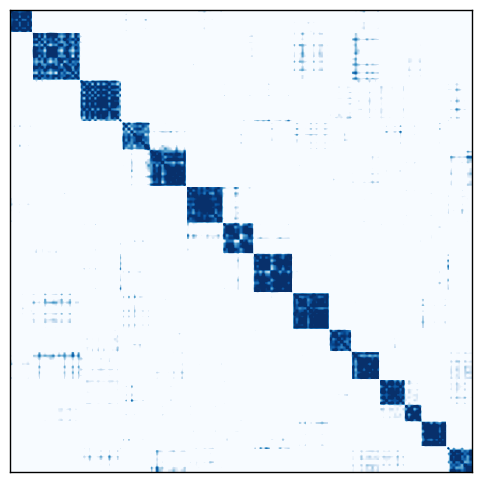}
            \caption{$\mathbf{H}^{(3)}$}
        \end{subfigure}%
        \hfill
        \begin{subfigure}{0.24\textwidth}
            \centering
            \includegraphics[width=\linewidth]{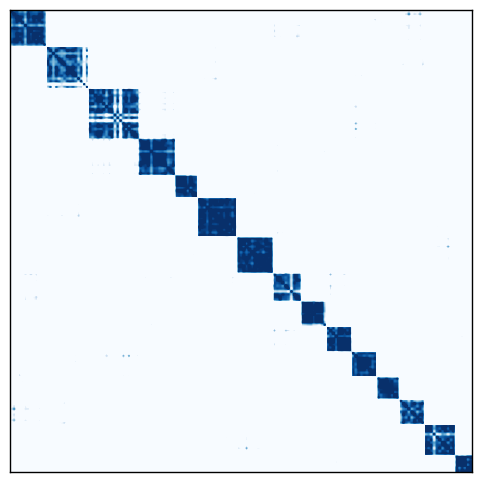}
            \caption{$\mathbf{H}$}
        \end{subfigure}%
    \end{minipage}
    \vspace{-10pt}
    \caption{Similarity matrices of $ \{\mathbf{H}^{{v}}\}_{v=1}^{3}$, $ \mathbf{H}$ on Yale with $r=0.5$.}
    \label{fig:one-row-four}
\end{figure}
\vspace{-2pt}
\begin{figure}[htbp]
    \centering
    \begin{minipage}{0.5\textwidth}
        \begin{subfigure}{0.15\textwidth}
            \centering
            \includegraphics[width=\linewidth]{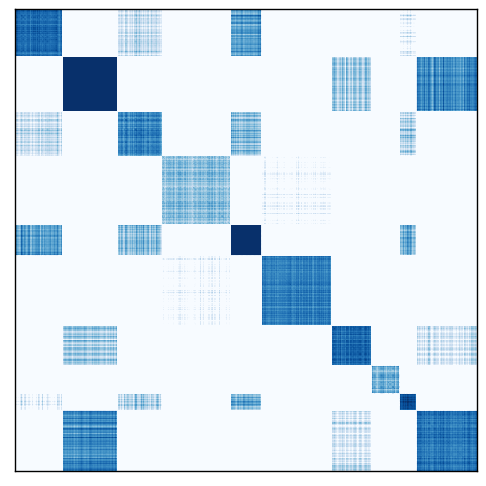}
            \caption{$\mathbf{Z}^{(1)}$}
        \end{subfigure}%
        \hfill
        \begin{subfigure}{0.15\textwidth}
            \centering
            \includegraphics[width=\linewidth]{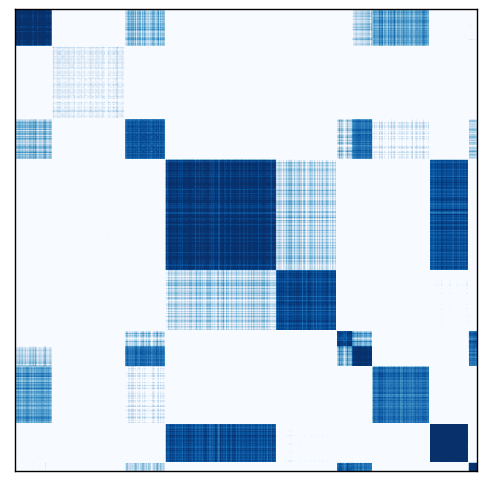}
            \caption{$\mathbf{Z}^{(2)}$}
        \end{subfigure}%
        \hfill
        \begin{subfigure}{0.15\textwidth}
            \centering
            \includegraphics[width=\linewidth]{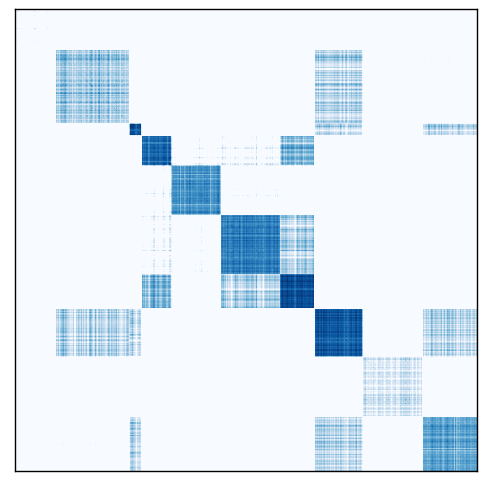}
            \caption{$\mathbf{Z}^{(3)}$}
        \end{subfigure}%
        \hfill
        \begin{subfigure}{0.15\textwidth}
            \centering
            \includegraphics[width=\linewidth]{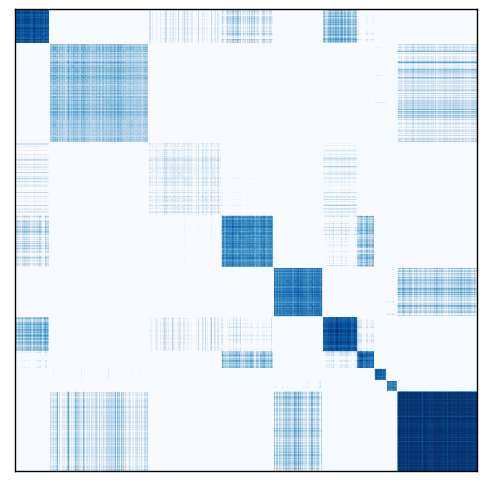}
            \caption{$\mathbf{Z}^{(4)}$}
        \end{subfigure}%
        \hfill
        \begin{subfigure}{0.15\textwidth}
            \centering
            \includegraphics[width=\linewidth]{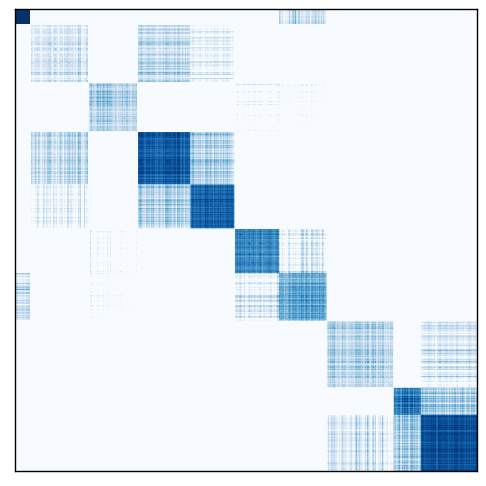}
            \caption{$\mathbf{Z}^{(5)}$}
        \end{subfigure}%
        \hfill
        \begin{subfigure}{0.15\textwidth}
            \centering
            \includegraphics[width=\linewidth]{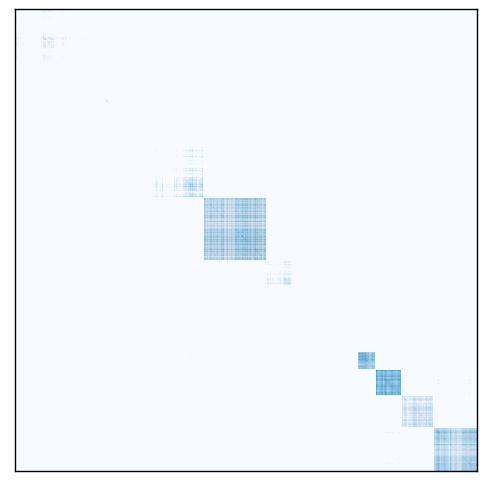}
            \caption{$\mathbf{Z}$}
        \end{subfigure}
    \end{minipage}
    \vspace{-10pt}
    \caption{Similarity matrices of $ \{\mathbf{Z}^{{v}}\}_{v=1}^{5}$ and $\mathbf{Z}$ on NUSWIDEOBJECT10 with $r=0.5$.}
    \label{fig:one-row}
\end{figure}
\vspace{-2pt}
\begin{figure}[htbp]
    \centering
    \begin{minipage}{0.5\textwidth}
        \begin{subfigure}{0.15\textwidth}
            \centering
            \includegraphics[width=\linewidth]{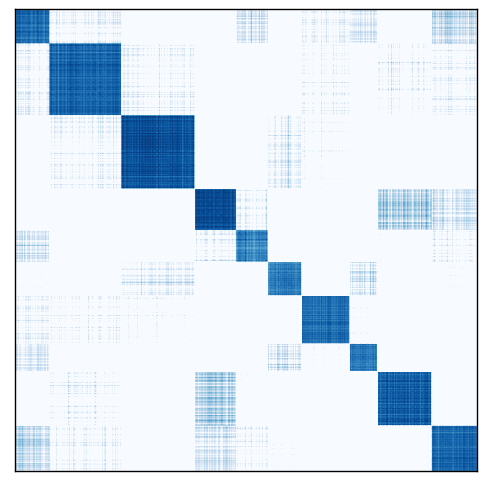}
            \caption{$\mathbf{H}^{(1)}$}
        \end{subfigure}%
        \hfill
        \begin{subfigure}{0.15\textwidth}
            \centering
            \includegraphics[width=\linewidth]{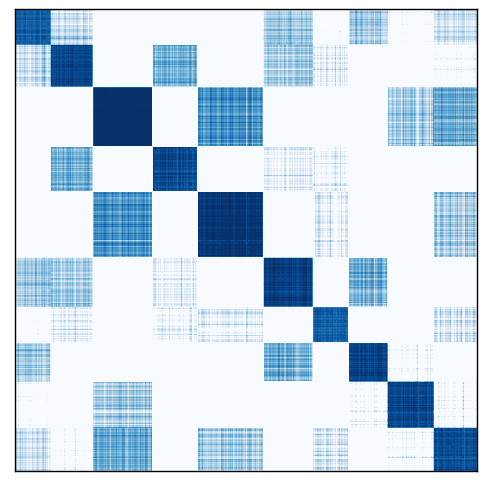}
            \caption{$\mathbf{H}^{(2)}$}
        \end{subfigure}%
        \hfill
        \begin{subfigure}{0.15\textwidth}
            \centering
            \includegraphics[width=\linewidth]{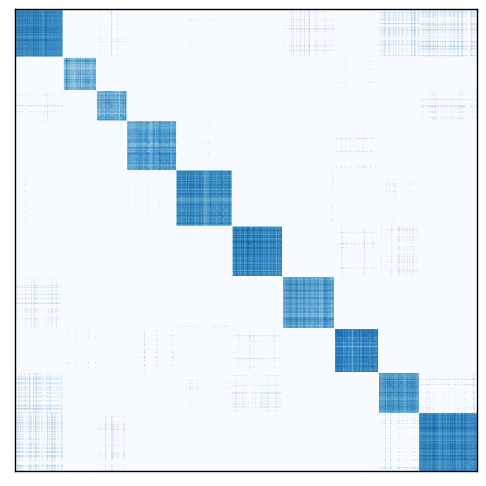}
            \caption{$\mathbf{H}^{(3)}$}
        \end{subfigure}%
        \hfill
        \begin{subfigure}{0.15\textwidth}
            \centering
            \includegraphics[width=\linewidth]{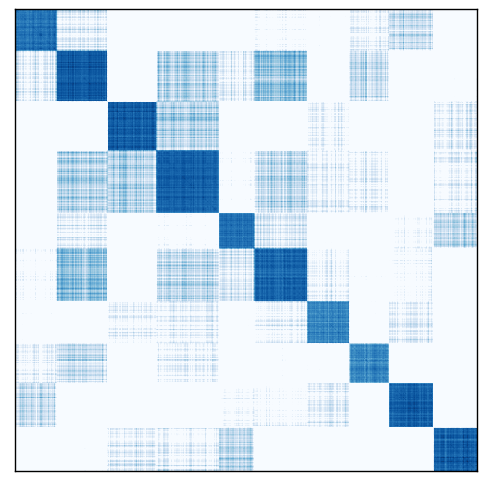}
            \caption{$\mathbf{H}^{(4)}$}
        \end{subfigure}%
        \hfill
        \begin{subfigure}{0.15\textwidth}
            \centering
            \includegraphics[width=\linewidth]{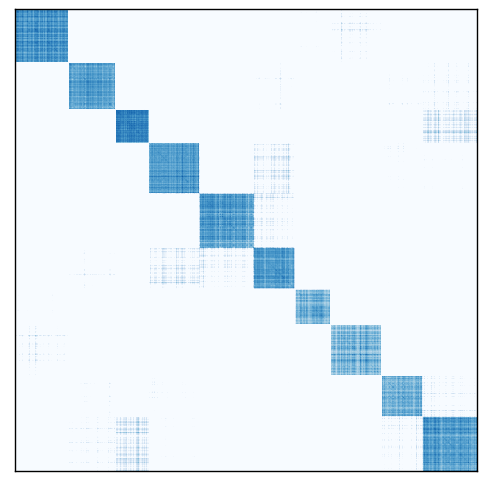}
            \caption{$\mathbf{H}^{(5)}$}
        \end{subfigure}%
        \hfill
        \begin{subfigure}{0.15\textwidth}
            \centering
            \includegraphics[width=\linewidth]{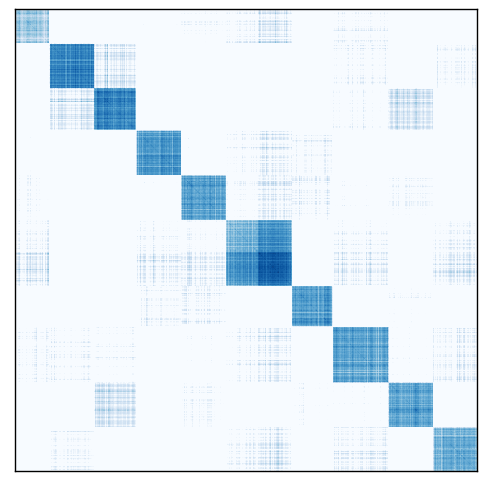}
            \caption{$\mathbf{H}$}
        \end{subfigure}
    \end{minipage}
    \vspace{-10pt}
    \caption{Similarity matrices of $ \{\mathbf{H}^{{v}}\}_{v=1}^{5}$ and $\mathbf{H}$ on NUSWIDEOBJECT10 with $r=0.5$.}
    \label{fig:6}
\end{figure}

\subsection{Analysis on FreeCSL}\label{experi:ablation}
\noindent\textbf{Parameter Sensitivity Analysis.} 
As in Section 3.5, we perform a parameter sensitivity analysis on the number of neighbors $\lambda$ and the regularization coefficient $\zeta$ in graph clustering, on YoutubeFace10, NoisyMNIST, Yale and NUSWIDEOBJ10 datasets. Fig. \ref{fig:Sensitivity} shows our model is highly stable, with minimal performance fluctuation even when $\lambda$ and $\zeta$ are adjusted to ranges of 3 to 32 and 0.05 to 0.5, respectively. A smaller number of neighbors $\lambda$ and more relaxed regularization constraints $\zeta$, will yield higher clustering accuracy (ACC). Except for the large-scale NoisyMNIST dataset, where a larger number of neighbors effectively enhance model performance by aggregating more useful neighbor information to discover cluster structures. In conclusion, our model present outstanding performance in complex clustering tasks without sacrificing computational resources for clustering accuracy or relying on strict regularization constraints.

\begin{figure}[!ht]
    \centering  
    \begin{subfigure}[t]{0.48\linewidth}  
        \includegraphics[width=\linewidth]{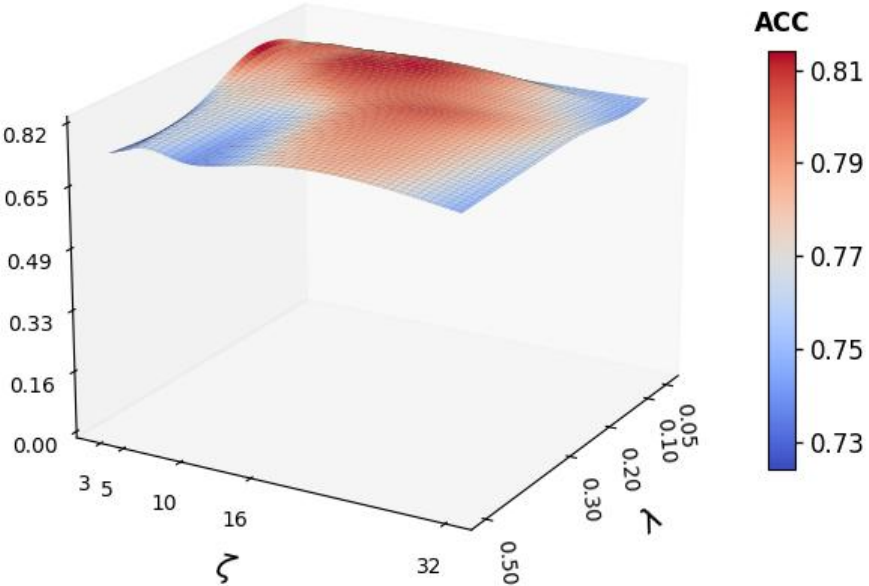}  
        \caption{YoutubeFace10}  
    \end{subfigure}  
    \hfill
    \begin{subfigure}[t]{0.48\linewidth}  
        \includegraphics[width=\linewidth]{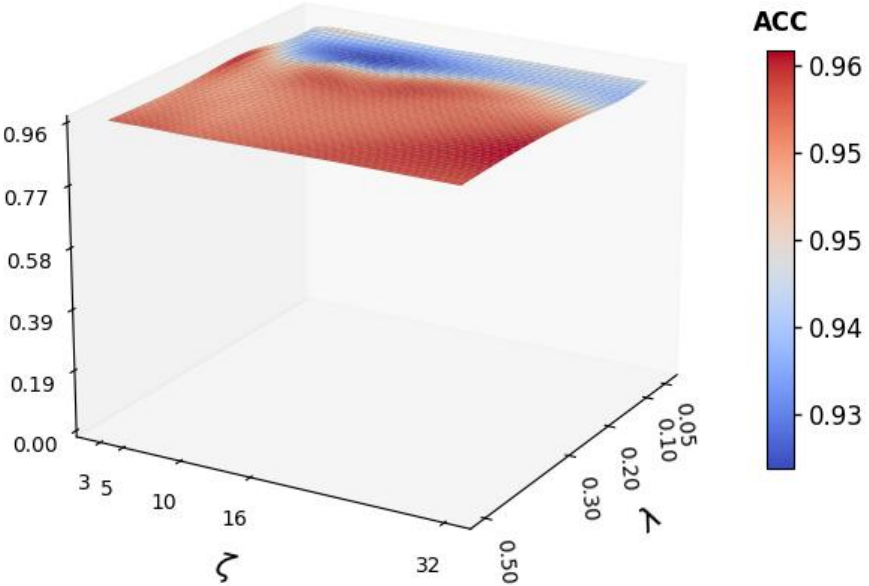}  
        \caption{NoisyMNIST}  
    \end{subfigure}  
    \vspace{5pt}
    \begin{subfigure}[t]{0.48\linewidth}  
        \includegraphics[width=\linewidth]{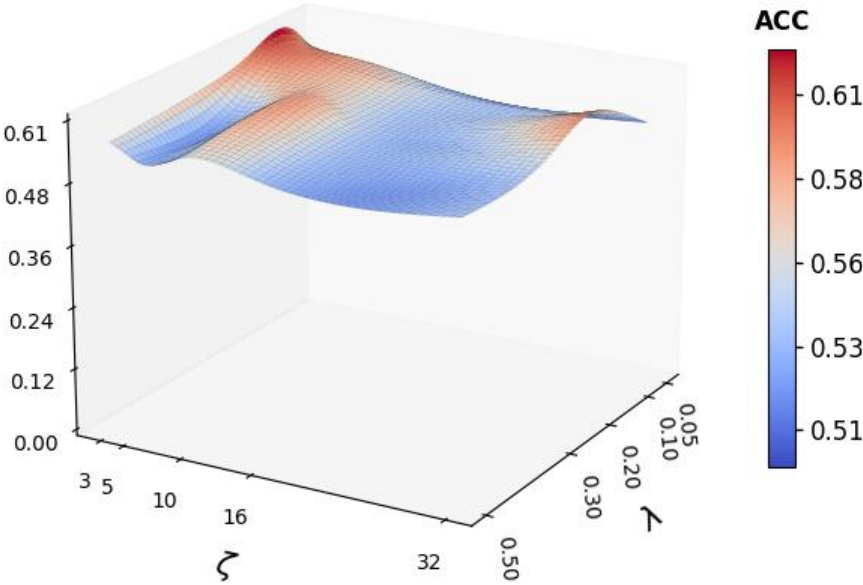}  
        \caption{Yale}  
    \end{subfigure}  
    \hfill
    \begin{subfigure}[t]{0.48\linewidth}  
        \includegraphics[width=\linewidth]{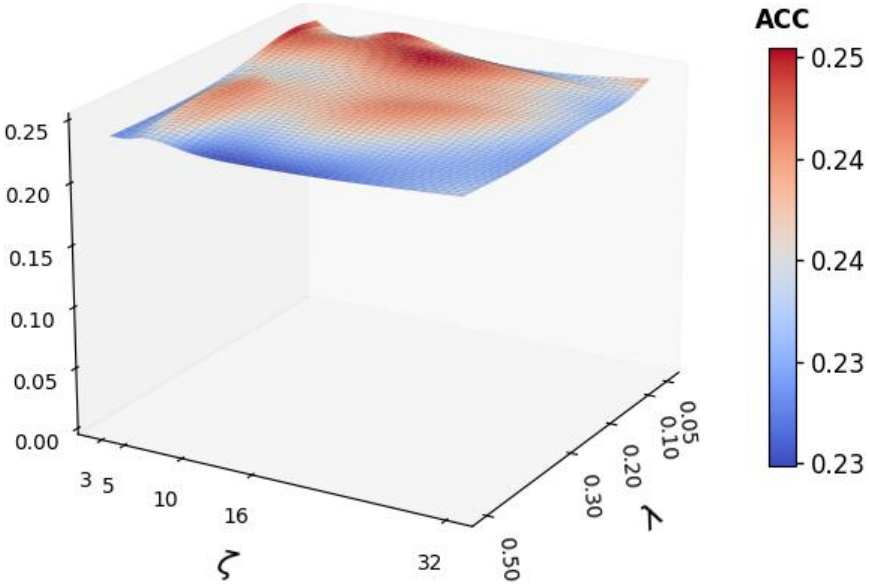}  
        \caption{NUSWIDEOBJ10}  
    \end{subfigure}  
    \vspace{-5pt}  
    \caption{Parameter analyses for $\zeta$ and $\lambda$ with $r=0.5$.}  
    \label{fig:Sensitivity}  
\end{figure}
\FloatBarrier
\begin{figure}[!ht] 
    \centering  
    \begin{subfigure}[t]{0.45\linewidth}  
        \includegraphics[width=\linewidth]{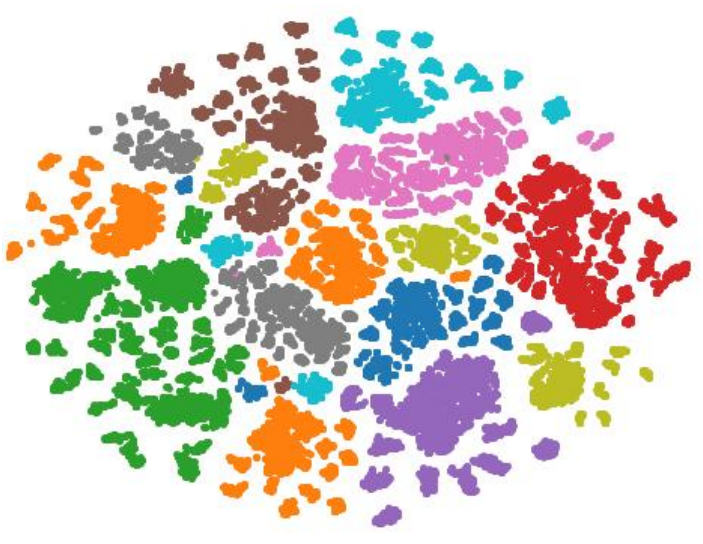}  
        \scriptsize\caption{{Pre-training on YoutubeFace10 (NMI = 69.41\%)}}  
    \end{subfigure}  
    \hfill  
    \begin{subfigure}[t]{0.45\linewidth}  
        \includegraphics[width=\linewidth]{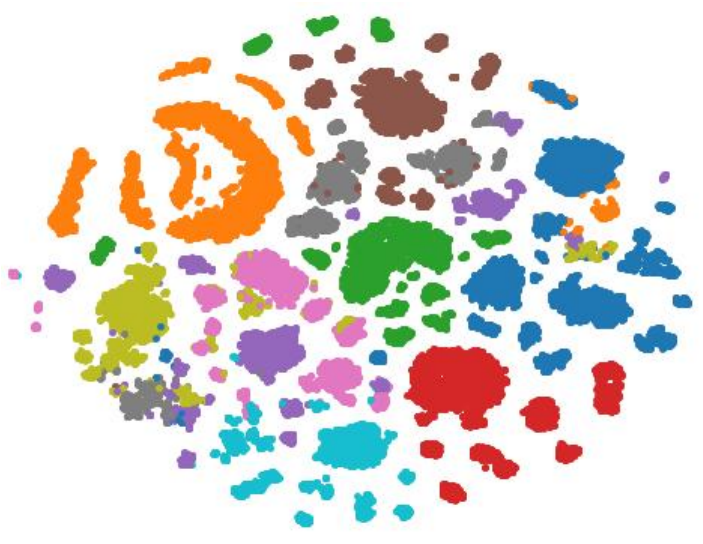}  
        \scriptsize\caption{{Training on YoutubeFace10 (NMI = 81.07\%)}}  
    \end{subfigure}  

    \vspace{0.5em} 
    \begin{subfigure}[t]{0.45\linewidth}  
        \includegraphics[width=\linewidth]{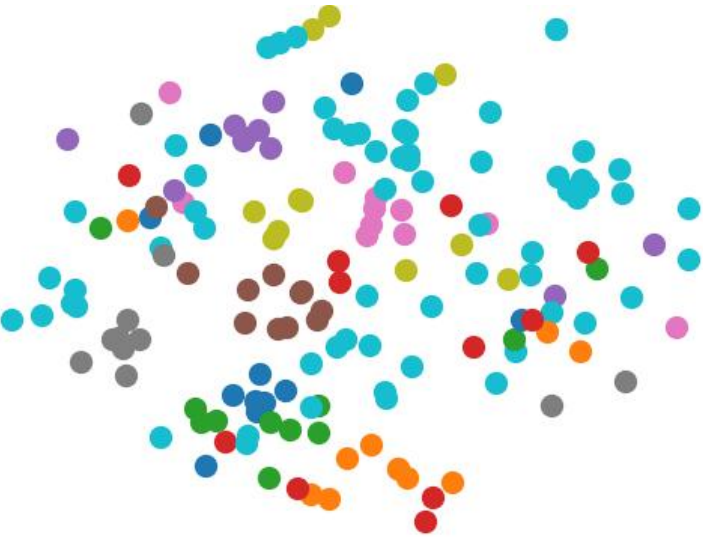}  
        \scriptsize{\caption{{Pre-training on Yale (NMI = 44.81\%)}} } 
    \end{subfigure}  
    \hfill  
    \begin{subfigure}[t]{0.45\linewidth}  
        \includegraphics[width=\linewidth]{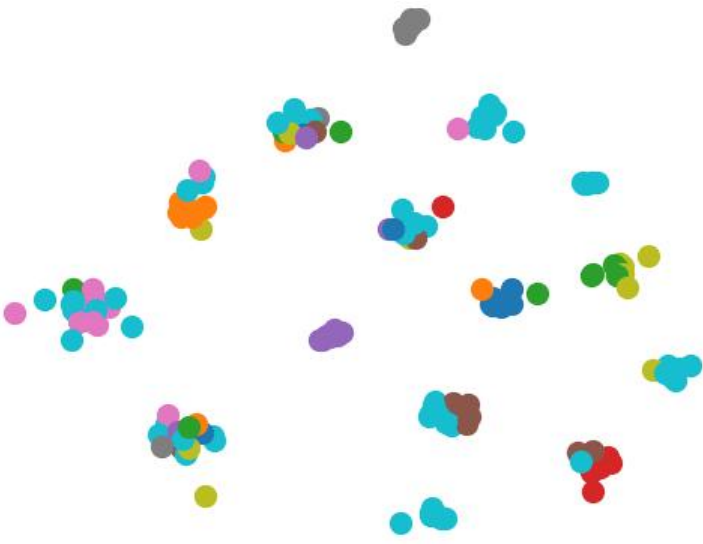}  
        \scriptsize\caption{{Training on Yale (NMI = 63.14\%)}}  
    \end{subfigure}  
    \vspace{-3pt}
    \caption{Visualization on YoutubeFace10 and Yale with $r=0.5$.}
    \vspace{-10pt}
\label{fig:t-SNE}
\end{figure}

\subsection{Visualization for Consensus Semantic Clusters}\label{experi:ablation}
Referring to true labels, we visualize the clustering effect of consensus semantic representations on YoutubeFace10 and Yale with the setting of missing rate $r=0.5$,
shown in Fig. \ref{fig:t-SNE} respectively. We can
observe that after the training of our model, all instances converge toward their respective clusters, where instances within the same cluster become more compact, and instances from
different clusters are separated far away.
In addition, the visualization results of the prototypes of
each cluster further confirm that through consensus prototype-based semantic learning, the shifted prototypes are re-estimated and accurately calibrated without the need for extra alignment processes.

\end{document}